\documentclass[a4paper,11pt]{style}

\usepackage{amsmath,amsfonts,amssymb,amsthm}
\usepackage{mathtools}
\usepackage{bbm}
\usepackage{accents}

\usepackage{comment}
\usepackage{adjustbox}
\usepackage{tikz}
\usepackage{overpic}
\usepackage{enumitem}
\usepackage{rotating}
\usepackage{algorithm}
\usepackage{algorithmic}

\usepackage[textsize=tiny]{todonotes}

\usepackage{amsmath,amsfonts,amssymb,amsthm}

\usepackage{amsmath}
\usepackage{amssymb}
\usepackage{mathtools}
\usepackage{amsthm}
\usepackage{natbib}
\usepackage{xcolor}

\usepackage[most]{tcolorbox}

\definecolor{sand}{RGB}{248,245,235}  
\definecolor{sandborder}{RGB}{220,215,200}
\theoremstyle{plain}
\newtheorem{definition}{Definition}
\newtheorem{example}{Example}

\newtheorem{proposition}{Proposition}
\newtheorem{lemma}{Lemma}
\newtheorem{theorem}{Theorem}

\newtheorem{corollary}{Corollary}

\newcommand{\E}{\mathbb{E}}

\newcommand{\R}{\mathbb{R}}

\usepackage[textsize=tiny]{todonotes}

\usepackage{xcolor,pifont}
\newcommand*\colourcheck[1]{%
  \expandafter\newcommand\csname #1check\endcsname{\textcolor{#1}{\ding{52}}}%
}
\colourcheck{blue}
\colourcheck{green}
\colourcheck{red}

\usepackage{natbib}

\usepackage[most]{tcolorbox}

\definecolor{sand}{RGB}{248,245,235}  
\definecolor{sandborder}{RGB}{220,215,200}

\usepackage[table]{xcolor}
\usepackage{tabularx}
\usepackage{array}
\definecolor{oursorange}{RGB}{255,241,222}
\newcolumntype{O}{>{\columncolor{orange!12}}r}

\newcommand{\redcross}{\textcolor{red}{\texttimes}}

\title{Is Flow Matching Just Trajectory Replay for Sequential Data?}

\renewcommand{\authorlist}{%
  Soon Hoe Lim\textsuperscript{1,2,*}
  \hspace{0.35cm}
  Shizheng Lin\textsuperscript{1,2,*}
  \\[3pt]
  Michael W. Mahoney\textsuperscript{3,4,5}
  \hspace{0.35cm}
  N. Benjamin Erichson\textsuperscript{4,5}
}

\renewcommand{\affiliationlist}{%
  {\small
    \textsuperscript{1}Department of Mathematics,
    KTH Royal Institute of Technology
  }\\[1.5pt]
  {\small
    \textsuperscript{2}Nordita,
    KTH Royal Institute of Technology and Stockholm University
  }\\[1.5pt]
  {\small
    \textsuperscript{3}Department of Statistics,
    University of California at Berkeley
  }\\[1.5pt]
  {\small
  \textsuperscript{4}International Computer Science Institute
  }\\[1.5pt]
  {\small
    \textsuperscript{5}Lawrence Berkeley National Laboratory
  }%
}

\renewcommand{\contributionlist}{%
  {\small\textsuperscript{*}Equal contribution.}%
}

\abstract{%
Flow matching (FM) is increasingly used in scientific domains for time series generation and forecasting, where data often arise from underlying dynamical systems. However, it is not well-understood whether it learns transferable dynamical structure or simply performs an effective ``trajectory replay''. We study this question by deriving the optimal velocity field targeted by the empirical FM objective on sequential data. For the Gaussian conditional paths commonly used in practice, we show that the resulting sampler is an ODE whose dynamics describe a nonparametric, memory-augmented continuous-time dynamical system. The driving velocity field admits a closed-form expression as a weighted mixture of instantaneous velocities determined by observed transitions. This characterization provides an interpretable view of neural FM models as parametric estimation of an underlying nonparametric target and naturally leads to FreeFM: a closed-form, training-free realization of this target. FreeFM provides a transparent reference for studying sequential FM while yielding effective probabilistic forecasts directly from historical transitions. \\

{\sffamily\bfseries\color{muxfg}Code: } {\color{blue}{\url{https://github.com/shoelim/FreeFM}}}}

\begin{document}

\maketitle

\section{Introduction}

Forecasting sequential data from dynamical systems involves learning the transition of an observed state  to the next and repeatedly applying the learned transition rule to generate future trajectories. A useful forecasting model should provide not only accurate point predictions, but also multiple plausible predictions. Modern generative models offer a flexible approach to this problem \citep{rasul2021autoregressive, bilos2023modelingtemporaldatacontinuous, alcaraz2023diffusionbasedtimeseriesimputation, lim2024elucidating, jahn2025trajectory, zhang2024trajectory}, but the learned models are typically black boxes and can be difficult to interpret; see related work in App.~\ref{app:relatedwork}.

Flow matching  (FM)~\citep{lipman2022flow,lipman2024flow}, including closely
related approaches such as Rectified Flow~\citep{liu2022flow} and Stochastic
Interpolants~\citep{albergo2023stochastic}, provides a simulation-free
regression framework for learning continuous-time generative dynamics.
Given a probability path $p_t(z)$ connecting a base distribution $p_0$ and a target distribution $p_1$, and an associated conditional velocity
field, FM trains a parameterized velocity $v_\theta(t,z)$ through a regression
objective rather than through numerical simulation of the generative ODE
during training.
The learned field can subsequently be integrated to generate new samples: $z(0) \sim p_0$,
\begin{equation}
\frac{dz(t)}{dt} = v_\theta(t, z(t)), \quad z(t)\in\mathbb{R}^d,\quad t\in[0,1],    
\end{equation}
transforming $p_0$ at $t=0$ to $p_1$ at $t=1$. 
FM formulations can  be used to learn generative models of sequential data using a memory bank of observed one-step transitions \citep{lim2024elucidating, zhang2024trajectory, jahn2025trajectory}. Despite promising results, the implicit bias of these models is not well understood.

In this work, we focus specifically on \emph{transition-based forecasting}:
given observed one-step pairs, we construct a
generative model for the next state and compose this one-step mechanism
autoregressively to obtain multi-step forecasts. 
FM and related constructions have recently been
applied to sequential and dynamical data
\citep{lim2024elucidating,zhang2024trajectory, jahn2025trajectory}, but the finite-sample dynamics
implicitly selected by the FM objective are not well understood.
This raises the  basic question:
\begin{tcolorbox}[
  colback=sand,
  colframe=sandborder,
  boxrule=0.3pt,
  left=8pt,
  right=8pt,
  top=3pt,
  bottom=3pt
]
\centering
\emph{When FM is applied to a finite bank of observed one-step transitions, what empirical velocity field $\hat v^*(t,z)$ does the FM objective select, and what dynamical system does this field define?
}  
\end{tcolorbox}

\begin{figure*}[!t]
    \centering
    \includegraphics[width=0.95\textwidth]{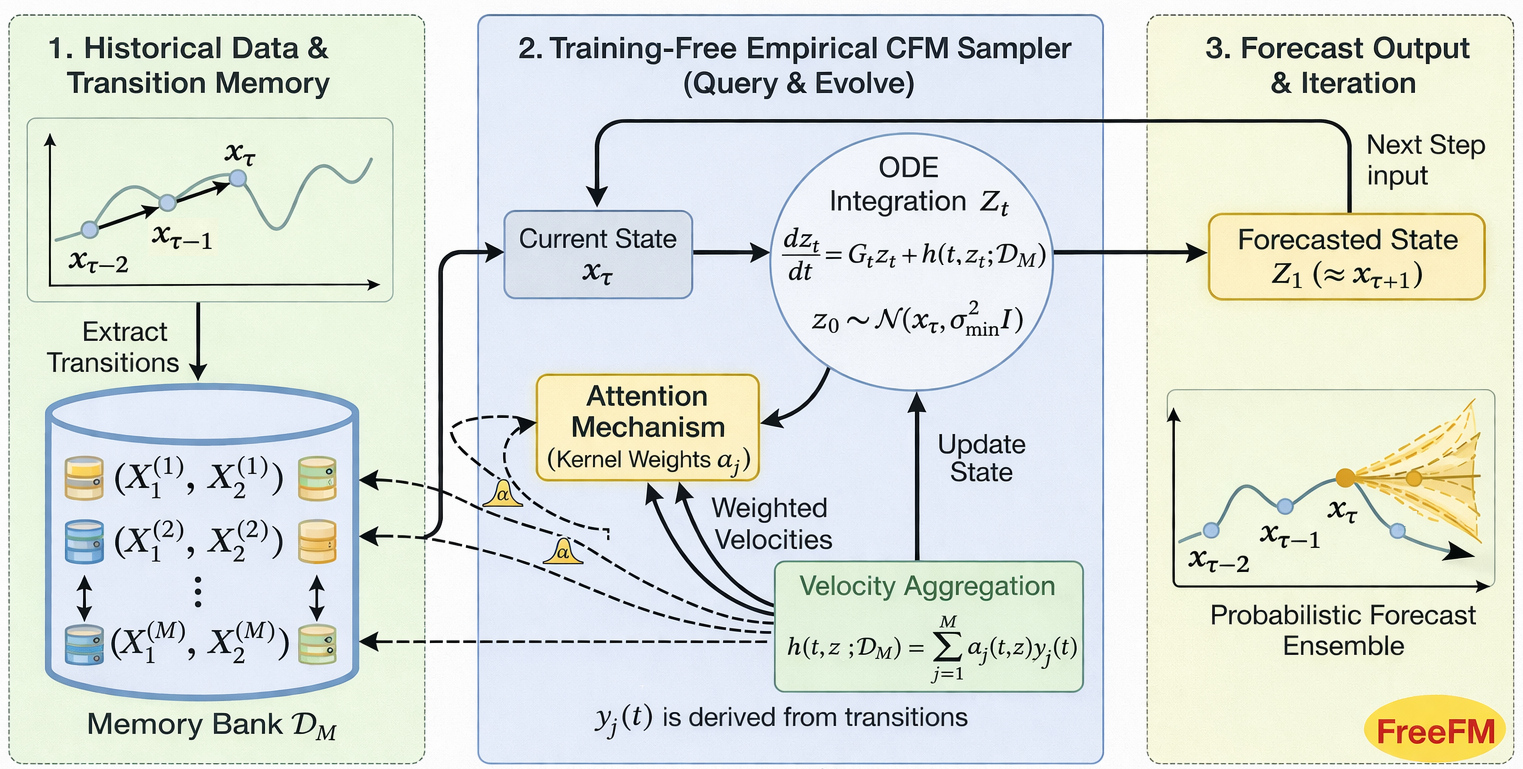}
\caption{{\bf What Dynamical System Are FM Forecasters Actually Sampling From?} For sequential transition data, the optimal empirical FM velocity field induces a  nonparametric, memory-based ODE that can be evaluated directly from historical
transitions. For the Gaussian bridge considered in this work, the same field
admits two equivalent representations:
$\hat{v}(t,z)=G_t z+h(t,z;\mathcal{D}_M)$ (see Eq.~\eqref{eq:CFM-memory-ODE}) and a decomposition into transition replay plus a path-induced score correction. Starting from a local Gaussian initialization around the current
state $x_\tau$, integration of this field gives a next-state sample, and
repeated application gives a multi-step forecast. Propagating multiple
particles yields a predictive ensemble.}
    \label{fig:schematic}
\end{figure*}

Our starting point is the
characterization of the optimal velocity field for FM in terms of conditional expectation \citep{scarvelis2023closed, bertrand2025closed, li2026kinetic, lim2025hidden}. Under an empirical
distribution, this gives a posterior-weighted mixture of conditional velocity
fields. This general mixture identity is inherited from FM, and our focus is to study what happens when the empirical data are  observed sequential transitions. 

We show that the resulting field defines a nonparametric, memory-based dynamical
system that retrieves and combines information from historical transitions.
For the  Gaussian probability path studied here, this structure admits
an explicit decomposition into a ``transition replay'' term and a path-induced score correction. Thus, empirical FM in this setting is not simply ``trajectory
replay'': the prescribed probability path systematically modifies the replay
dynamics.

This characterization yields FreeFM, which directly evaluates the empirical
transition-level FM field without fitting a neural  model. Besides
providing a simple forecasting procedure, FreeFM serves as an interpretable
empirical reference for understanding what a neural FM model trained on the
same transition-level objective approximates. 
In more detail, our main contributions are as follows:
\begin{itemize}[leftmargin=*]
\item \textbf{Empirical FM as transition memory.}
We specialize the empirical FM posterior-mixture construction to
observed transition pairs and show that it defines a nonparametric
memory-based dynamical system. For the Gaussian probability path,
the optimal empirical field further decomposes into 
a transition replay term and a path-induced score correction (see Sec.~\ref{sec:method}). This
decomposition makes explicit how the probability path controls transition
retrieval and modifies the replay dynamics.
 We further study the behavior of the resulting  sampler; see App.~\ref{app:duhamel_dm}.
\item \textbf{FreeFM as an interpretable empirical reference.}
The resulting full-memory velocity can be evaluated directly from the
transition bank without gradient-based training, yielding FreeFM
(Fig.~\ref{fig:schematic}). We analyze the numerical properties of the
resulting ODE sampler and approximation schemes for practical computation
(see Sec.~\ref{sec:practical}). 
Proofs of the  theoretical results  are provided in App.~\ref{app:proof}. 
\item \textbf{Empirical evaluation.}
Focusing on  data arising in nonlinear dynamics \citep{gilpin2021chaos, berry2025limits}, we validate the theoretical insight and demonstrate the effectiveness of FreeFM on standard benchmark tasks (see Sec.~\ref{sec:exp}). We also evaluate  on  low- and high-dimensional real-world datasets (see App.~\ref{app_realworld}). Additional evaluations are provided in App.~\ref{app_ablation_sigma}-\ref{app_more}. Overall, our results support  FreeFM as a strong empirical reference in well-covered low-to-moderate dimensional regimes rather than as a universally superior
forecaster.  
\end{itemize}

\section{Empirical Flow Matching for Sequential Data} 
\label{sec:method}

In this section, we specialize FM (see App.~\ref{app:background} for background) to sequential data represented by
observed one-step transitions. 
We then study empirical FM and use the optimal empirical velocity field  to derive a training-free model for probabilistic forecasting. Finally, we discuss various interpretations of this model.

{\bf Setting.} We are given $N$ independent realizations of the trajectory of states in $\mathbb{R}^d$, each sampled at equidistant intervals:
$$\big\{ x^{(n)}_\tau  \in \mathbb{R}^d :\ \tau = 0,\dots, T_n-1 \big\}, \ n=1, \dots, N, $$
where $x_\tau^{(n)}$ is the realization of  the  state $X_\tau$ at time index $\tau$.
From each realized trajectory, we extract all consecutive one-step transitions $(x_\tau^{(n)},\, x_{\tau + 1}^{(n)})$.
Collecting these transitions across all trajectories yields the \emph{transition dataset}:
$$\mathcal{D}_M := \big\{ X^{(j)} \big\}_{j=1}^{M}, \quad M = \sum_{n=1}^N (T_n - 1).$$
Here, we introduce a single transition index $j \in \{1,\dots,M\}$ via a bijective mapping
$j \longmapsto \big(\tau(j),\, n(j)\big)$. 
Each $j$ indexes a specific realized transition pair:
$$X^{(j)} = \big( x_{\tau(j)}^{(n(j))},\, x_{\tau(j)+1}^{(n(j))} \big) =: \big( X^{(j)}_1,\, X^{(j)}_2 \big),$$
realized from the random variable $X$.
This flattened representation $\mathcal{D}_M = \{ \big( X^{(j)}_1,\, X^{(j)}_2 \big)  \}_{j=1}^M$ serves
as the \emph{memory bank} for the empirical FM sampler that we introduce later. The key idea is to construct a  velocity field whose flow interpolates between the neighborhoods of $X_1^{(j)}$ and $X_2^{(j)}$ for all transitions in $\mathcal{D}_M$. 
Note that the constructions of the memory bank  and the sampler rely only on observed transitions, and the core procedure holds without imposing any assumptions on the underlying data dynamics.

Next, we address the earlier question for the transition data setting and derive a training-free model driven by a closed-form velocity field. 

\subsection{Affine Conditional Probability Paths}
\label{sec:ts-affine-general}

We first instantiate a conditional version of FM (referred to as CFM) for the transition dataset.
Let the conditioning variable be a transition pair $X := (X_1, X_2) \sim P_X$. We associate with each transition a probability path on the state space $\mathbb{R}^d$ as follows.

Let $Z \in \mathbb{R}^d$ be a base random variable (independent of $X$) with probability density function  $K$.
For each transition $X$, we define an affine conditional flow map $\psi_t(\cdot \mid X): \mathbb{R}^d \to \mathbb{R}^d$ as:
\begin{equation}
\psi_t(Z \mid X) = m_t(X) + \sigma_t(X)\, Z,
\end{equation}
where
$m_t : [0,1]\times\R^{2d} \to \R^{d}, $
$\sigma_t : [0,1]\times\R^{2d} \to (0,\infty)$
are differentiable in $t$. Defining the random variable $Z_t := \psi_t(Z \mid X)$ for $t \in [0,1]$, the induced conditional density is given by the change of variable formula: 
\begin{equation}
p_t(z \mid X)
=
\frac{1}{\sigma_t(X)^d}\,
K\!\left( \frac{z - m_t(X)}{\sigma_t(X)} \right).
\label{eq:ts-probpath-affine}
\end{equation}

This flow is generated by a vector field $v(t,z \mid X)$ that is affine with respect to $z$. 
More precisely, the (unique) vector field that generates $\psi_t$ via the ODE
\(\tfrac{d}{dt}\psi_t(Z\mid X) = v(t, \psi_t(Z\mid X)\mid X)\) is:
\begin{equation}
v(t,z\mid X) = a_t(X)\, z + b_t(X), \quad \text{ with:}
\end{equation}
\begin{equation}
a_t(X) = \frac{\partial_t \sigma_t(X)}{\sigma_t(X)},
\quad
b_t(X) = \partial_t m_t(X) - a_t(X)\, m_t(X).
\end{equation}
The population mixture path is defined as the marginal density:
\begin{equation}
p_t(z)
=
\int p_t(z \mid X=x)\, dP_X(x),
\end{equation}
where $P_X$ is the transition distribution on the product space.
Our goal is to find a velocity field $v(t,z)$ that generates this marginal path, so that the constructed flow transports $p_0$ to $p_1$ along
$\{p_t\}_{t\in[0,1]}$.

\subsection{Empirical CFM: Closed-Form Solution}
\label{sec:ts-emp-affine}

The transition dataset $\mathcal D_M$ defines an empirical joint law
$\hat p_1 = \frac1M\sum_{j=1}^M \delta_{X^{(j)}}  \text{ on } \R^{2d}.$
Equivalently, we may introduce a discrete latent index $C \in  \{1, \dots, M\}$ with the uniform prior $\pi(j)=1/M$ and set $X=X^{(C)}$.
This choice corresponds to a paired (diagonal) coupling between consecutive states, in contrast to an independent or optimal   coupling between endpoint samples.

Replacing the population measure $P_X$ with the empirical measure $\hat p_1$, the corresponding empirical marginal path becomes a mixture of the conditional densities:
\[
\hat p_t(z)
=
\frac{1}{M} \sum_{j=1}^M p_t(z \mid X^{(j)}).
\]
We define the \emph{empirical responsibilities} (or posterior weights) $w_j(t, z)$ as the normalized contribution of the $j$-th transition to the density at location $z$ and time $t$:
\begin{equation}
w_j(t, z)
=
\frac{
    p_t(z \mid X^{(j)})
}{
    \sum_{k=1}^M p_t(z \mid X^{(k)})
},
\quad j=1,\dots,M.
\label{eq:ts-weights-general}
\end{equation}
Let the conditioning variable $X \sim \hat{p}_1$ on $\R^D$, where $D=2d$ for the transition setting,  and  $p_t(z|X)$ be a path on $\R^d$.  The following theorem provides the  minimizer for the FM objective under this mixture model.

\begin{theorem}[Closed-Form Empirical  FM]
\label{thm:ts-emp-affine}
For the affine conditional flow generated by  $v(t,z|X) = a_t(X)z + b_t(X)$ (where $a_t: \R^D \to \R$, $b_t: \R^D \to \R^d$), 
the (unique) minimizer of the empirical CFM (equivalently FM) objective
\begin{align}
&\hat{\mathcal{L}}_{\mathrm{CFM}}[v'] =
\mathbb{E}_{t}\,
\mathbb{E}_{X}\,
\mathbb{E}_{Z_t}
\|v'(t,Z_t) - v(t,Z_t\mid X)\|^2,
\end{align}
where the expectation is over $t\sim\mathcal{U}[0,1]$, $X\sim\hat p_1$ and $Z_t\sim p_t(\cdot\mid X)$, 
admits the  closed-form formula:
\begin{equation} \label{eq_closedformv}
\hat{v}^*(t,z)
=
\sum_{j=1}^M
    w_j(t, z)\, \big(
        a_t(X^{(j)})\, z + b_t(X^{(j)})
    \big),
\end{equation}
where the weights $w_j(t, z)$ are given by \eqref{eq:ts-weights-general}.
\end{theorem}

See App.~\ref{proof_thm1} for a proof. The optimal velocity field $\hat{v}^*$ is a weighted mixture of the affine velocity fields per-transition attached to each observed transition, with weights determined by the  posterior probability that the point $z$ at time $t$ belongs to the conditional path originating from $X^{(j)}$. 
The formula \eqref{eq_closedformv} allows us to evaluate the vector field at any point $(t,z)$ simply by summing the data set, without training neural nets.

\subsection{Gaussian Bridge Conditional Path}

We now specialize the general framework to the Gaussian  path  proposed in \citep{lim2024elucidating}. Although the specific probability path is a modeling choice, we take it as our canonical model due to its principled motivation from dynamical optimal transport.

For each transition $X^{(j)} := (X_1^{(j)}, X_2^{(j)})$, we define
\begin{equation} \label{eq_ourbbpath}
Z_t^{(j)} = (1-t)X^{(j)}_1 + t X^{(j)}_2 + c_t \xi,
\quad
\xi \sim\mathcal N(0,I_d),
\end{equation}
where
$c_t^2 = \sigma_{\min}^2 + \sigma^2 t(1-t)$ with $\sigma \geq 0$ and  $\sigma_{\min}>0$ (which prevents degeneracy of the posterior responsibilities near
$t=0,1$).
Applying Theorem \ref{thm:ts-emp-affine} to this specific Gaussian path gives the following result.

\begin{proposition} 
\label{prop:freefm}
For the Gaussian path \eqref{eq_ourbbpath}, the optimal empirical CFM velocity field is given by:
\begin{equation}
\label{eq:bb-vstar}
\hat v^*(t,z)
=
G_t z
+ h(t, z; \mathcal{D}_M), 
\end{equation}
where 
$h(t, z; \mathcal{D}_M) =  \sum_{j=1}^M
\alpha_j(t,z) y_j(t).$
Here,  $y_j(t) = \dot m_t^{(j)} - G_t m_t^{(j)}$, where
$m_t^{(j)} = (1-t)X^{(j)}_1 + t X^{(j)}_2$, 
$\ \ \dot m_t^{(j)} = X^{(j)}_2 - X^{(j)}_1,$
\begin{equation}
G_t
= g(t) I_d := \frac{\sigma^2(1-2t)}{2(\sigma_{\min}^2 + \sigma^2 t(1-t))}\, I_d,
\end{equation}
\begin{equation}
w_j(t, z) = \alpha_j(t,z)
=
\frac{
\exp\!\left(-\frac{\|z-m_t^{(j)}\|^2}{2 c_t^2}\right)
}{
\sum_{k=1}^M 
\exp\!\left(-\frac{\|z-m_t^{(k)}\|^2}{2 c_t^2}\right)
}.
\end{equation}    
\end{proposition}

This velocity decomposes into a global linear term $G_t z$ and a local
data-dependent nonlinear term $h(t, z; \mathcal{D}_M)$ weighted by Gaussian-kernel attention. In general, different choices of the conditional path give rise to a different form of $\hat{v}_t^*$ and thus to a different sampler; see App.~\ref{app:background}. This allows us to engineer paths tailored to specific data and sampler properties.


\noindent {\bf Continuous-time memory-based sampler.}
The closed-form empirical CFM velocity $\hat v^*(t,z)$ in Proposition \ref{prop:freefm} yields
a \emph{training-free sampler} that uses the entire dataset as a memory bank.
Given an initial state $x_\tau \in \R^d$ at discrete time $\tau$, we evolve
the continuous-time state $Z_t$ according to
\begin{equation}
\label{eq:CFM-memory-ODE}
\frac{dZ_t}{dt} = G_t Z_t
+ h(t, Z_t; \mathcal{D}_M), 
\ \ 
Z_0 \sim \mathcal N(x_\tau, \sigma_{\text{min}}^2 I_d),
\end{equation}
from $t = 0$ to $t=1$ for some $\sigma_{\text{min}}>0$, to obtain an estimate  for the next state $x_{\tau+1}$. 
Iterating the one-step sampling procedure in
\eqref{eq:CFM-memory-ODE} produces a multi-step forecaster. By propagating multiple particles through the ODE, we obtain predictive ensembles for uncertainty estimation and their sample
means as point forecasts.
Integration of the full-memory field transports the smoothed empirical source mixture to the target mixture. This marginal guarantee does not ensure that forecasts from query-centered Gaussian initialization preserve transition pairings or recover the true conditional transition law.

This sampler describes a \emph{nonparametric memory-augmented dynamical system}:
the velocity field at each $(t,z)$ is computed by applying soft attention to
the stored transitions, effectively blending historical dynamics
based on proximity to the current ODE state (see also App.~\ref{subsec:ts_unified_interpretation}). 
Moreover, we can connect the velocity field with the empirical score function $\hat{s}^*(t,z)$:
\begin{equation*} 
    \hat{v}^\ast(t, z) = \sum_{j=1}^M w_j(t,z)\, \dot{m}_t(X^{(j)}) - \frac{\sigma^2(1-2t)}{2} \hat{s}^*(t, z),
\end{equation*}
where $\hat{s}^*(t,z) := \nabla_z \log \hat{p}_t(z)$.
Thus, the optimal empirical CFM velocity is the sum of a data-dependent velocity mixture and a path-induced score correction.

\noindent {\bf Trajectory replay vs. score-based correction.}
For the Gaussian bridge, let
$\Delta^{(j)}
    :=
    X_2^{(j)}-X_1^{(j)}$
denote the observed increment of transition $j$, and define
\begin{equation}
    R(t, z;\mathcal D_M)
    :=
    \sum_{j=1}^M
    \alpha_j(t,z)\Delta^{(j)},
    \label{eq:transition-replay-field}
\end{equation}
which is a state-dependent convex combination of one-step increments from the transition bank.
We use {\it trajectory replay} to refer to the resulting dynamical
mechanism in which a forecast trajectory is generated by repeatedly
retrieving and applying such observed increments as the state
evolves. This need not reproduce any single stored trajectory: the
responsibilities may vary with $(t,z)$, so a rollout may smoothly combine
or switch among transitions originating from different locations and
different training trajectories.

From Proposition \ref{prop:freefm}, we can write:
\begin{tcolorbox}[
  colback=sand,
  colframe=sandborder,
  boxrule=0.3pt,
  left=5pt,
  right=5pt,
  top=0pt,
  bottom=1pt
]
\begin{align*}
\hat v^*(t,z)
&=
\underbrace{R(t, z;\mathcal D_M)}_{\text{transition replay}}  + 
\underbrace{G_t\,\Big(z-\sum_{j=1}^M \alpha_j(t,z)\,m_t^{(j)}\Big)}_{\text{mixture score correction}},
\end{align*}
\end{tcolorbox}
i.e., the optimal empirical FM velocity field decomposes into two structurally distinct components.  

The first term is the \emph{transition replay} term, which averages observed one-step transition vectors $\Delta^{(j)}$ using responsibilities $\alpha_j(t,z)$ that measure how close $z$ is to the interpolated means $m_t^{(j)}$. Thus, it acts as a soft local nonparametric model of the
dynamics. More precisely, if the effective bandwidth $c_t \to 0$ and
the closest interpolated transition mean is unique, the responsibilities
concentrate on that transition and $R(t,z;\mathcal D_M)$ approaches
nearest-neighbor transition replay.
Thus, the trajectory replay forms a continuum between generalization and memorization: positive $c_t$ yields kernel-averaged transition dynamics, while $c_t \to 0$ recovers nearest-neighbor memorization as a limiting case.
 
The second term is a \emph{mixture score correction}, which tends to zero as
$\sigma \to 0$. Since
$\nabla_z \log \hat p_t(z)
=
-\frac{1}{c_t^2}
\left(
z-\sum_{j=1}^M \alpha_j(t,z)m_t^{(j)}
\right),$
this correction is proportional to the score of the Gaussian mixture marginal
$\hat p_t$. More precisely, it is opposite to the empirical score for
$t<1/2$, vanishes at $t=1/2$, and is aligned with the score for $t>1/2$.
Equivalently, it first pushes $z$ away from and then pulls it toward the
responsibility-weighted mean
$\sum_{j=1}^M \alpha_j(t,z)m_t^{(j)}$, reflecting the expansion and
contraction of the Gaussian bridge along the probability path.

Together, these two terms reveal that empirical FM induces a memory-based dynamical system in which transition replay is modified by a score-based correction. The parameter $\sigma$ affects both mechanisms (see App.~\ref{app_ablation_sigma} for ablation study): increasing
$\sigma$ broadens the Gaussian responsibilities in the interior of the path
and simultaneously changes the magnitude of the score correction. In particular, with $\sigma_{\min}>0$ fixed, $\sigma\to0$ gives finite-bandwidth transition replay with no score correction, rather than hard nearest-neighbor lookup; hard replay instead
requires the effective bandwidth $c_t$ itself to become small.
Thus, FreeFM is best viewed as a nonparametric memory-based model whose behavior is jointly determined by transition retrieval and the prescribed probability path.

\section{Practical Considerations}
\label{sec:practical}
In this section, we address the practical challenges associated with
the training-free sampler.

\subsection{Numerical Stiffness and Integration Schemes}

For  $t\in[0,1]$, the (spatial) Lipschitz constant of a map
$f(t,\cdot):\R^d\to\R^d$ is defined as (App.~\ref{app_lipschitzproof}):
\[
\operatorname{Lip}_z(f)(t) :=\sup_{z\neq z'}\frac{\|f(t,z)-f(t,z')\|}{\|z-z'\|}\in[0,\infty].
\]
The following spatial Lipschitz bound for the velocity field exposes a stiffness issue associated with numerical integration of the proposed ODE sampler.

\begin{proposition}[Lipschitz bound] \label{prop_grads}
Let $t \in [0,1]$ and $\mathcal{D}_M$ be given.
Assume that $\sigma > 0$ and, for all $j, t$, $\|\dot{m}_t^{(j)}\| \le R_1$ and $\|m_t^{(j)}\| \le R_m$. Then,  the $z$-Lipschitz constant of $h$  is dominated by $c_t^{-4}$, as $c_t \to 0$:
\begin{equation}
\operatorname{Lip}_z(h)(t) \leq \sup_{z \in \R^d} \|\nabla_z h(t, z; \mathcal{D}_M)\| = \mathcal{O}(c_t^{-4}).
\end{equation}
Moreover, $\operatorname{Lip}_z(\hat{v}^\ast)(t) \leq \sup_{z \in \R^d} \|\nabla_z \hat{v}^\ast\| = \mathcal{O}(c_t^{-4}).$
\end{proposition}

See also Theorem \ref{prop_grads_long} (in App.~\ref{app_lipschitzproof}) for a  detailed version and its proof.

As the effective bandwidth $c_t$ decreases, the Lipschitz bound can deteriorate: while $\|G_t\|=\mathcal{O}(c_t^{-2})$, Proposition~\ref{prop_grads} gives $\operatorname{Lip}_z(h)(t)=\mathcal{O}(c_t^{-4})$, so the state dependence of the posterior weights can induce stronger spatial sensitivity than the linear term alone. Near the endpoints, $c_t$ can become small as $\sigma_{\min}\to0$; indeed, for $\sigma_{\min}=0$, $\|G_t\|=\frac{|1-2t|}{2t(1-t)}\sim\frac{1}{2t}$ as $t\downarrow0$, with analogous behavior near $t=1$.

Stiffness can be controlled via the variance floor $\sigma_{\min}$. Choosing  $\sigma_{\min}$ involves a trade-off: smaller values improve endpoint matching but increase stiffness, while larger values improve numerical stability at the cost of smoothing. In practice, we use a small step size and select $\sigma$ and
$\sigma_{\min}$ by validation. Algorithm \ref{alg:sampler} in App.~\ref{app:algorithm} describes the ODE sampler (FreeFM) for probabilistic prediction. We note that while an exponential Euler scheme fits naturally with the structure of the ODE, 
it does not remove the state dependence of the posterior-weighted term $h$, whose spatial Lipschitz bound can deteriorate as $\mathcal O(c_t^{-4})$. Thus, a tuned $\sigma_{\min}$ and sufficiently small step size may still be necessary.

\subsection{Scalability and Approximation Schemes}

Evaluating the empirical CFM velocity 
$\hat v_t^*(z) := \hat{v}(t, z)$ at a given $(t,z)$  requires computing $M$ responsibilities
\(
\alpha_j(t,z)
\propto
\exp\!\big(-\|z-m_t^{(j)}\|^2/(2c_t^2)\big).
\)
Thus, the naive cost of one velocity evaluation is $O(Md)$, which becomes prohibitively costly for large transition sets. 
A simple solution to mitigate the cost is top--$R$ posterior truncation.
The responsibilities $\alpha_j(t,z)$ often concentrate sharply near the
transitions whose interpolated locations $m_t^{(j)}$ are closest to $z$.
Let $\mathcal{I}_R(z,t)$ denote the indices of the $R$ largest softmax weights:
$\mathcal{I}_R(z,t)
=
\operatorname*{arg\,topR}_{j}
\; \alpha_j(t,z).$ Define the truncated velocity estimator
\begin{equation}
\hat v_{t,R}(z)
:=
G_t z
+
\frac{
\sum_{j\in\mathcal{I}_R(z,t)}
    \alpha_j(t,z)\, y_j(t) 
}{
\sum_{k\in\mathcal{I}_R(z,t)}\alpha_k(t, z)
}.
\end{equation}
The following result quantifies the truncation error. 
\begin{proposition}[Truncation error] \label{prop_truncate} 
Let $t \in [0,1]$ and $\mathcal{D}_M$ be given. 
Suppose that  $\|y_j(t)\|\le C$ for all $j, t$ for some constant $C>0$. Then
\[
\big\|
\hat v_t^*(z) - \hat v_{t,R}(z)
\big\|
\;\le\;
2 C 
\bigg(
1 - \sum_{j\in\mathcal{I}_R(z,t)}\alpha_j(t,z)
\bigg).
\]
\end{proposition}
See App.~\ref{proof_truncate} for a proof. 
When the responsibilities are concentrated, a small $R$ captures most of their mass, and Proposition~\ref{prop_truncate} provides a 
bound on the resulting velocity error. 
Top-$R$ truncation reduces the number of components used in the final
weighted aggregation but does not remove the full memory scan. Our
scan-and-select implementation costs
$\mathcal O(Md+M+Rd)=\mathcal O(Md+Rd)$
per query: $\mathcal O(Md)$ for distances, $\mathcal O(M)$ for selection, and $\mathcal O(Rd)$ for
aggregation.

\section{Empirical Results}
\label{sec:exp}

In this section, we validate the theoretical insight and evaluate the effectiveness of FreeFM on nonlinear dynamics forecasting tasks (see App.~\ref{app:experiments} for details and additional results). 
While the framework is applicable in principle more broadly to sequential data, we focus on this setting to provide a clear and controlled evaluation. 
We provide empirical results on real-world datasets, covering  low- and high-dimensional settings, as well as the wall-clock cost comparison in App.~\ref{app_realworld}.

\begin{figure}[!t]
    \centering
    \begin{subfigure}{0.48\textwidth}
        \centering
        \includegraphics[width=\textwidth]{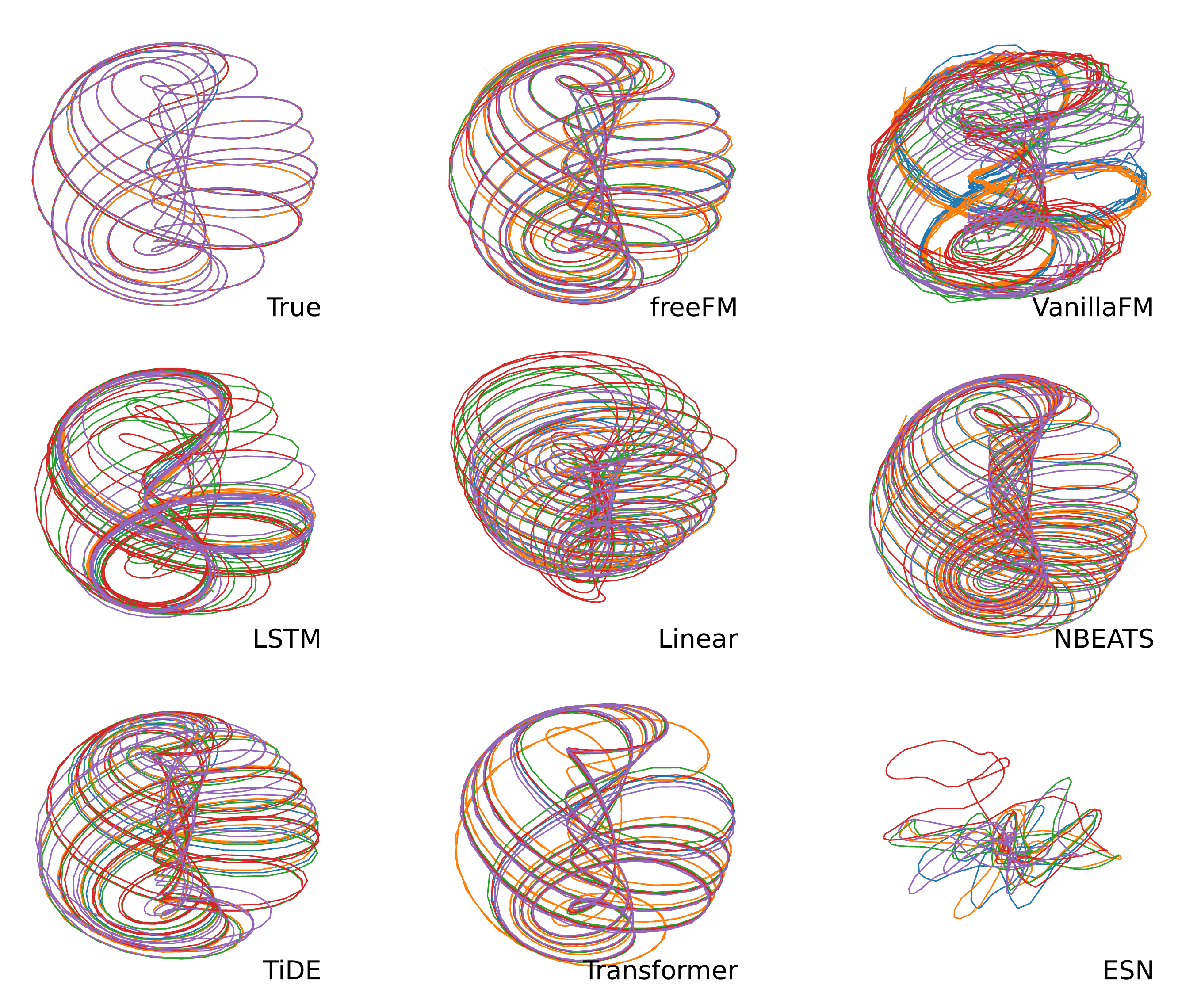}
        \caption{Forecast Trajectory for Aizawa Attractor}
        \label{fig:ConditionalForecastSub1}
    \end{subfigure}
    \hfill
    \begin{subfigure}{0.5\textwidth}
        \centering
        \includegraphics[width=\textwidth]{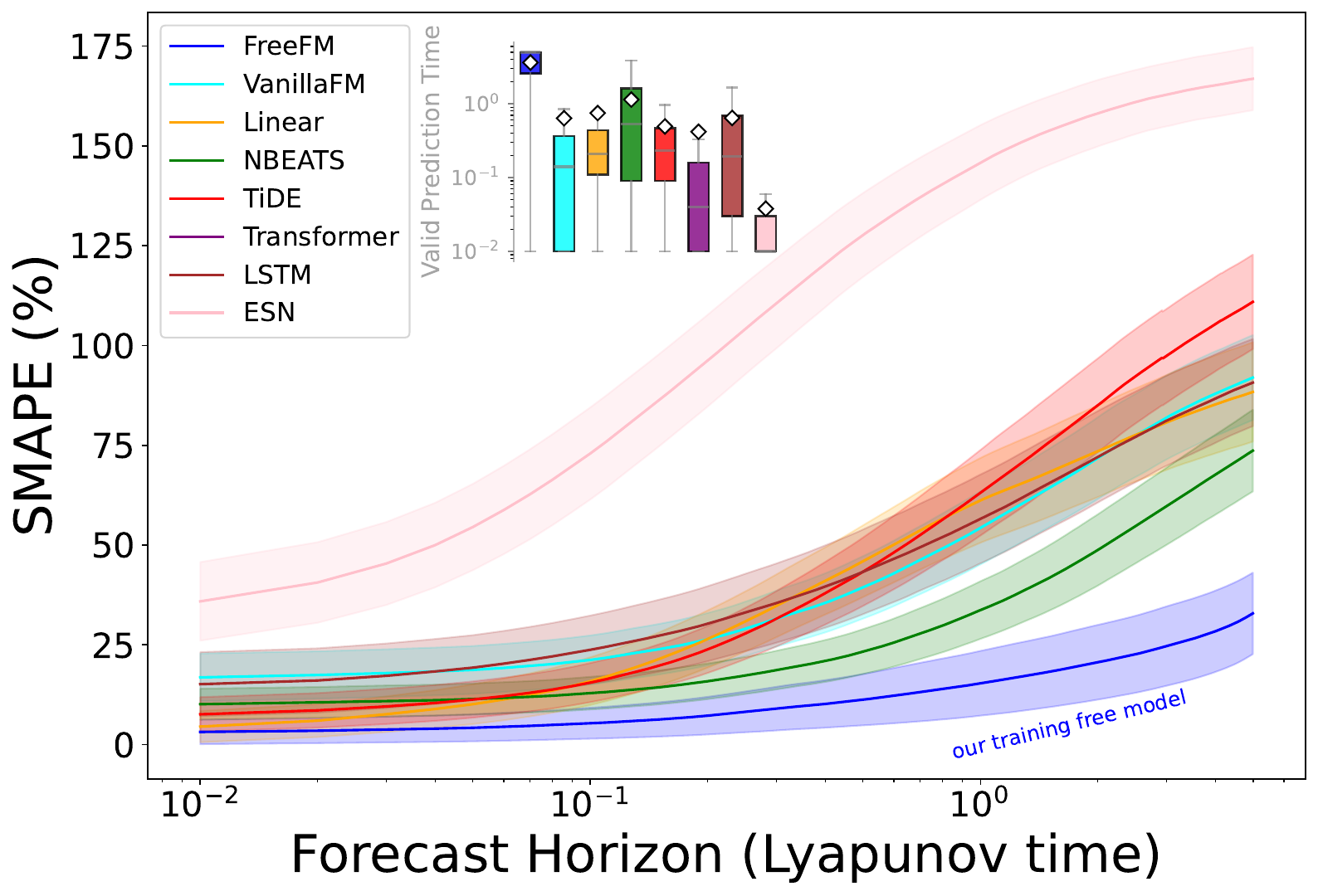}
        \caption{sMAPE and VPT Comparison}
        \label{fig:ConditionalForecastSub2}
    \end{subfigure}
    \caption{\textbf{Conditional Forecast.} (a) Examples of forecast trajectories generated by FreeFM and baseline models  for the Aizawa attractor. Each trajectory originates from a different initial condition. (b) sMAPE and VPT of  forecast results from FreeFM and baseline models. Shaded regions indicate the mean $\pm0.3$ standard deviation across the 135 dynamical systems, each with 20 trajectories originating from randomly sampled initial conditions.}
    \label{fig:ConditionalForecast}
\end{figure}

\noindent {\bf Dataset.}
We evaluate on Dysts~\citep{gilpin2021chaos,gilpin2023modelscaleversusdomain}, which is a benchmark comprising 135 low-dimensional chaotic systems with varying state dimensions. The systems are described by ODEs that are aligned with respect to dominant timescales and integration procedures.

The sensitivity to initial conditions of a chaotic system is quantified by its largest Lyapunov exponent $\lambda$, which measures how quickly nearby trajectories diverge or converge \citep{ott2002chaos}. 
Stable systems or systems approaching periodic orbits have zero or negative Lyapunov exponents, while chaotic systems have positive Lyapunov exponents, implying that trajectories diverge exponentially with small changes in initial conditions. 
The Lyapunov time is the characteristic timescale of predictability, defined as the time required for an initial error to grow by a factor of $e$ ($\tau\equiv\lambda^{-1}$). 
Generally, after 3-5 Lyapunov times, the system becomes effectively unpredictable. 
For our experiments, following \citep{zhang2024zero}, we integrate all systems using an implicit Runge-Kutta scheme and uniformly downsample all time series. Unlike~\citep{zhang2024zero}, we choose a finer granularity of 100 time points per Lyapunov time~$\tau$.

\noindent {\bf Baselines.}
\begin{figure}[!t]
    \centering
    \begin{subfigure}{0.6\textwidth}
        \centering
        \includegraphics[width=\textwidth]{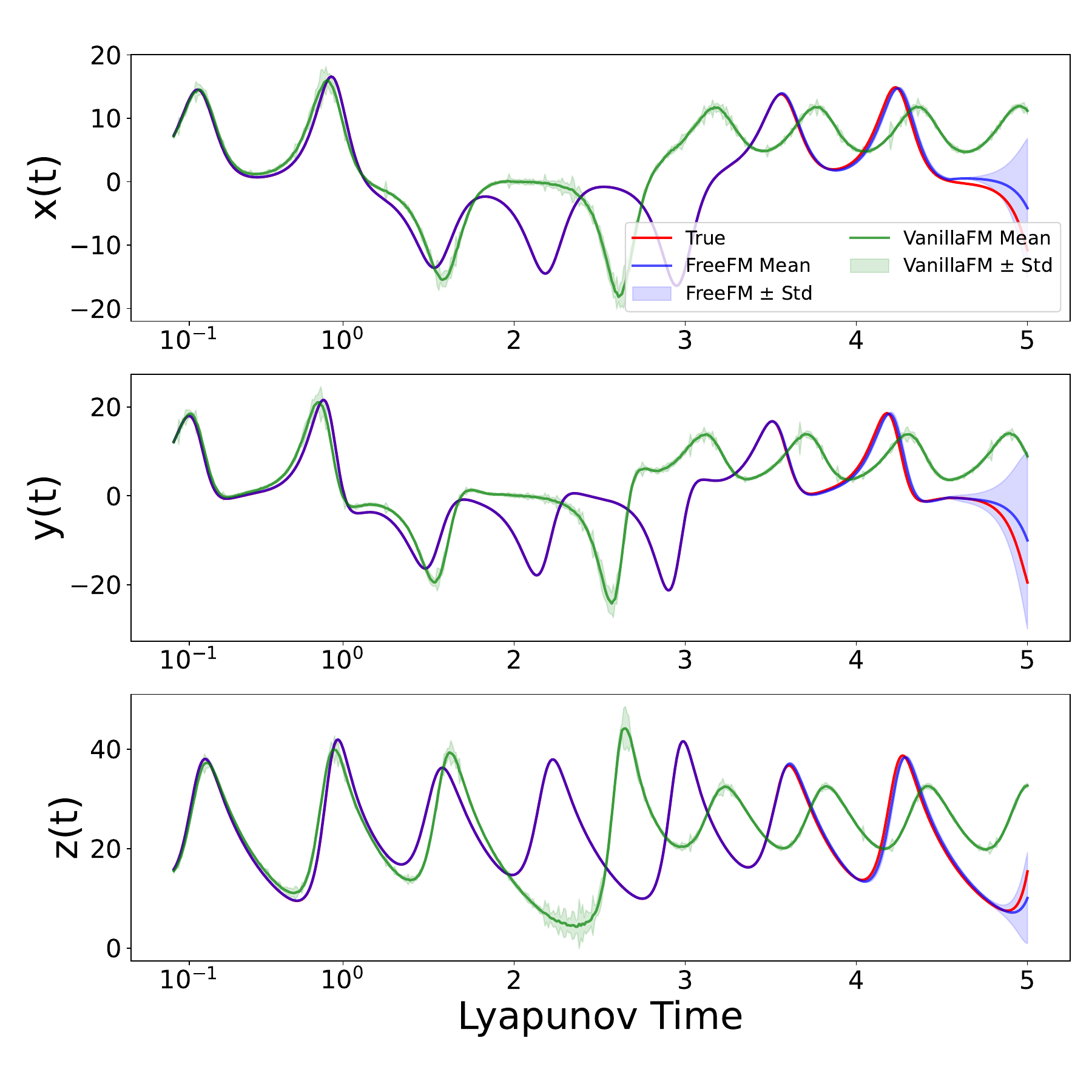}
   
    \end{subfigure}
    \begin{subfigure}{0.6\textwidth}
        \centering
        \includegraphics[width=1\textwidth]{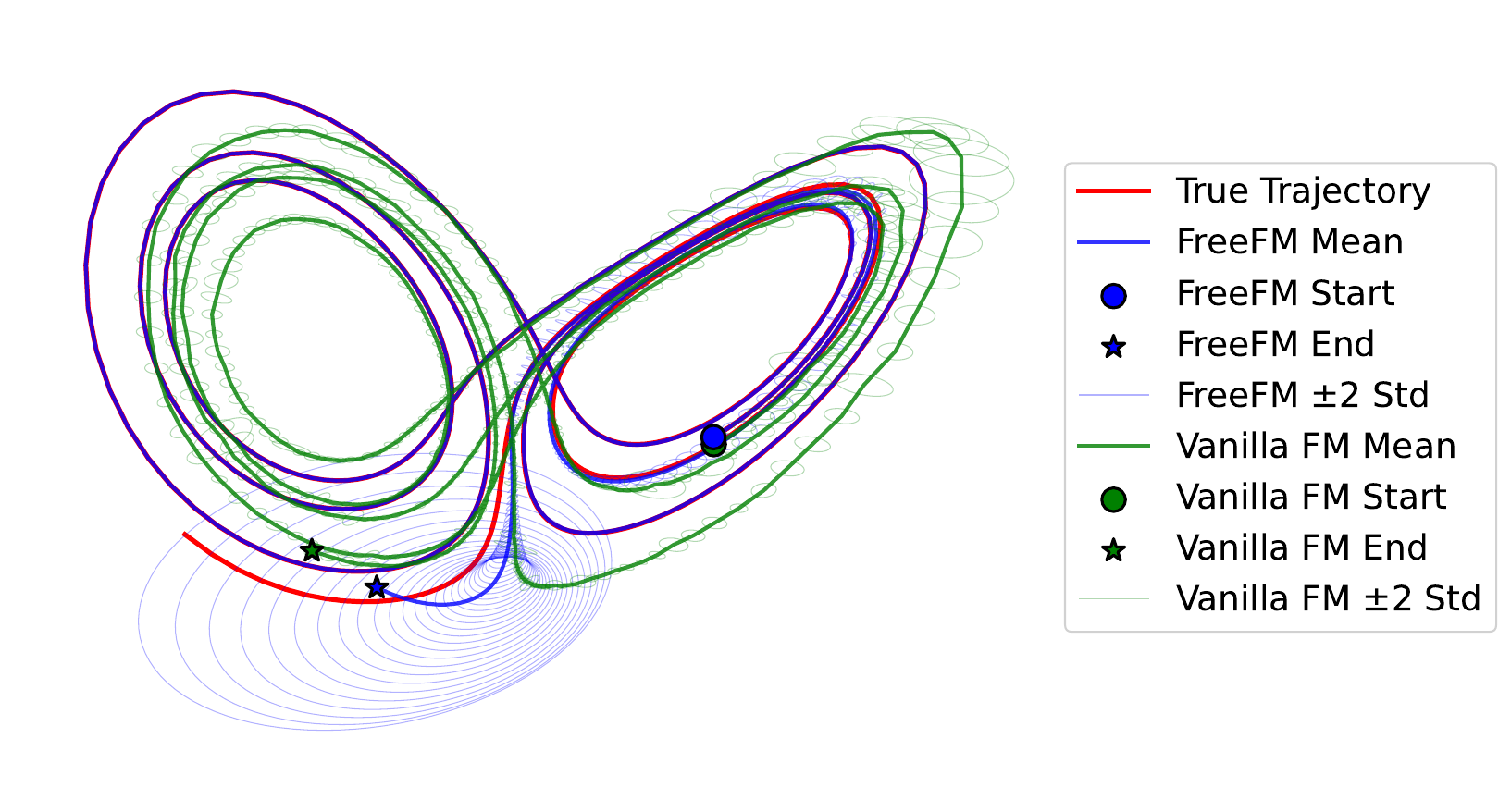}
    
    \end{subfigure}
\vspace{-0.5cm}
    \caption{\textbf{Probabilistic Forecast for Lorenz-63.} Forecasts generated by FreeFM and the trained VanillaFM baseline. The upper plots show ensemble means with pointwise standard deviation bands. The lower plot shows the corresponding two-standard-deviation tubes around the ensemble means. Ensemble statistics are computed from 50 Monte Carlo forecasts. }
    \label{fig:ProbabilisticForecast}
\end{figure}
For the conditional forecast experiment and the long term attractor reconstruction experiment, we select seven widely used models in dynamical systems forecasting as our baselines: 
(1) N-BEATS~\citep{oreshkin2020nbeatsneuralbasisexpansion}, a deep neural net that uses interpretable basis expansion to decompose predictions into trend and seasonality components; (2) TiDE~\citep{das2024longtermforecastingtidetimeseries}, an efficient MLP-based encoder-decoder model; (3) Echo State Networks (ESNs)~\citep{jaeger2004harnessing}, a reservoir computing model that has been shown to perform well on chaotic dynamical systems; (4) Transformer~\citep{vaswani2017attention}; (5) LSTM~\citep{hochreiter1997long}; (6) linear regression, a simple baseline that serves as a lower bound for model performance; and (7) Vanilla FM: a neural FM model trained on the same prescribed probability path as FreeFM.

To separate generic transition retrieval from the continuous-time sampler structure
in  FreeFM, App.~\ref{app_more} reports
 comparisons with two nonparametric baselines: a matched kernel conditional estimator (KCE) and the
kernel analog forecaster (KAF) \citep{zhao2016analog}. In addition to Vanilla FM, we further compare FreeFM with another neural FM baseline (TFM \citep{zhang2024trajectory}) in App.~\ref{app:tfm}.

\subsection{Conditional Forecasting}
\label{subsec:ConditionalForecast}

\begin{figure}[!t]
    \centering
    \begin{subfigure}{0.49\textwidth}
        \centering
        \includegraphics[width=\textwidth]{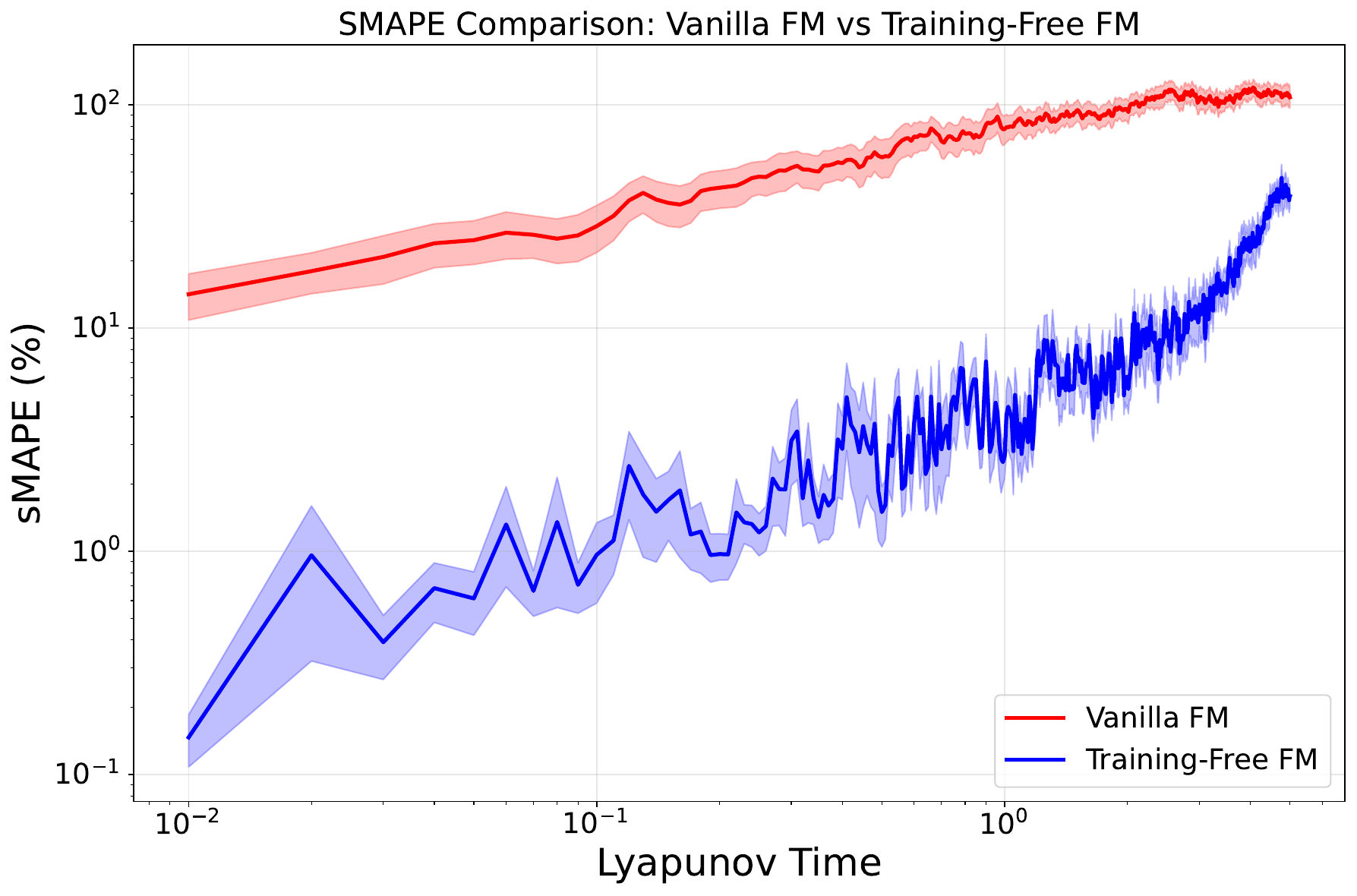}
        \caption{sMAPE Comparison}
    \end{subfigure}
    \hfill
    \vspace{+0.2cm}
    \begin{subfigure}{0.49\textwidth}
        \centering
        \includegraphics[width=\textwidth]{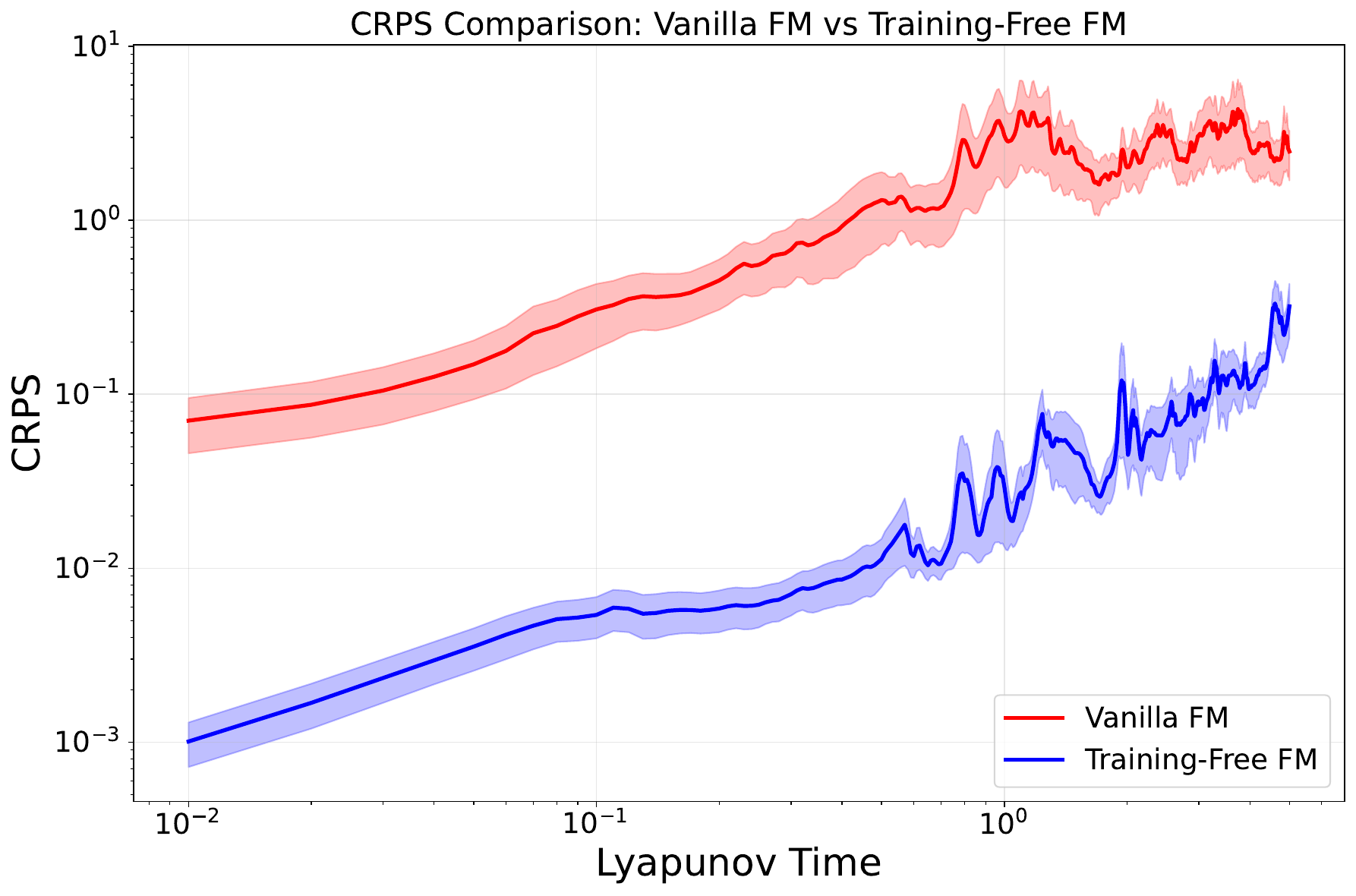}
        \caption{CRPS Comparison}
    \end{subfigure}
    \caption{\textbf{Probabilistic Forecast for Dysts.}   (a) sMAPE of forecasts generated by FreeFM and the trained VanillaFM baseline. (b) CRPS of the corresponding probabilistic forecasts. Shaded bands indicate the mean $\pm 0.3$ standard deviation across systems. Results are aggregated over 135 dynamical systems, with 20 randomly sampled initial conditions per system and 50 Monte Carlo forecasts per initial condition.}
    \label{fig:ProbabilisticForecastSub3}
\end{figure}

We generate 20 trajectories for each of the 135 chaotic systems, with a length of 812 time points and a granularity of 100 points per Lyapunov time. 
This corresponds to approximately 8.12 Lyapunov times per trajectory, placing the system firmly in the chaotic regime where sensitive dependence on initial conditions is fully manifested. 
Since chaotic systems typically become effectively unpredictable after 3--5 Lyapunov times, our trajectories are long enough to exhibit the characteristic exponential divergence of nearby trajectories, rather than operating in a regime where chaotic behavior has not yet developed. Each trajectory originates from a different, randomly sampled initial condition. We evaluate our model and the baselines at a prediction horizon of $5\tau$, corresponding to 500 time points. We used the first 312 time points as an initial condition to forecast the remaining 500 time points.

We use the symmetric mean absolute percentage error (sMAPE) and the valid prediction time (VPT) (see App.~\ref{subsec:EvaluationMetrics}) to evaluate the forecast results.  The results are presented in Fig.~\ref{fig:ConditionalForecast}. Compared to all fully-trained baseline models, our training-free model outperforms the  baselines on average across the 135 chaotic systems. In particular, our model achieves an average VPT greater than 1, exceeding the highest VPT among all baselines, even with a relatively tight threshold for VPT computation.

\subsection{Probabilistic Forecasting}

Next, we demonstrate that our model performs well in probabilistic forecasting. A key advantage of generative models  is that they naturally incorporate uncertainty into predictions, making probabilistic forecasting a well-suited task. We compare our model with the fully trained vanilla FM model \citep{lipman2022flow}. For the data settings, we follow the conditional generation setting to generate 20 trajectories of length 812 with a granularity of 100 points per Lyapunov time from 135 chaotic systems in Dysts. The trajectories are divided into 312 observed time points and 500 testing points. For each trajectory, we generate 50 different predictions for probabilistic evaluation. 

To evaluate the quality of probabilistic forecasts, we use sMAPE and the Continuous Ranked Probability Score (CRPS).
Examples and results of the probabilistic forecast are presented in Fig.~\ref{fig:ProbabilisticForecast} and Fig~\ref{fig:ProbabilisticForecastSub3}. 
Our training-free model outperforms the fully trained vanilla flow matching models in terms of CRPS. 
From Fig.~\ref{fig:ProbabilisticForecast}, we can see that our training-free models still have good forecast quality until 3 Lyapunov times, and give reasonable probabilistic forecast result after 4 Lyapunov times. We  evaluate the long term forecasting ability of our training-free model in Sec.~\ref{subsec:LongTermAttractorReconstruction}. 

\subsection{Long Term Attractor Reconstruction}
\label{subsec:LongTermAttractorReconstruction}

Chaotic systems become effectively unpredictable after several Lyapunov times. 
Here, we evaluate our model's ability to reconstruct the attractors in the long term, beyond where point forecasts fail. 
We quantify this ability using the correlation dimension, a non-parametric measure that characterizes the fractal dimension of strange attractors in chaotic dynamical systems~\citep{grassberger1983measuring}. 
The long-term dynamics of chaotic systems evolve on a strange attractor, and the correlation dimension characterizes how the attractor fills the phase space by measuring the scaling behavior of nearby point pairs. 
We compute the correlation dimension for both predicted and ground-truth trajectories, and we evaluate the results using the root mean square error (RMSE) between them. 
Following prior studies~\citep{hess2023generalizedteacherforcinglearning,göring2024outofdomaingeneralizationdynamicalsystems}, we also compute the Kullback–Leibler (KL) divergence between the true and reconstructed attractors~\citep{kullback1951information}.

The results are shown in Fig.~\ref{fig:LongTermAttractorReconstruction}. In terms of both the correlation dimension and the KL divergence, our training-free model outperforms the baselines.

\begin{figure}[!h]
    \centering
    \begin{subfigure}{0.58\textwidth}
        \centering
        \includegraphics[width=\textwidth]{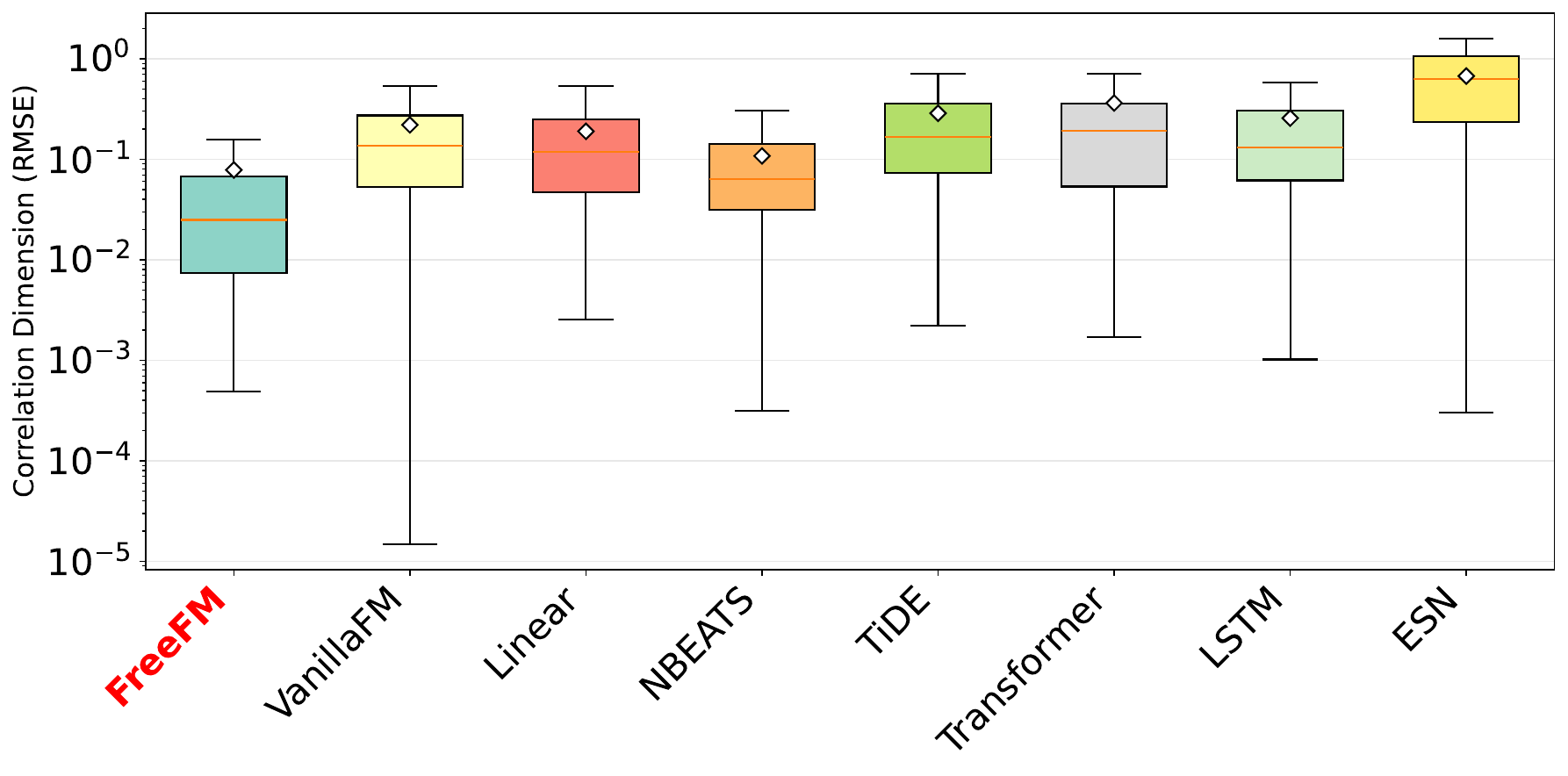}
        \label{fig:LongTermAttractorReconstructionSub1}
    \end{subfigure}
    \begin{subfigure}{0.58\textwidth}
        \centering
        \includegraphics[width=\textwidth]{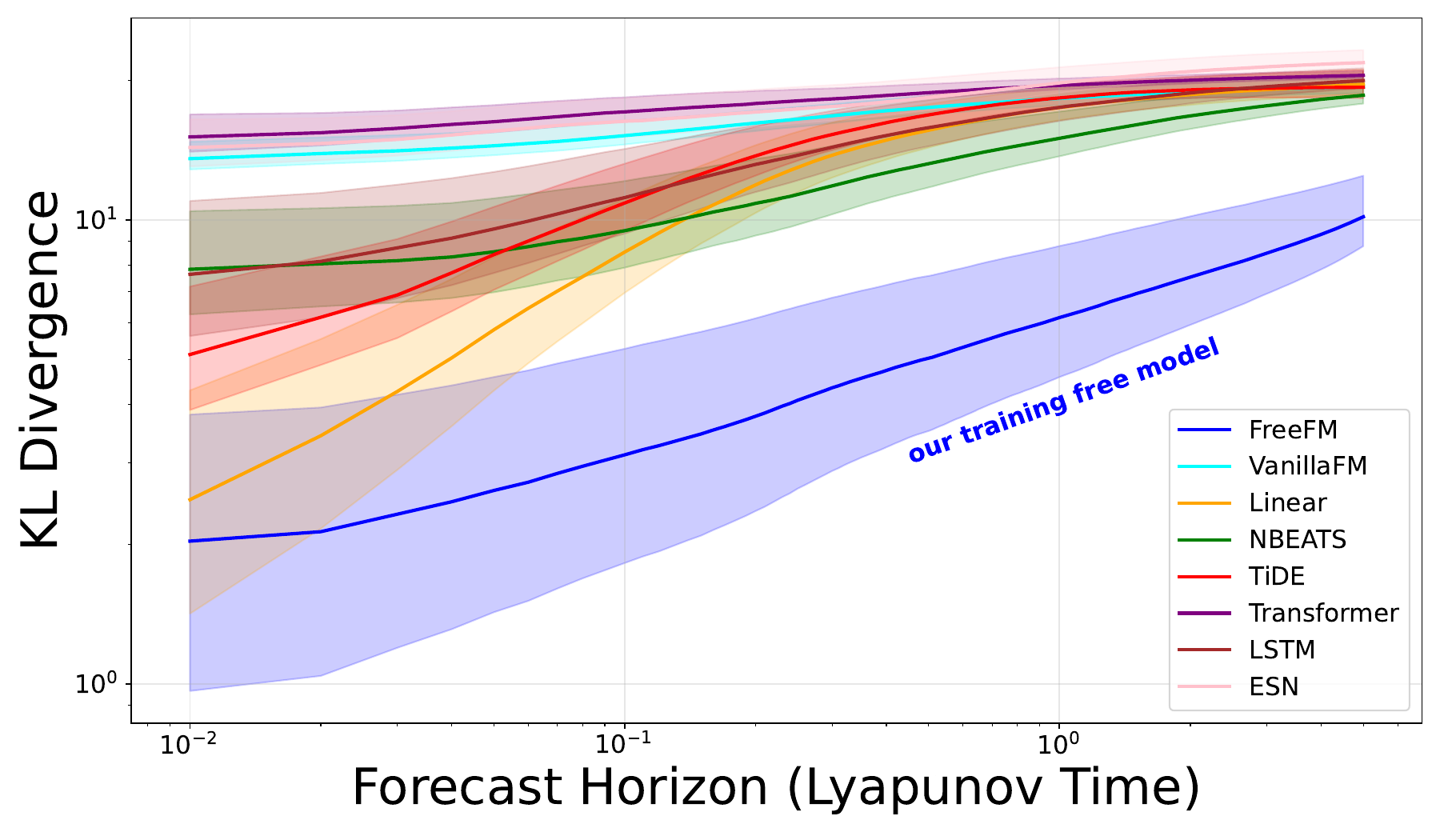}
        \label{fig:LongTermAttractorReconstructionSub2}
    \end{subfigure}
        \vspace{-0.5cm}
    \caption{\textbf{Attractor Reconstruction.} 
    The upper plot shows boxplots across systems of the RMSE between the correlation dimensions estimated from the predicted and ground truth trajectories (white diamonds mark the mean). The lower plot shows the KL divergence between the true and reconstructed attractors for FreeFM and the baseline models; shaded bands indicate the mean $\pm 0.3$ standard deviation across systems. Results are aggregated over 135 dynamical systems, each with 20 randomly sampled initial conditions. For FreeFM and VanillaFM, 50 Monte Carlo forecasts are generated per initial condition.}

    \label{fig:LongTermAttractorReconstruction}
\end{figure}

\section{Conclusion}
We studied what velocity field a  FM model learns when applied to  sequential data in the empirical setting.
Specializing the general CFM posterior-mixture identity to a transition bank yields a nonparametric memory-augmented dynamical system. For the Gaussian bridge used here, the optimal empirical velocity decomposes into a transition replay term and a score-based correction.
More generally, any conditional path with tractable densities and velocities induces an empirical transition dynamics. FreeFM provides an analytically tractable instance and an interpretable empirical reference for neural FM models trained on the same objective.
FreeFM is most effective when test states remain within well-sampled regions of a coherent local transition geometry. Sparse coverage, high intrinsic dimension, distribution shift, and long autoregressive horizons weaken distance-based retrieval, motivating hybrid methods that combine empirical transition memory with scalable learned representations.

\section*{Acknowledgments}
The computations were enabled by resources provided by the National Academic Infrastructure for Supercomputing in Sweden (NAISS), funded by the Swedish Research Council (NAISS 2026/3-562). 
SHL would like to acknowledge support from the Wallenberg Initiative on Networks and Quantum Information (WINQ) and the Swedish Research Council (VR/2021-03648).
NBE would like to acknowledge support from the U.S. Department of Energy, Office of Science, Office of Advanced Scientific Computing Research, EXPRESS: 2025 Exploratory Research for Extreme-Scale Science program, under Contract
Number DE-AC02-05CH11231 at Lawrence Berkeley National Laboratory. 
MWM would like to acknowledge support from DARPA, NSF, the DOE Competitive Portfolios grant, and the DOE SciGPT grant.

\bibliography{reference}

@misc{hess2023generalizedteacherforcinglearning,
      title={Generalized Teacher Forcing for Learning Chaotic Dynamics}, 
      author={Florian Hess and Zahra Monfared and Manuel Brenner and Daniel Durstewitz},
      year={2023},
      eprint={2306.04406},
      archivePrefix={arXiv},
      primaryClass={cs.LG},
      url={https://arxiv.org/abs/2306.04406}, 
}

@article{tire2024retrieval,
  title={Retrieval augmented time series forecasting},
  author={Tire, Kutay and Taga, Ege Onur and Ildiz, Muhammed Emrullah and Oymak, Samet},
  journal={arXiv preprint arXiv:2411.08249},
  year={2024}
}

@article{zhao2016analog,
  title={Analog forecasting with dynamics-adapted kernels},
  author={Zhao, Zhizhen and Giannakis, Dimitrios},
  journal={Nonlinearity},
  volume={29},
  number={9},
  pages={2888--2939},
  year={2016},
  publisher={IOP Publishing}
}

@article{koehler2001asymmetry,
  title={The asymmetry of the s{A}{P}{E} measure and other comments on the {M}3-competition},
  author={Koehler, Anne B},
  journal={International Journal of Forecasting},
  volume={17},
  number={4},
  pages={570--574},
  year={2001},
  publisher={ELSEVIER SCIENCE BV PO BOX 211, 1000 AE AMSTERDAM, NETHERLANDS}
}

@article{kullback1951information,
  title={On information and sufficiency},
  author={Kullback, Solomon and Leibler, Richard A},
  journal={The Annals of Mathematical Statistics},
  volume={22},
  number={1},
  pages={79--86},
  year={1951},
  publisher={JSTOR}
}

@inproceedings{gilpin2021chaos,
title={Chaos as an interpretable benchmark for forecasting and data-driven modelling},
author={William Gilpin},
booktitle={Thirty-fifth Conference on Neural Information Processing Systems Datasets and Benchmarks Track (Round 2)},
year={2021},
}

@misc{brinke2025flowmatchinggeometrictrajectory,
      title={Flow Matching for Geometric Trajectory Simulation}, 
      author={Kiet Bennema ten Brinke and Koen Minartz and Vlado Menkovski},
      year={2025},
      eprint={2505.18647},
      archivePrefix={arXiv},
      primaryClass={cs.LG},
      url={https://arxiv.org/abs/2505.18647}, 
}

@misc{göring2024outofdomaingeneralizationdynamicalsystems,
      title={Out-of-Domain Generalization in Dynamical Systems Reconstruction}, 
      author={Niclas Göring and Florian Hess and Manuel Brenner and Zahra Monfared and Daniel Durstewitz},
      year={2024},
      eprint={2402.18377},
      archivePrefix={arXiv},
      primaryClass={cs.LG},
      url={https://arxiv.org/abs/2402.18377}, 
}

@misc{gilpin2023modelscaleversusdomain,
      title={Model scale versus domain knowledge in statistical forecasting of chaotic systems}, 
      author={William Gilpin},
      year={2023},
      eprint={2303.08011},
      archivePrefix={arXiv},
      primaryClass={cs.LG},
      url={https://arxiv.org/abs/2303.08011}, 
}

@article{grassberger1983measuring,
  title={Measuring the strangeness of strange attractors},
  author={Grassberger, Peter and Procaccia, Itamar},
  journal={Physica D: Nonlinear Phenomena},
  volume={9},
  number={1-2},
  pages={189--208},
  year={1983},
  publisher={Elsevier}
}

@article{hochreiter1997long,
  title={Long short-term memory},
  author={Hochreiter, Sepp and Schmidhuber, J{\"u}rgen},
  journal={Neural Computation},
  volume={9},
  number={8},
  pages={1735--1780},
  year={1997},
  publisher={MIT press}
}

@article{vaswani2017attention,
  title={Attention is all you need},
  author={Vaswani, Ashish and Shazeer, Noam and Parmar, Niki and Uszkoreit, Jakob and Jones, Llion and Gomez, Aidan N and Kaiser, {\L}ukasz and Polosukhin, Illia},
  journal={Advances in Neural Information Processing Systems},
  volume={30},
  year={2017}
}

@article{herzen2022darts,
  title={Darts: User-friendly modern machine learning for time series},
  author={Herzen, Julien and L{\"a}ssig, Francesco and Piazzetta, Samuele Giuliano and Neuer, Thomas and Tafti, L{\'e}o and Raille, Guillaume and Van Pottelbergh, Tomas and Pasieka, Marek and Skrodzki, Andrzej and Huguenin, Nicolas and others},
  journal={Journal of Machine Learning Research},
  volume={23},
  number={124},
  pages={1--6},
  year={2022}
}

@article{li2026kinetic,
  title={A Kinetic-Energy Perspective of Flow Matching},
  author={Li, Ziyun and Hu, Huancheng and Lim, Soon Hoe and Li, Xuyu and Gao, Fei and Diao, Enmao and Ding, Zezhen and Vazirgiannis, Michalis and Bostrom, Henrik},
  journal={arXiv preprint arXiv:2602.07928},
  year={2026}
}

@inproceedings{lai2018modeling,
  title={Modeling long-and short-term temporal patterns with deep neural networks},
  author={Lai, Guokun and Chang, Wei-Cheng and Yang, Yiming and Liu, Hanxiao},
  booktitle={The 41st International ACM SIGIR Conference on Research \& Development in Information Retrieval},
  pages={95--104},
  year={2018}
}

@article{jaeger2004harnessing,
  title={Harnessing nonlinearity: Predicting chaotic systems and saving energy in wireless communication},
  author={Jaeger, Herbert and Haas, Harald},
  journal={Science},
  volume={304},
  number={5667},
  pages={78--80},
  year={2004},
  publisher={American Association for the Advancement of Science}
}

@misc{das2024longtermforecastingtidetimeseries,
      title={Long-term Forecasting with {T}i{D}E: Time-series Dense Encoder}, 
      author={Abhimanyu Das and Weihao Kong and Andrew Leach and Shaan Mathur and Rajat Sen and Rose Yu},
      year={2024},
      eprint={2304.08424},
      archivePrefix={arXiv},
      primaryClass={stat.ML},
      url={https://arxiv.org/abs/2304.08424}, 
}

@misc{oreshkin2020nbeatsneuralbasisexpansion,
      title={{N-BEATS}: Neural basis expansion analysis for interpretable time series forecasting}, 
      author={Boris N. Oreshkin and Dmitri Carpov and Nicolas Chapados and Yoshua Bengio},
      year={2020},
      eprint={1905.10437},
      archivePrefix={arXiv},
      primaryClass={cs.LG},
      url={https://arxiv.org/abs/1905.10437}, 
}

@InProceedings{godahewa2021monash,
  author = "Godahewa, Rakshitha and Bergmeir, Christoph and Webb, Geoffrey I. and Hyndman, Rob J. and Montero-Manso, Pablo",
  title = "Monash Time Series Forecasting Archive",
  booktitle = "Neural Information Processing Systems Track on Datasets and Benchmarks",
  year = "2021"
}

@article{brunick2013mimicking,
  title={Mimicking an {I}t{\^o} process by a solution of a stochastic differential equation},
  author={Brunick, Gerard and Shreve, Steven}, 
  volume = {23},
    journal = {The Annals of Applied Probability},
    number = {4},
    publisher = {Institute of Mathematical Statistics},
    pages = {1584 -- 1628},
    year = {2013},
    doi = {10.1214/12-AAP881},
    }

@article{dupire1994pricing,
  title={Pricing with a smile},
  author={Dupire, Bruno and others},
  journal={Risk},
  volume={7},
  number={1},
  pages={18--20},
  year={1994}
}

@phdthesis{bentata2012markovian,
  title={Markovian Projection of Stochastic Processes},
  author={Bentata, Amel},
  year={2012},
  school={Universit{\'e} Pierre et Marie Curie-Paris VI}
}

@article{hopfield1984neurons,
  title={Neurons with graded response have collective computational properties like those of two-state neurons.},
  author={Hopfield, John J},
  journal={Proceedings of the National Academy of Sciences},
  volume={81},
  number={10},
  pages={3088--3092},
  year={1984}
}

@article{gu2021efficiently,
  title={Efficiently modeling long sequences with structured state spaces},
  author={Gu, Albert and Goel, Karan and R{\'e}, Christopher},
  journal={arXiv preprint arXiv:2111.00396},
  year={2021}
}

@article{marzouk2024distribution,
  title={Distribution learning via neural differential equations: a nonparametric statistical perspective},
  author={Marzouk, Youssef and Ren, Zhi Robert and Wang, Sven and Zech, Jakob},
  journal={Journal of Machine Learning Research},
  volume={25},
  number={232},
  pages={1--61},
  year={2024}
}

@article{benton2023error,
  title={Error bounds for flow matching methods},
  author={Benton, Joe and Deligiannidis, George and Doucet, Arnaud},
  journal={arXiv preprint arXiv:2305.16860},
  year={2023}
}

@article{tsimpos2025optimal,
  title={Optimal Scheduling of Dynamic Transport},
  author={Tsimpos, Panos and Ren, Zhi and Zech, Jakob and Marzouk, Youssef},
  journal={arXiv preprint arXiv:2504.14425},
  year={2025}
}

@book{durrett2019probability,
  title={Probability: Theory and Examples},
  author={Durrett, Rick},
  edition={5},
  year={2019},
  publisher={Cambridge University Press}
}

@book{hartman2002ordinary,
  title={Ordinary Differential Equations},
  author={Hartman, Philip},
  year={2002},
  publisher={SIAM}
}

@article{harvey2012kernel,
  title={Kernel density estimation for time series data},
  author={Harvey, Andrew and Oryshchenko, Vitaliy},
  journal={International Journal of Forecasting},
  volume={28},
  number={1},
  pages={3--14},
  year={2012},
  publisher={Elsevier}
}

@misc{piterbarg2006markovian,
  title={Markovian projection method for volatility calibration},
  author={Piterbarg, Vladimir},
  year={2006},
  publisher={SSRN}
}

@article{duong2025coarse,
  title={Coarse graining of stochastic differential equations: averaging and projection method},
  author={Duong, Manh Hong and Hartmann, Carsten and Ottobre, Michela},
  journal={arXiv preprint arXiv:2506.14939},
  year={2025}
}

@book{lafon2004diffusion,
  title={Diffusion Maps and Geometric Harmonics},
  author={Lafon, St{\'e}phane S},
  year={2004},
  publisher={Yale University}
}

@article{jones2024diffusion,
  title={Diffusion geometry},
  author={Jones, Iolo},
  journal={arXiv preprint arXiv:2405.10858},
  year={2024}
}

@book{bakry2013analysis,
  title={Analysis and Geometry of {M}arkov Diffusion Operators},
  author={Bakry, Dominique and Gentil, Ivan and Ledoux, Michel},
  volume={348},
  year={2013},
  publisher={Springer Science \& Business Media}
}

@inproceedings{zhouerror,
  title={An Error Analysis of Flow Matching for Deep Generative Modeling},
  author={Zhou, Zhengyu and Liu, Weiwei},
  booktitle={Forty-second International Conference on Machine Learning}
}

@article{nadaraya1964estimating,
  title={On estimating regression},
  author={Nadaraya, Elizbar A},
  journal={Theory of Probability \& Its Applications},
  volume={9},
  number={1},
  pages={141--142},
  year={1964},
  publisher={SIAM}
}

@article{theodoropoulos2025momentum,
  title={Momentum Multi-Marginal {S}chr{\"o}dinger Bridge Matching},
  author={Theodoropoulos, Panagiotis and Saravanos, Augustinos D and Theodorou, Evangelos A and Liu, Guan-Horng},
  journal={arXiv preprint arXiv:2506.10168},
  year={2025}
}

@article{gruver2023large,
  title={Large language models are zero-shot time series forecasters},
  author={Gruver, Nate and Finzi, Marc and Qiu, Shikai and Wilson, Andrew G},
  journal={Advances in Neural Information Processing Systems},
  volume={36},
  pages={19622--19635},
  year={2023}
}

@article{lacour2008nonparametric,
  title={Nonparametric estimation of the stationary density and the transition density of a {M}arkov chain},
  author={Lacour, Claire},
  journal={Stochastic Processes and their Applications},
  volume={118},
  number={2},
  pages={232--260},
  year={2008},
  publisher={Elsevier}
}

@article{zhao2023data,
  title={Data-driven probability density forecast for stochastic dynamical systems},
  author={Zhao, Meng and Jiang, Lijian},
  journal={Journal of Computational Physics},
  volume={492},
  pages={112422},
  year={2023},
  publisher={Elsevier}
}

@article{roussas1969nonparametric,
  title={Nonparametric estimation of the transition distribution function of a {M}arkov process},
  author={Roussas, George G},
  journal={The Annals of Mathematical Statistics},
  pages={1386--1400},
  year={1969},
  publisher={JSTOR}
}

@book{chan2013chaos,
  title={Chaos: A Statistical Perspective},
  author={Chan, Kung-Sik and Tong, Howell},
  year={2013},
  publisher={Springer Science \& Business Media}
}

@article{vogt2012nonparametric,
  title={Nonparametric regression for locally stationary time series},
  author={Vogt, Michael},
  year={2012}
}

@book{bosq2012nonparametric,
  title={Nonparametric Statistics for Stochastic Processes: Estimation and Prediction},
  author={Bosq, Denis},
  volume={110},
  year={2012},
  publisher={Springer Science \& Business Media}
}

@article{wu2010kernel,
  title={Kernel estimation for time series: An asymptotic theory},
  author={Wu, Wei Biao and Huang, Yinxiao and Huang, Yibi},
  journal={Stochastic Processes and their Applications},
  volume={120},
  number={12},
  pages={2412--2431},
  year={2010},
  publisher={Elsevier}
}

@book{ott2002chaos,
  title={Chaos in Dynamical Systems},
  author={Ott, Edward},
  year={2002},
  publisher={Cambridge University Press}
}

@article{huke2006embedding,
  title={Embedding nonlinear dynamical systems: A guide to {T}akens' theorem},
  author={Huke, Jeremy P},
  year={2006},
  publisher={Manchester Institute for Mathematical Sciences, University of Manchester}
}

@article{ratas2024application,
  title={Application of next-generation reservoir computing for predicting chaotic systems from partial observations},
  author={Ratas, Irmantas and Pyragas, Kestutis},
  journal={Physical Review E},
  volume={109},
  number={6},
  pages={064215},
  year={2024},
  publisher={APS}
}

@article{liu2025sundial,
  title={Sundial: A family of highly capable time series foundation models},
  author={Liu, Yong and Qin, Guo and Shi, Zhiyuan and Chen, Zhi and Yang, Caiyin and Huang, Xiangdong and Wang, Jianmin and Long, Mingsheng},
  journal={arXiv preprint arXiv:2502.00816},
  year={2025}
}

@article{ansari2025chronos,
  title={Chronos-2: From univariate to universal forecasting},
  author={Ansari, Abdul Fatir and Shchur, Oleksandr and K{\"u}ken, Jaris and Auer, Andreas and Han, Boran and Mercado, Pedro and Rangapuram, Syama Sundar and Shen, Huibin and Stella, Lorenzo and Zhang, Xiyuan and others},
  journal={arXiv preprint arXiv:2510.15821},
  year={2025}
}

@article{jin2023time,
  title={Time-{L}{L}{M}: Time series forecasting by reprogramming large language models},
  author={Jin, Ming and Wang, Shiyu and Ma, Lintao and Chu, Zhixuan and Zhang, James Y and Shi, Xiaoming and Chen, Pin-Yu and Liang, Yuxuan and Li, Yuan-Fang and Pan, Shirui and others},
  journal={arXiv preprint arXiv:2310.01728},
  year={2023}
}

@article{naiman2024utilizing,
  title={Utilizing image transforms and diffusion models for generative modeling of short and long time series},
  author={Naiman, Ilan and Berman, Nimrod and Pemper, Itai and Arbiv, Idan and Fadlon, Gal and Azencot, Omri},
  journal={Advances in Neural Information Processing Systems},
  volume={37},
  pages={121699--121730},
  year={2024}
}

@article{jahn2025trajectory,
  title={Trajectory Generator Matching for Time Series},
  author={Jahn, T and Chemseddine, J and Hagemann, P and Wald, C and Steidl, G},
  journal={arXiv preprint arXiv:2505.23215},
  year={2025}
}

@article{naiman2023generative,
  title={Generative modeling of regular and irregular time series data via {K}oopman {VAEs}},
  author={Naiman, Ilan and Erichson, N Benjamin and Ren, Pu and Mahoney, Michael W and Azencot, Omri},
  journal={arXiv preprint arXiv:2310.02619},
  year={2023}
}

@article{yuan2024diffusion,
  title={Diffusion-{T}{S}: Interpretable diffusion for general time series generation},
  author={Yuan, Xinyu and Qiao, Yan},
  journal={arXiv preprint arXiv:2403.01742},
  year={2024}
}

@inproceedings{neklyudov2023action,
  title={Action matching: Learning stochastic dynamics from samples},
  author={Neklyudov, Kirill and Brekelmans, Rob and Severo, Daniel and Makhzani, Alireza},
  booktitle={International conference on machine learning},
  pages={25858--25889},
  year={2023},
  organization={PMLR}
}

@article{li2024fourier,
  title={From {F}ourier to neural {ODE}s: Flow matching for modeling complex systems},
  author={Li, Xin and Zhang, Jingdong and Zhu, Qunxi and Zhao, Chengli and Zhang, Xue and Duan, Xiaojun and Lin, Wei},
  journal={arXiv preprint arXiv:2405.11542},
  year={2024}
}

@article{islam2025longitudinal,
  title={Longitudinal Flow Matching for Trajectory Modeling},
  author={Islam, Mohammad Mohaiminul and Kuipers, Thijs P and Vadgama, Sharvaree and de Vente, Coen and Khan, Afsana and S{\'a}nchez, Clara I and Bekkers, Erik J},
  journal={arXiv preprint arXiv:2510.03569},
  year={2025}
}

@article{lim2024elucidating,
  title={Elucidating the design choice of probability paths in flow matching for forecasting},
  author={Lim, Soon Hoe and Wang, Yijin and Yu, Annan and Hart, Emma and Mahoney, Michael W and Li, Xiaoye S and Erichson, N Benjamin},
  journal={Transaction on Machine Learning Research},
  year={2025}
}

@article{zhang2024trajectory,
  title={Trajectory Flow Matching with Applications to Clinical Time Series Modeling},
  author={Zhang, Xi and Pu, Yuan and Kawamura, Yuki and Loza, Andrew and Bengio, Yoshua and Shung, Dennis L and Tong, Alexander},
  journal={arXiv preprint arXiv:2410.21154},
  year={2024}
}

@article{matheson1976scoring,
  title={Scoring rules for continuous probability distributions},
  author={Matheson, James E and Winkler, Robert L},
  journal={Management Science},
  volume={22},
  number={10},
  pages={1087--1096},
  year={1976},
  publisher={INFORMS}
}

@TECHREPORT{annan_tuning_TR,
  author =       {A. Yu and D. Lyu and S. H. Lim and M. W. Mahoney and N. B. Erichson},
  title =        {Tuning Frequency Bias of State Space Models},
  number =       {Preprint: arXiv:2410.02035},
  year =         {2024},
}

@TECHREPORT{Erichson_Koopman1_20_TR,
  author =       {O. Azencot and N. B. Erichson and V. Lin and M. W. Mahoney},
  title =        {Forecasting Sequential Data using {C}onsistent {K}oopman {A}utoencoders},
  number =       {Preprint: arXiv:2003.02236},
  year =         {2020},
}

@article{chen2024probabilistic,
  title={Probabilistic Forecasting with Stochastic Interpolants and {F}\"{o}llmer Processes},
  author={Chen, Yifan and Goldstein, Mark and Hua, Mengjian and Albergo, Michael S and Boffi, Nicholas M and Vanden-Eijnden, Eric},
  journal={arXiv preprint arXiv:2403.13724},
  year={2024}
}

@article{majeedi2025lets,
  title={LETS Forecast: Learning Embedology for Time Series Forecasting},
  author={Majeedi, Abrar and Gajjala, Viswanatha Reddy and GNVV, Satya Sai Srinath Namburi and Elkordi, Nada Magdi and Li, Yin},
  journal={arXiv preprint arXiv:2506.06454},
  year={2025}
}

@article{shi2025stochastic,
  title={Stochastic process learning via operator flow matching},
  author={Shi, Yaozhong and Ross, Zachary E and Asimaki, Domniki and Azizzadenesheli, Kamyar},
  journal={arXiv preprint arXiv:2501.04126},
  year={2025}
}

@article{lee2025multi,
  title={Multi-marginal stochastic flow matching for high-dimensional snapshot data at irregular time points},
  author={Lee, Justin and Moradijamei, Behnaz and Shakeri, Heman},
  journal={arXiv preprint arXiv:2508.04351},
  year={2025}
}

@article{kollovieh2024flow,
  title={Flow matching with {G}aussian process priors for probabilistic time series forecasting},
  author={Kollovieh, Marcel and Lienen, Marten and L{\"u}dke, David and Schwinn, Leo and G{\"u}nnemann, Stephan},
  journal={arXiv preprint arXiv:2410.03024},
  year={2024}
}

@article{bartosh2025sde,
  title={{SDE Matching}: Scalable and Simulation-Free Training of Latent Stochastic Differential Equations},
  author={Bartosh, Grigory and Vetrov, Dmitry and Naesseth, Christian A},
  journal={arXiv preprint arXiv:2502.02472},
  year={2025}
}

@book{fan2003nonlinear,
  title={Nonlinear Time Series: Nonparametric and Parametric Methods},
  author={Fan, Jianqing and Yao, Qiwei},
  year={2003},
  publisher={Springer}
}

@article{gilpin2024generative,
  title={Generative learning for nonlinear dynamics},
  author={Gilpin, William},
  journal={Nature Reviews Physics},
  volume={6},
  number={3},
  pages={194--206},
  year={2024},
  publisher={Nature Publishing Group UK London}
}

@article{sugihara1990nonlinear,
  title={Nonlinear forecasting as a way of distinguishing chaos from measurement error in time series},
  author={Sugihara, George and May, Robert M},
  journal={Nature},
  volume={344},
  number={6268},
  pages={734--741},
  year={1990},
  publisher={Nature Publishing Group UK London}
}

@article{brunton2021modern,
  title={Modern {K}oopman theory for dynamical systems},
  author={Brunton, Steven L and Budi{\v{s}}i{\'c}, Marko and Kaiser, Eurika and Kutz, J Nathan},
  journal={arXiv preprint arXiv:2102.12086},
  year={2021}
}

@article{lipman2024flow,
  title={Flow matching guide and code},
  author={Lipman, Yaron and Havasi, Marton and Holderrieth, Peter and Shaul, Neta and Le, Matt and Karrer, Brian and Chen, Ricky TQ and Lopez-Paz, David and Ben-Hamu, Heli and Gat, Itai},
  journal={arXiv preprint arXiv:2412.06264},
  year={2024}
}

@article{gottwald2025stable,
  title={Stable generative modelling using {S}chr{\"o}dinger bridges},
  author={Gottwald, Georg A and Li, Fengyi and Marzouk, Youssef and Reich, Sebastian},
  journal={Philosophical Transactions A},
  volume={383},
  number={2299},
  pages={20240332},
  year={2025},
  publisher={The Royal Society}
}

@article{hang2018kernel,
  title={Kernel density estimation for dynamical systems},
  author={Hang, Hanyuan and Steinwart, Ingo and Feng, Yunlong and Suykens, Johan AK},
  journal={Journal of Machine Learning Research},
  volume={19},
  number={35},
  pages={1--49},
  year={2018}
}

@article{zhang2024zero,
  title={Zero-shot forecasting of chaotic systems},
  author={Zhang, Yuanzhao and Gilpin, William},
  journal={arXiv preprint arXiv:2409.15771},
  year={2024}
}

@inproceedings{rohbeck2025modeling,
  title={Modeling complex system dynamics with flow matching across time and conditions},
  author={Rohbeck, Martin and De Brouwer, Edward and Bunne, Charlotte and Huetter, Jan-Christian and Biton, Anne and Chen, Kelvin Y and Regev, Aviv and Lopez, Romain},
  booktitle={The Thirteenth International Conference on Learning Representations},
  year={2025}
}

@article{krishnapriyan2023learning,
  title={Learning continuous models for continuous physics},
  author={Krishnapriyan, Aditi S and Queiruga, Alejandro F and Erichson, N Benjamin and Mahoney, Michael W},
  journal={Communications Physics},
  volume={6},
  number={1},
  pages={319},
  year={2023},
  publisher={Nature Publishing Group UK London}
}

@article{albergo2024learning,
  title={Learning to sample better},
  author={Albergo, Michael S and Vanden-Eijnden, Eric},
  journal={Journal of Statistical Mechanics: Theory and Experiment},
  volume={2024},
  number={10},
  pages={104014},
  year={2024},
  publisher={IOP Publishing}
}

@article{scarvelis2023closed,
  title={Closed-form diffusion models},
  author={Scarvelis, Christopher and Borde, Haitz S{\'a}ez de Oc{\'a}riz and Solomon, Justin},
  journal={arXiv preprint arXiv:2310.12395},
  year={2023}
}

@article{wald2025flow,
  title={Flow Matching: {M}arkov kernels, stochastic processes and transport plans},
  author={Wald, Christian and Steidl, Gabriele},
  journal={Variational and Information Flows in Machine Learning and Optimal Transport},
  pages={185--254},
  year={2025},
  publisher={Springer}
}

@article{mena2025statistical,
  title={Statistical Properties of Rectified Flow},
  author={Mena, Gonzalo and Kuchibhotla, Arun Kumar and Wasserman, Larry},
  journal={arXiv preprint arXiv:2511.03193},
  year={2025}
}

@article{bertrand2025closed,
  title={On the Closed-Form of Flow Matching: Generalization Does Not Arise from Target Stochasticity},
  author={Bertrand, Quentin and Gagneux, Anne and Massias, Mathurin and Emonet, R{\'e}mi},
  journal={arXiv preprint arXiv:2506.03719},
  year={2025}
}

@article{kunkel2025minimax,
  title={On the minimax optimality of Flow Matching through the connection to kernel density estimation},
  author={Kunkel, Lea and Trabs, Mathias},
  journal={arXiv preprint arXiv:2504.13336},
  year={2025}
}

@misc{hu2025flowtstimeseriesgeneration,
      title={Flow{T}{S}: Time Series Generation via Rectified Flow}, 
      author={Yang Hu and Xiao Wang and Zezhen Ding and Lirong Wu and Huatian Zhang and Stan Z. Li and Sheng Wang and Jiheng Zhang and Ziyun Li and Tianlong Chen},
      year={2025},
      eprint={2411.07506},
      archivePrefix={arXiv},
      primaryClass={cs.LG},
      url={https://arxiv.org/abs/2411.07506}, 
}

@article{kidger2022neural,
  title={On neural differential equations},
  author={Kidger, Patrick},
  journal={arXiv preprint arXiv:2202.02435},
  year={2022}
}

@article{kollovieh2024predict,
  title={Predict, refine, synthesize: Self-guiding diffusion models for probabilistic time series forecasting},
  author={Kollovieh, Marcel and Ansari, Abdul Fatir and Bohlke-Schneider, Michael and Zschiegner, Jasper and Wang, Hao and Wang, Yuyang Bernie},
  journal={Advances in Neural Information Processing Systems},
  volume={36},
  year={2024}
}

@article{chen2018neural,
  title={Neural ordinary differential equations},
  author={Chen, Ricky TQ and Rubanova, Yulia and Bettencourt, Jesse and Duvenaud, David K},
  journal={Advances in Neural Information Processing Systems},
  volume={31},
  year={2018}
}

@article{liu2022flow,
  title={Flow straight and fast: Learning to generate and transfer data with rectified flow},
  author={Liu, Xingchao and Gong, Chengyue and Liu, Qiang},
  journal={arXiv preprint arXiv:2209.03003},
  year={2022}
}

@article{liu2025navigation,
  title={From Navigation to Refinement: Revealing the Two-Stage Nature of Flow-based Diffusion Models through Oracle Velocity},
  author={Liu, Haoming and Liu, Jinnuo and Li, Yanhao and Bai, Liuyang and Ji, Yunkai and Guo, Yuanhe and Wan, Shenji and Wen, Hongyi},
  journal={arXiv preprint arXiv:2512.02826},
  year={2025}
}

@inproceedings{rasul2021autoregressive,
  title={Autoregressive denoising diffusion models for multivariate probabilistic time series forecasting},
  author={Rasul, Kashif and Seward, Calvin and Schuster, Ingmar and Vollgraf, Roland},
  booktitle={International Conference on Machine Learning},
  pages={8857--8868},
  year={2021},
  organization={PMLR}
}

@article{lipman2022flow,
  title={Flow matching for generative modeling},
  author={Lipman, Yaron and Chen, Ricky TQ and Ben-Hamu, Heli and Nickel, Maximilian and Le, Matt},
  journal={arXiv preprint arXiv:2210.02747},
  year={2022}
}

@article{tong2023improving,
  title={Improving and generalizing flow-based generative models with minibatch optimal transport},
  author={Tong, Alexander and Malkin, Nikolay and Huguet, Guillaume and Zhang, Yanlei and Rector-Brooks, Jarrid and Fatras, Kilian and Wolf, Guy and Bengio, Yoshua},
  journal={arXiv preprint arXiv:2302.00482},
  year={2023}
}

@article{albergo2023stochastic,
  title={Stochastic interpolants: A unifying framework for flows and diffusions},
  author={Albergo, Michael S and Boffi, Nicholas M and Vanden-Eijnden, Eric},
  journal={arXiv preprint arXiv:2303.08797},
  year={2023}
}

@article{auer2025tirex,
  title={Ti{R}ex: Zero-Shot Forecasting Across Long and Short Horizons with Enhanced In-Context Learning},
  author={Auer, Andreas and Podest, Patrick and Klotz, Daniel and B{\"o}ck, Sebastian and Klambauer, G{\"u}nter and Hochreiter, Sepp},
  journal={arXiv preprint arXiv:2505.23719},
  year={2025}
}

@misc{bilos2023modelingtemporaldatacontinuous,
      title={Modeling Temporal Data as Continuous Functions with Stochastic Process Diffusion}, 
      author={Marin Biloš and Kashif Rasul and Anderson Schneider and Yuriy Nevmyvaka and Stephan Günnemann},
      year={2023},
      eprint={2211.02590},
      archivePrefix={arXiv},
      primaryClass={cs.LG},
      url={https://arxiv.org/abs/2211.02590}, 
}

@misc{alcaraz2023diffusionbasedtimeseriesimputation,
      title={Diffusion-based Time Series Imputation and Forecasting with Structured State Space Models}, 
      author={Juan Miguel Lopez Alcaraz and Nils Strodthoff},
      year={2023},
      eprint={2208.09399},
      archivePrefix={arXiv},
      primaryClass={cs.LG},
      url={https://arxiv.org/abs/2208.09399}, 
}

@article{leonard2013survey,
  title={A survey of the {S}chr\"{o}dinger problem and some of its connections with optimal transport},
  author={L{\'e}onard, Christian},
  journal={arXiv preprint arXiv:1308.0215},
  year={2013}
}

@article{li2021markov,
  title={Markov neural operators for learning chaotic systems},
  author={Li, Zongyi and Kovachki, Nikola and Azizzadenesheli, Kamyar and Liu, Burigede and Bhattacharya, Kaushik and Stuart, Andrew and Anandkumar, Anima},
  journal={arXiv preprint arXiv:2106.06898},
  pages={2--3},
  year={2021}
}

@article{kovachki2023neural,
  title={Neural operator: Learning maps between function spaces with applications to {P}{D}{E}s},
  author={Kovachki, Nikola and Li, Zongyi and Liu, Burigede and Azizzadenesheli, Kamyar and Bhattacharya, Kaushik and Stuart, Andrew and Anandkumar, Anima},
  journal={Journal of Machine Learning Research},
  volume={24},
  number={89},
  pages={1--97},
  year={2023}
}

@article{lu2019deeponet,
  title={Deep{O}net: Learning nonlinear operators for identifying differential equations based on the universal approximation theorem of operators},
  author={Lu, Lu and Jin, Pengzhan and Karniadakis, George Em},
  journal={arXiv preprint arXiv:1910.03193},
  year={2019}
}

@article{gyongy1986mimicking,
  title={Mimicking the one-dimensional marginal distributions of processes having an {I}t{\^o} differential},
  author={Gy{\"o}ngy, Istv{\'a}n},
  journal={Probability Theory and Related Fields},
  volume={71},
  number={4},
  pages={501--516},
  year={1986},
  publisher={Springer}
}

@article{mcgoff2015statistical,
  title={Statistical inference for dynamical systems: A review},
  author={McGoff, Kevin and Mukherjee, Sayan and Pillai, Natesh},
  year={2015}
}

@article{berry2025limits,
  title={Limits of learning dynamical systems},
  author={Berry, Tyrus and Das, Suddhasattwa},
  journal={SIAM Review},
  volume={67},
  number={1},
  pages={107--137},
  year={2025},
  publisher={SIAM}
}

@article{coifman2006diffusion,
  title={Diffusion maps},
  author={Coifman, Ronald R and Lafon, St{\'e}phane},
  journal={Applied and Computational Harmonic Analysis},
  volume={21},
  number={1},
  pages={5--30},
  year={2006},
  publisher={Elsevier}
}

@article{nadler2006diffusion,
  title={Diffusion maps, spectral clustering and reaction coordinates of dynamical systems},
  author={Nadler, Boaz and Lafon, St{\'e}phane and Coifman, Ronald R and Kevrekidis, Ioannis G},
  journal={Applied and Computational Harmonic Analysis},
  volume={21},
  number={1},
  pages={113--127},
  year={2006},
  publisher={Elsevier}
}

@article{lim2025hidden,
  title={On The Hidden Biases of Flow Matching Samplers},
  author={Lim, Soon Hoe},
  journal={arXiv preprint arXiv:2512.16768},
  year={2025}
}

@book{harlim2018data,
  title={Data-Driven Computational Methods: Parameter and Operator Estimations},
  author={Harlim, John},
  year={2018},
  publisher={Cambridge University Press}
}

@article{berry2015nonparametric,
  title={Nonparametric forecasting of low-dimensional dynamical systems},
  author={Berry, Tyrus and Giannakis, Dimitrios and Harlim, John},
  journal={Physical Review E},
  volume={91},
  number={3},
  pages={032915},
  year={2015},
  publisher={APS}
}

@TECHREPORT{neurde_TR,
  author =       {J. A. L. Benitez and J. Guo and K. Hegazy and I. Dokmanic and M. W. Mahoney and M. V. de Hoop},
  title =        {Neural equilibria for long-term prediction of nonlinear conservation laws},
  number =       {Preprint: arXiv:2501.06933},
  year =         {2025},
}
\bibliographystyle{plainnat}

\newpage
\appendix
\onecolumn
\section*{Appendix}

This appendix is organized as follows. In App.~\ref{app:relatedwork}, we discuss related work in detail. In App.~\ref{app:background}, we provide background on flow matching (FM), conditional flow matching (CFM), and their empirical counterparts, as well as related discussion. In App.~\ref{app:algorithm}, we provide a detailed algorithm describing the proposed training-free model for probabilistic forecasting, as well as further interpretations. In App.~\ref{app:proof}, we provide proof of the main theoretical results presented in the main paper.  
In App.~\ref{app:duhamel_dm}, we offer additional theoretical results and insights on the proposed ODE model. In App.~\ref{app:experiments}, we provide experimental details and additional empirical results. In App.~\ref{app_realworld}, we provide a broader empirical evaluation on several real-world datasets. 
In App.~\ref{app_ablation_sigma}, we study the Gaussian path scale and isolate
the path-induced correction. In App.~\ref{app:tfm}, we compare FreeFM to another neural FM baseline (TFM \citep{zhang2024trajectory}). 
In App.~\ref{app_more}, we provide comparisons of FreeFM to nonparametric baselines.

\section{Related Work} 
\label{app:relatedwork}

\subsection{Sequence Modeling}
Sequence modeling has a long history that spans dynamical systems, statistical time series analysis, and data-driven learning. Classical approaches model temporal evolution through dynamical system formalisms such as ODEs and SDEs, as well as through statistical models including ARIMA, ARFIMA and their nonlinear extensions \citep{fan2003nonlinear, chan2013chaos}. These perspectives support principled inference in challenging regimes, including partially observed systems \citep{ratas2024application}, irregular sampling, and likelihood-based or general statistical inference for dependent data \citep{mcgoff2015statistical}. While such methods provide interpretability and theoretical structure, they often rely on parametric assumptions or explicit model specification that can be restrictive when the governing dynamics are unknown or highly complex or chaotic.

Motivated by these limitations, a broad line of work studies nonparametric approaches that estimate dynamics directly from data. This includes nonparametric methods for stationary processes, such as kernel-based density and transition estimation \citep{wu2010kernel,harvey2012kernel, hang2018kernel}, and locally stationary processes \citep{vogt2012nonparametric}, as well as nonparametric estimation for Markov chains \citep{lacour2008nonparametric} and transition distribution functionals of Markov processes \citep{roussas1969nonparametric}, with general treatments for stochastic processes \citep{bosq2012nonparametric}.
Closely related in spirit is Empirical Dynamical Modeling (EDM), which uses delay embeddings and memory-based local regression (motivated by Takens' theorem) to forecast nonlinear dynamics directly from historical trajectories \citep{sugihara1990nonlinear}. Our results connect to this tradition by revealing that, at the empirical optimum and under commonly used conditional paths \citep{lim2024elucidating}, flow matching induces a continuous-time, memory-augmented dynamical system whose vector field aggregates information from past transitions.

A complementary data-driven view is provided by Koopman operator methods and related lifting-based approaches, which represent nonlinear dynamics through linear evolution in a (infinite-dimensional) function space \citep{Erichson_Koopman1_20_TR,brunton2021modern, zhao2023data}. Whereas many nonparametric forecasting schemes emphasize local similarity and memory, Koopman-based approaches aim to create global representations that support prediction, system identification, and control. These perspectives are not mutually exclusive; rather, they emphasize different inductive biases (local memory versus global structure), a distinction that becomes salient when interpreting the mechanisms induced by modern generative objectives.

Finally, neural sequence models have become central due to their expressive parametric function approximation and scalability, including echo state and reservoir computing methods \citep{jaeger2004harnessing}, recurrent architectures \citep{hochreiter1997long}, Transformers \citep{vaswani2017attention}, State-Space Models \citep{gu2021efficiently,annan_tuning_TR}, as well as continuous-time neural models such as Neural ODEs and Neural SDEs \citep{chen2018neural, bartosh2025sde, kidger2022neural,krishnapriyan2023learning}.  Many recent works further develop hybrid approaches that combine the mechanistic structure with learned components. Complementary operator-learning approaches, such as various forms of neural operators \citep{kovachki2023neural, li2021markov,lu2019deeponet}, study function-to-function representations of dynamics in high-dimensional settings, addressing a modeling regime that is orthogonal to the empirical transition-level dynamics analyzed here.  In contrast to proposing a new parametric architecture, our work characterizes the \emph{optimal empirical} velocity field targeted by flow matching on sequential data, thereby providing an interpretable reference point for what expressive neural models trained with stochastic optimization are implicitly approximating.

\subsection{Generative Modeling}
Modern generative modeling approaches for sequential data often define dynamics implicitly through learned transport or score fields rather than explicit transition models. Flow matching has emerged as a promising framework, offering simple objectives and ODE-based sampling~\citep{lipman2022flow}. A growing body of work develops extensions and applications of flow matching across settings relevant to sequential data, including improved training and sampling \citep{tong2023improving}, connections to action and control perspectives \citep{neklyudov2023action}, trajectory and sequence constructions \citep{zhang2024trajectory, jahn2025trajectory}, stochastic variants \citep{shi2025stochastic}, longitudinal and irregularly structured data \citep{islam2025longitudinal}, Fourier and spectral parameterizations \citep{li2024fourier}, multi-marginal and multi-structure formulations \citep{lee2025multi}, as well as navigation- and latent/hidden-state-related viewpoints \citep{liu2025navigation, lim2025hidden}, among others \citep{kollovieh2024flow}. 

While these works demonstrate the flexibility and empirical strength of flow matching, they typically treat the learned velocity field as an implicit neural object. In contrast, extending the empirical approach of \citep{scarvelis2023closed, bertrand2025closed, li2026kinetic} to the sequential data setting, we ask what velocity field is learned at the \emph{empirical optimum} given finite sequential data, and we show that under a commonly used choice of conditional probability path the optimal field admits a training-free, closed-form characterization with explicit dependence on historical transitions.

Diffusion- and score-based generative models have also been widely adapted to sequential and continuous-time settings, often through SDE formulations \citep{yuan2024diffusion, bilos2023modelingtemporaldatacontinuous, gilpin2024generative}, including approaches that emphasize temporal structure and forecasting \citep{naiman2023generative, naiman2024utilizing, kollovieh2024predict}. These methods learn score fields that define stochastic sampling dynamics; our analysis is complementary in that it exposes when flow matching yields an \emph{explicit} empirical optimum—interpretable as a similarity-weighted replay of past transitions—highlighting a distinct mechanism from purely parametric score~learning.

Several related probabilistic transport formalisms further connect generative modeling and dynamical systems. Stochastic interpolants \citep{chen2024probabilistic} provide a general perspective on the construction of stochastic paths and associated dynamics, while Schr\"odinger bridge formulations characterize entropic optimal transport trajectories between distributions \citep{theodoropoulos2025momentum, gottwald2025stable}. More broadly, existing diffusion- and flow-based generative approaches often employ probability paths that transport a base distribution (e.g., a Gaussian) to an entire trajectory viewed as a single object \citep{hu2025flowtstimeseriesgeneration, naiman2024utilizing, bilos2023modelingtemporaldatacontinuous}, whereas the paths considered here are explicitly aligned with temporal structure \citep{lim2024elucidating, zhang2024trajectory, jahn2025trajectory}, yielding velocity fields that decompose naturally into sequential transitions. Our results emphasize that in the flow matching setting for sequential data, the \emph{choice of probability path} can fundamentally determine the induced training-free dynamics at the empirical optimum, suggesting a modeling axis that parallels (but is distinct from) choices of interpolation or regularization in these related frameworks.

In parallel, large language models and foundation models have recently been explored for time series and sequence forecasting, including large-scale pretraining and zero-shot transfer \citep{gruver2023large, jin2023time, zhang2024zero, liu2025sundial, ansari2025chronos, auer2025tirex}. These approaches emphasize generalization across tasks and datasets through scale, whereas our work studies the structure of the \emph{finite-data empirical optimum} for flow matching, clarifying the extent to which the induced dynamics reflect general dynamical principles versus dataset-dependent replay.

\noindent {\bf Positioning of This Work.}
The posterior-mixture representation of the empirical CFM optimum follows
from the standard conditional expectation identity of CFM and is not by
itself a new contribution. Our focus is its specialization to observed
sequential transition pairs and the dynamical structure revealed by that
specialization.
Our approach builds on the closed-form perspective \citep{scarvelis2023closed,bertrand2025closed, li2026kinetic}.  Under a commonly used choice of conditional probability path \citep{lim2024elucidating}, we show that the optimal empirical solution induces a nonparametric, memory-augmented ODE whose vector field admits a closed-form expression, and enables training-free forecasting by balancing the act of replaying historical transitions and injecting score-based correction. This perspective bridges nonparametric sequence modeling and modern generative learning (see also Table~\ref{tab:freefm_related_families}) and positions neural FM models trained by stochastic optimization as parametric surrogates of a nonparametric solution.

We believe our results provide valuable insights for the time series community; see the paragraph before App.~\ref{app:fm_time_series}. Flow-based generative modeling is gaining increasing attention in sequence modeling, yet its foundations in sequential settings remain relatively underexplored. Our work helps fill this gap by uncovering connections between FM, memory effects, and nonparametric dynamical systems, and raises questions about the role of model parameterization.

\begin{table}[t]
\vspace{+0.1cm}
\small 
\centering
\begin{tabular}{@{\hspace{0pt}}lccccc}
\hline
Method family & Local & Nonparametric & Cont-time & Probabilistic & Neural \\
\hline
EDM \citep{sugihara1990nonlinear}  & \greencheck & \greencheck & $\times$ & $\sim$ & $\times$ \\
KAF \citep{zhao2016analog} & \greencheck & \greencheck & $\times$ & $\sim$ & $\times$ \\
KCE  & \greencheck & \greencheck & \redcross & $\sim$ & \redcross \\
Diffusion map  \citep{coifman2006diffusion} & \greencheck & \greencheck & \redcross & \redcross & \redcross \\
CNF/FM models \citep{lim2024elucidating} & \redcross & \redcross & \greencheck & \greencheck & \greencheck \\
Retrieval-based models \citep{tire2024retrieval} & \greencheck & $\sim$ & \redcross & $\sim$ & $\sim$ \\
\textbf{FreeFM} & \greencheck & \greencheck & \greencheck & \greencheck & \redcross \\
\hline \\
\end{tabular}
\caption{{\bf Comparison of FreeFM with related model families.} 
For a fixed empirical transition distribution and conditional probability
path, the full-memory FreeFM velocity field is the unrestricted $L^2$ minimizer of
the corresponding CFM objective. A neural model trained on that same
objective approximates this field subject to function class and optimization error.
FreeFM combines local retrieval, nonparametric transition modeling,
continuous-time flow-based forecasting, and training-free predictive
ensemble generation. Here, $\sim$ indicates that a property may be present
depending on the specific method instantiation.}
\label{tab:freefm_related_families}
\end{table}

\section{Background and Related Discussion}
\label{app:background}

The main goal of  generative modeling is to use finitely many samples from a distribution to construct a sampling scheme capable of generating new samples from the same distribution. 

Among the families of existing generative models, flow matching (FM) is notable for its flexibility and simplicity. Given a target probability distribution, FM uses a parametric model (e.g., neural network) to learn the velocity vector field that defines a deterministic, continuous transformation (a normalizing flow) and transports a source probability distribution (e.g., standard Gaussian) to the target distribution.

From now on, whenever we write a distribution on state space through a density, we assume that the corresponding density exists, in which case we shall abuse the notation and use the same symbol to denote both the distribution and the density, unless stated otherwise.

\subsection{Flow Matching (FM)}
\label{app_FM}
The goal of FM is to find a velocity field $v: [0,1] \times  \R^d \to \R^d$, such that, if we solve the ODE:
$$\frac{dz(t)}{dt} = v(t, z(t)) =: v_t(z(t)), \  z(0) = z_0 \in \R^d,$$ then the law of $z(1)$ when $z_0 \sim p_0$ is $p_1$ (in which case we say that $v$ drives $p_0$ to $p_1$). The law of $z(t)$ for $t \in [0,1]$ is described by  a probability path $p: [0,1] \times \R^d \to \R$, denoted $p_t(z)$, that evolves from $p_0$ at $t=0$ to $p_1$ at $t=1$. If we know $v$, then we can first sample $z_0 \sim p_0$ and then evolve the ODE from $t=0$ to $t=1$ to generate new samples. 

The velocity field $v$ generates the continuous flow $\psi: [0,1] \times \R^d$ given as $\psi_t(z) = z(t)$, and the probability path via the push-forward distributions: $p_t = [\psi_t]_{\#} p_0$, i.e., $\psi_t(Z) \sim p_t$ for $Z \sim p_0$. In particular, $Z \sim p_0$ implies that $\psi_1(Z) \sim p_1$, i.e., $\psi_t$ can be viewed as a dynamical transport map. The ODE corresponds to the Lagrangian description (the $v$-generated trajectories viewpoint), and a change of variables link it to the Eulerian description (the evolving probability path $p_t$ viewpoint). Indeed, under standard regularity assumptions, a necessary and sufficient condition for $v$ to generate $p_t$ is given by the continuity equation \citep{albergo2024learning, wald2025flow,neurde_TR}: 
\begin{equation} \label{eq_continuity}
    \frac{\partial p_t}{\partial t} + \nabla \cdot (p_t v) = 0,
\end{equation}
where $\nabla \cdot$ denotes the divergence operator. 
This equation ensures that the flow defined by $v$ conserves the mass (or probability) described by $p_t$. In general, even for simple choice of path, the velocity field does not admit a closed-form expression when  $p_0$ and $p_1$ are known, except in special cases such as Gaussians, mixture of Gaussians and uniform distributions \citep{mena2025statistical}.

Given the population velocity $v$, sampling from $p_1$ is achieved by sampling $z_0\sim p_0$ and integrating the ODE forward in time.
In practice, $v$ is approximated by a parametric model $v_\theta$ trained via the FM objective
\begin{equation}\label{eq:LFM}
L_{\mathrm{FM}}[v_\theta]
=
\mathbb E_{t\sim\mathcal U[0,1],\,Z_t\sim p_t}
\big[\|v_\theta(t,Z_t)-v(t,Z_t)\|^2\big].
\end{equation}

\subsubsection{Conditional Flow Matching via a Latent Variable}

Conditional flow matching (CFM) \citep{lipman2022flow,tong2023improving} generalizes FM by introducing a latent variable
$C \sim \pi$
taking values in a measurable space $\mathcal C$.
For each realization $C=c$, we specify a conditional probability path $p_t(z\mid C=c)$ generated by a conditional velocity field $v(t,z\mid C=c)$.

The marginal probability path is the mixture
\begin{equation}\label{eq:latent_mixture_path}
p_t(z)=\int p_t(z\mid c)\,\pi(dc),
\end{equation}
and the associated velocity field is given by conditional expectation
\begin{equation}\label{eq:latent_markov_velocity}
v(t,z)
=
\E\!\left[v(t,Z_t\mid C)\mid Z_t=z\right]
=
\int v(t,z\mid c)\frac{p_t(z\mid c)\pi(dc)}{p_t(z)}.
\end{equation}

Given a parametric model $v_\theta$, the CFM regression objective reads
\begin{equation}\label{eq:LCFM}
L_{\mathrm{CFM}}[v_\theta]
=
\E_{t\sim\mathcal U[0,1],\,C\sim\pi,\,Z_t\sim p_t(\cdot\mid C)}
\big[\|v_\theta(t,Z_t)-v(t,Z_t\mid C)\|^2\big].
\end{equation}
It is shown in \citep{lipman2022flow} that minimizing $L_{\mathrm{CFM}}$ is equivalent, in terms of minimizers, to minimizing the FM objective \eqref{eq:LFM}.
Thus, CFM learns a  velocity field by projecting latent (possibly non-Markov) conditional dynamics onto state space that shares the same dimension as the data~space.

\medskip
\noindent{\bf Recovering the standard CFM formulation.}
Setting $C=X$ with $X\sim p_1$ yields
\[
p_t(z)=\int p_t(z\mid x)p_1(x)\,dx,
\]
and
\[
v(t,z)=\int v(t,z\mid x)\frac{p_t(z\mid x)p_1(x)}{p_t(z)}dx,
\]
with the objective
\[
L_{\mathrm{CFM}}[v_\theta]
=
\E_{t,X,Z_t\sim p_t(\cdot\mid X)}
\big[\|v_\theta(t,Z_t)-v(t,Z_t\mid X)\|^2\big].
\]

In order to apply CFM, we need to specify the boundary distributions $p_0$ and $p_1$, and the conditional probability path $p_t(z|x)$. Below are some examples.

\noindent {\bf Rectified Flow.} 
    A canonical choice \citep{liu2022flow} is $p_0 = \mathcal{N}(0, I_d)$, $p_1 = p^*$, and 
    \begin{equation} \label{eq_probpath_RF}
         p_t(z | X = x_1) = \mathcal{N}(z; t x_1, (1-t)^2 I_d),
    \end{equation}
    which corresponds to the conditional velocity field $v(t, z| X=x_1) = \frac{x_1 - z}{1-t}$. This conditional probability path realizes linear interpolating paths of the form $Z_t = (1-t) x_0 + t x_1$ between a (reference) Gaussian sample $x_0$ and a data sample $x_1$. In practice,  regularized versions of rectified flow are preferred for numerical stability  (since $v$ blows up as $t \to 1$). A simple version is to modify the conditional probability path to $$p_t(\cdot | X = x_1) = \mathcal{N}(t x_1, (1-(1-\sigma_{\min})t)^2 I_d),$$ for some small $\sigma_{\min} > 0$, which corresponds to the regularized conditional velocity field $v(t, z| X=x_1) = \frac{x_1 - (1-\sigma_{\min})z}{1-(1-\sigma_{\min})t}$. Another version is to  consider a smoothed version of the data distribution $p^*$; e.g., $p_1 = p^* \star \mathcal{N}(0, \sigma_{\min}^2 I_d)$, where $\star$ denotes convolution.

\noindent {\bf Affine Conditional Flows.} 
Consider a base random variable $Z \sim \mathbb{Q}$ with probability density function (PDF) $K$ (not necessarily Gaussian) and, for $t \in [0,1]$,  the affine conditional flow defined by $\psi_t(Z|X) = m_t(X) + \sigma_t(X) Z$ for some time-differentiable functions $m: [0,1] \times \R^d \to \R^d$ and $\sigma: [0,1] \times \R^d \to (0,\infty)$. Since $\psi_t$ is linear in $Z$, we can obtain its density via the change of variables:
\begin{equation} \label{eq_probpath_affine}
    p_t(z|X) = \frac{1}{\sigma_t^d(X)} K\left(\frac{z-m_t(X)}{\sigma_t(X)} \right).
\end{equation}
Then, as in Theorem 3 in \citep{lipman2022flow}, we can show that the unique vector field that defines $\psi_t(\cdot | X)$ via the ODE $\frac{d}{dt} \psi_t(z|X) = v(t, \psi_t(z|X) | X)$ has the form: 
\begin{equation}
    v(t, z| X) = a_t(X) z + b_t(X),
\end{equation}
where 
\begin{align} \label{eq_ab}
    a_t(X) &= \frac{\frac{\partial \sigma_t}{\partial t}(X)}{\sigma_t(X)}, \quad 
    b_t(X) = \frac{\partial m_t}{\partial t}(X) - m_t(X) a_t(X).
\end{align}
The rectified flow in previous example is a special case of this family of conditional flows (with $K = \mathcal{N}(0, I_d)$,  $m_t(X) = tX$ and $\sigma_t(X) = 1-t$). The Gaussian flows considered in \citep{lipman2022flow, tong2023improving, albergo2023stochastic} are also special cases.

All the formulations thus far are in the idealized continuous-time setting. In practice, we work with  Monte Carlo estimates of the objective and use the optimized $v_\theta$ to generate new samples by simulating the ODE with a numerical scheme. Note, however, that the training of CFM is simulation-free: the dynamics are only simulated at inference time and not when training the parametric model. An end-to-end error analysis for deep generative models based on FM was recently provided \citep{zhouerror}.

\subsection{Empirical Flow Matching}  \label{subsec:empiricalfm}

Suppose that we are given a source distribution $p_0$ and $N$ i.i.d.\ samples
$x^{(i)} \sim p_1, \  i=1,\dots,N,$
i.e., we only have access to the target distribution $p_1$ through a finite sample.
We approximate $p_1$ by its empirical distribution
$\hat{p}_1 := \frac{1}{N} \sum_{i=1}^N \delta_{x^{(i)}}.$
We shall refer to FM and CFM instantiated with $p_1=\hat{p}_1$ as \emph{empirical FM} and \emph{empirical CFM}, respectively.

\paragraph{Latent-variable formulation}
Introduce a discrete latent variable
\[
C \sim \pi, \qquad \pi(i)=\tfrac{1}{N}, \quad i\in\{1,\dots,N\},
\]
and define the conditional probability path
\[
p_t(z\mid C=i) := p_t(z\mid x^{(i)}),
\]
where $p_t(\cdot\mid x^{(i)})$ denotes a conditional probability path such as
\eqref{eq_probpath_RF} or \eqref{eq_probpath_affine}.
The corresponding marginal probability path is
\begin{equation}\label{eq_emp_pt}
\hat{p}_t(z)
=
\mathbb E_{C}[\,p_t(z\mid C)\,]
=
\frac{1}{N}\sum_{i=1}^N p_t(z\mid x^{(i)}),
\end{equation}
and the associated velocity field, given by conditional expectation, takes the form
\begin{equation}\label{eq_emp_v}
\hat{v}(t,z)
=
\mathbb E\!\left[v(t,z\mid C)\mid Z_t=z\right]
=
\sum_{i=1}^N v(t,z\mid x^{(i)})
\frac{p_t(z\mid x^{(i)})}{\sum_{j=1}^N p_t(z\mid x^{(j)})}.
\end{equation}

\paragraph{Empirical FM and CFM objectives}
The empirical FM objective is
\begin{equation}
\hat{L}_{\text{FM}}[v']
=
\mathbb{E}_{t \sim \mathcal{U}[0,1],\, Z_t \sim \hat{p}_t}
\big[\|v'(t, Z_t) - \hat{v}(t, Z_t)\|^2\big],
\end{equation}
while the empirical CFM objective can be written, using the latent variable $C$, as
\begin{align}
\hat{L}_{\text{CFM}}[v']
&=
\mathbb{E}_{t \sim \mathcal{U}[0,1],\, C\sim\pi,\, Z_t \sim p_t(\cdot \mid C)}
\big[\|v'(t, Z_t) - v(t, Z_t\mid C)\|^2\big] \nonumber \\
&=
\frac{1}{N}\sum_{i=1}^N
\mathbb{E}_{t \sim \mathcal{U}[0,1],\, Z_t \sim p_t(\cdot \mid x^{(i)})}
\big[\|v'(t, Z_t) - v(t, Z_t\mid x^{(i)})\|^2\big].
\end{align}

If each conditional velocity field $v(t,\cdot\mid x^{(i)})$ generates the corresponding conditional probability path $p_t(\cdot\mid x^{(i)})$, then the velocity field $\hat{v}(t,\cdot)$ generates the empirical path $\hat{p}_t$ (see Lemma~2.1 in \citep{kunkel2025minimax}).
Moreover, as in the population case, the minimizing arguments of empirical FM and empirical CFM coincide (see Theorem~2.2 in \citep{kunkel2025minimax}).

For the conditional probability paths considered earlier, the empirical objective admits a closed-form minimizer
\[
\hat{v}^*
\in
\arg\min_v \hat{L}_{\mathrm{CFM}}[v]
=
\arg\min_v \hat{L}_{\mathrm{FM}}[v],
\]
yielding a \emph{training-free} generative model.
Sampling is performed by integrating the ODE
\begin{equation}\label{eq_closedformsampler}
\frac{d\hat{z}^*(t)}{dt}
=
\hat{v}^*(t,\hat{z}^*(t)),
\qquad
\hat{z}^*(0)\sim p_0,
\end{equation}
which is evolved to $t=1$ to obtain new samples.

\begin{proposition}[Coupled empirical affine FM/CFM minimizer]\label{prop:coupled_emp_affine}
Let $C$ be a discrete latent variable taking values in an index set $\mathcal{I}$ with prior
$\pi(c)$, and suppose that for each $c\in\mathcal{I}$ we are given an affine conditional flow on
$\R^d$ of the form
\[
\psi_t(Z\mid C=c)=m_t(c)+\sigma_t(c)Z,\qquad Z\sim K,
\]
where $m_t(c)\in\R^d$ and $\sigma_t(c)>0$ are differentiable in $t$.
Let $p_t(\cdot\mid c)$ be the induced conditional density and let
$v(t,z\mid c)=a_t(c)z+b_t(c)$ be the (unique) velocity field generating $\psi_t(\cdot\mid c)$,
with
\[
a_t(c)=\frac{\dot\sigma_t(c)}{\sigma_t(c)},
\qquad
b_t(c)=\dot m_t(c)-a_t(c)m_t(c).
\]
Define the mixture path $p_t(z)=\sum_{c\in\mathcal{I}}\pi(c)\,p_t(z\mid c)$.
Then the unique minimizer of the empirical CFM (equivalently FM) objective
\[
\hat{\mathcal L}_{\mathrm{CFM}}[v']
=\E_{t\sim \mathcal U[0,1]}\E_{C\sim\pi}\E_{Z_t\sim p_t(\cdot\mid C)}
\|v'(t,Z_t)-v(t,Z_t\mid C)\|^2
\]
admits the closed form expression:
\[
\hat v^*(t,z)=\sum_{c\in\mathcal{I}} w_c(t,z)\,\big(a_t(c)z+b_t(c)\big),
\qquad
w_c(t,z)=\frac{\pi(c)\,p_t(z\mid c)}{\sum_{c'\in\mathcal{I}}\pi(c')\,p_t(z\mid c')}.
\]
\end{proposition}
\begin{proof}[Proof of Proposition~\ref{prop:coupled_emp_affine}]
Let $t\sim\mathcal U[0,1]$, $C\sim\pi$ on the discrete index set $\mathcal I$, and
$Z_t\sim p_t(\cdot\mid C)$, where $p_t(\cdot\mid c)$ is induced by the affine conditional flow
$\psi_t(\cdot\mid c)=m_t(c)+\sigma_t(c)Z$ with $Z\sim K$.
The objective is
\[
\hat{\mathcal L}_{\mathrm{CFM}}[v']
=
\E_{t}\E_{C}\E_{Z_t\sim p_t(\cdot\mid C)}
\|v'(t,Z_t)-v(t,Z_t\mid C)\|^2.
\]

Fix $t\in[0,1]$ and consider minimizing over measurable functions $z\mapsto v'(t,z)$.
Write the inner expectation as an integral against the joint law of $(C,Z_t)$:
\[
\E_{C}\E_{Z_t\sim p_t(\cdot\mid C)}\|v'(t,Z_t)-v(t,Z_t\mid C)\|^2
=
\sum_{c\in\mathcal I}\pi(c)\int_{\R^d}\|v'(t,z)-v(t,z\mid c)\|^2\,p_t(z\mid c)\,dz.
\]
Equivalently, letting
\[
p_t(z)=\sum_{c\in\mathcal I}\pi(c)p_t(z\mid c)
\]
be the marginal density of $Z_t$, we can rewrite the same quantity as
\[
\int_{\R^d}
\left(
\sum_{c\in\mathcal I}\pi(c)p_t(z\mid c)\,\|v'(t,z)-v(t,z\mid c)\|^2
\right)dz.
\]
For each fixed $z$, the integrand is a strictly convex quadratic function of the vector
$v'(t,z)\in\R^d$ (strictly convex whenever $p_t(z)>0$). Hence, the minimizer is obtained
pointwise by setting the gradient to zero:
\[
0
=
\frac{\partial}{\partial v'(t,z)}
\sum_{c\in\mathcal I}\pi(c)p_t(z\mid c)\,\|v'(t,z)-v(t,z\mid c)\|^2
=
2\sum_{c\in\mathcal I}\pi(c)p_t(z\mid c)\big(v'(t,z)-v(t,z\mid c)\big).
\]
Solving yields
\[
v'^*(t,z)
=
\frac{\sum_{c\in\mathcal I}\pi(c)p_t(z\mid c)\,v(t,z\mid c)}{\sum_{c'\in\mathcal I}\pi(c')p_t(z\mid c')}
=
\sum_{c\in\mathcal I} w_c(t,z)\,v(t,z\mid c),
\qquad
w_c(t,z)=\frac{\pi(c)p_t(z\mid c)}{p_t(z)}.
\]
This is exactly the conditional expectation form
$v'^*(t,z)=\E[v(t,z\mid C)\mid Z_t=z]$.

Finally, using the affine form $v(t,z\mid c)=a_t(c)z+b_t(c)$ gives
\[
\hat v^*(t,z)
=
\sum_{c\in\mathcal I} w_c(t,z)\,\big(a_t(c)z+b_t(c)\big),
\]
which proves the claimed closed form.

Uniqueness holds because for each $t$ and almost every $z$ with $p_t(z)>0$, the pointwise
quadratic is strictly convex, hence the minimizer is unique $p_t$-a.e.; this determines a unique
minimizer in $L^2(dt\otimes p_t(dz))$.
\end{proof}

\begin{example}[Empirical Rectified Flow]
For the rectified flow example, the minimizer $\hat{v}^*$ admits the closed-form expression \citep{bertrand2025closed}
\begin{equation}\label{eq_empRF}
\hat{v}^*(t,z)
=
\sum_{i=1}^N
w_i(t,z)\,\frac{x^{(i)}-z}{1-t},
\end{equation}
where the weights are given by
\[
w_i(t,z)
=
\mathrm{softmax}_i\!\left(
\left(
-\frac{\|z-tx^{(j)}\|^2}{2(1-t)^2}
\right)_{j\in[N]}
\right),
\]
with  $\mathrm{softmax}_i$ denoting the $i$th component of the vector obtained after applying the softmax operation.
The optimal velocity field is thus a time-dependent weighted average of the directions pointing toward the empirical samples.
Analogous formulas hold for regularized variants of rectified flow.
\end{example}

\begin{example}[Empirical Affine Flows]\label{eg_empaffineflows}
To construct a flow from $p_0(z)=K(z)$ to the smoothed empirical distribution
\[
\tilde{p}_1(z)
=
\frac{1}{N\sigma_{\min}^d}
\sum_{i=1}^N
K\!\left(\frac{z-x^{(i)}}{\sigma_{\min}}\right),
\]
we may choose functions $m_t$ and $\sigma_t$ satisfying
\[
m_0(C)=0, \quad m_1(C)=x^{(C)}, \qquad
\sigma_0(C)=1, \quad \sigma_1(C)=\sigma_{\min}.
\]
Here and throughout, we write $x^{(C)}:=x^{(i)}$ when $C=i$.
The final marginal distribution, obtained by averaging $p_1(z\mid C=i)$ over the empirical distribution of $C$, coincides with the  kernel density estimator $\tilde{p}_1$ \citep{nadaraya1964estimating}.
Heuristically, when $K$ is Gaussian and $\sigma_{\min}\to 0$, this construction recovers rectified flow.

For affine conditional flows, the empirical minimizer admits the closed-form representation below. This follows from Proposition~\ref{prop:coupled_emp_affine}
by taking $\mathcal I=[N]$, $\pi(i)=1/N$, and $c=i$ indexing the samples $x^{(i)}$.

\begin{proposition}\label{app_prop1}
For the family of affine conditional flows, the minimizer $\hat{v}^*$ of the empirical FM objective admits the closed-form expression
\[
\hat{v}^*(t,z)
=
\sum_{i=1}^N
w_i(t,z)\,\big(a_t(x^{(i)})\,z+b_t(x^{(i)})\big),
\]
where $a_t$ and $b_t$ are given in \eqref{eq_ab}, and the weights are
\[
w_i(t,z)
=
\frac{p_t(z\mid x^{(i)})}{\sum_{j=1}^N p_t(z\mid x^{(j)})},
\qquad
p_t(z\mid x^{(i)})
=
\frac{1}{\sigma_t^d(x^{(i)})}
K\!\left(\frac{z-m_t(x^{(i)})}{\sigma_t(x^{(i)})}\right).
\]
\end{proposition}

Intuitively, the optimal empirical velocity field $\hat{v}^*$ is a convex combination of the individual conditional velocity fields $v(t, z| x^{(i)})$, weighted by  $w_i(t, z)$ which  tells us how likely the observed point $z$ at time $t$ is to belong to the flow path originating from the sample $x^{(i)}$.

\end{example}

\subsection{Empirical Couplings, Time Series, and Multi-Marginal Extensions}

The empirical FM construction is generalized by introducing a latent variable $C=(I,J)$ indexing pairs of samples from $p_0$ and $p_1$ with an arbitrary coupling $\pi_{ij}$.
This yields KDE-to-KDE transport and recovers empirical FM as a special case.

\begin{example}[Empirical Affine Flows: KDE-to-KDE Transport]\label{eg_kde2kde_affine}
In many settings we observe i.i.d.\ samples from both endpoints,
\[
x_0^{(i)} \sim p_0,\quad i=1,\dots,N_0,
\qquad
x_1^{(j)} \sim p_1,\quad j=1,\dots,N_1,
\]
and would like to transport a \emph{smoothed} empirical approximation of $p_0$ to a
smoothed empirical approximation of $p_1$.
Fix a base density $K$ on $\R^d$ and bandwidths $\sigma_0,\sigma_1>0$, and define the kernel density estimators (KDEs):
\[
\tilde p_0(z)=\frac{1}{N_0\sigma_0^d}\sum_{i=1}^{N_0}
K\!\left(\frac{z-x_0^{(i)}}{\sigma_0}\right),
\qquad
\tilde p_1(z)=\frac{1}{N_1\sigma_1^d}\sum_{j=1}^{N_1}
K\!\left(\frac{z-x_1^{(j)}}{\sigma_1}\right).
\]
Introduce a discrete latent variable
\[
C=(I,J)\in[N_0]\times[N_1],\qquad \mathbb{P}(C=(i,j))=\pi_{ij},
\]
where $\pi$ is an arbitrary coupling between the empirical endpoints (e.g.\ the
independent coupling $\pi_{ij}=\frac{1}{N_0N_1}$, an optimal transport (OT) coupling, or an empirical
time series coupling when paired data are available).

For each $(i,j)$, define an affine conditional flow with
\[
\psi_t(Z\mid I=i,J=j)=m_t^{ij}+\sigma_t^{ij} Z,\qquad Z\sim K,
\]
where $m_t^{ij}\in\R^d$ and $\sigma_t^{ij}>0$ are chosen such that
\[
m_0^{ij}=x_0^{(i)},\quad m_1^{ij}=x_1^{(j)},\qquad
\sigma_0^{ij}=\sigma_0,\quad \sigma_1^{ij}=\sigma_1.
\]
Then the induced conditional density is
\[
p_t(z\mid i,j)=\frac{1}{(\sigma_t^{ij})^d}\,
K\!\left(\frac{z-m_t^{ij}}{\sigma_t^{ij}}\right),
\]
and the marginal path is the mixture
\[
p_t(z)=\sum_{i=1}^{N_0}\sum_{j=1}^{N_1}\pi_{ij}\,p_t(z\mid i,j),
\]
so that $p_0=\tilde p_0$ and $p_1=\tilde p_1$ by construction.

The following corollary follows from Proposition~\ref{prop:coupled_emp_affine}
by taking $\mathcal I=[N_0]\times[N_1]$, $c=(i,j)$, and $\pi(c)=\pi_{ij}$.

\begin{corollary}[KDE-to-KDE affine transport with coupling]\label{cor:kde2kde_closedform}
In the setting of Example~\ref{eg_kde2kde_affine} with $C=(I,J)$ and coupling $\pi_{ij}$,
the minimizer is
\[
\hat{v}^*(t,z)=\sum_{i=1}^{N_0}\sum_{j=1}^{N_1}
w_{ij}(t,z)\,\Big(a_t^{ij} z+b_t^{ij}\Big),
\qquad
w_{ij}(t,z)=\frac{\pi_{ij}\,p_t(z\mid i,j)}{\sum_{i',j'}\pi_{i'j'}\,p_t(z\mid i',j')},
\]
with $a_t^{ij}=\dot\sigma_t^{ij}/\sigma_t^{ij}$ and $b_t^{ij}=\dot m_t^{ij}-a_t^{ij}m_t^{ij}$.
\end{corollary}

Sampling from $\tilde p_1$ is performed by integrating the training-free ODE
$\dot z(t)=\hat{v}^*(t,z(t))$ with $z(0)\sim \tilde p_0$.

\medskip
\noindent\textbf{Special cases.}
(i) Taking $N_0=1$, $x_0^{(1)}=0$, and $\sigma_0=1$ recovers the standard ``base-to-KDE"
setting in Example~\ref{eg_empaffineflows}.
(ii) If $(x_0^{(i)},x_1^{(i)})$ are paired observations (time series), one may choose the
diagonal coupling $\pi_{ij}=\frac{1}{N}\mathbf 1\{i=j\}$, yielding a KDE-to-KDE flow constructed from the empirical temporal coupling.
\end{example}

If instead a true joint distribution $(X_0,X_1)\sim p_{01}$ exists—as in time series or dynamical systems—setting $C=(X_0,X_1)$ yields
\[
v^*(t,z)=\E[v(t,z\mid X_0,X_1)\mid Z_t=z],
\]
which defines the Markovian projection (see the next subsection) of the pairwise temporal dynamics -- this is the approach we consider in the main paper.

\paragraph{Extension to multi-marginal FM}
The framework readily extends to
$C=(X_0,X_1,\dots,X_K),$
corresponding to multi-marginal flow matching (MMFM) \citep{rohbeck2025modeling,lee2025multi}.
Conditional paths may be defined using spline-interpolated means, Gaussian bridges, or other structured probability paths satisfying multiple marginal constraints. It would also be interesting to extend to the more challenging setting where the sequences are irregularly sampled.

\paragraph{Time series and conditional forecasting}
In conditional forecasting, $C$ naturally represents past observations $(X_{\tau-k},\dots,X_\tau)$.
The latent-variable CFM framework therefore provides a principled way to construct Markov generative models for forecasting even when the true dynamics are non-Markovian.

\subsection{Markovian Projection and Relation to Schr\"odinger Bridges}

FM and CFM can be interpreted as regression-based Markovian projections \citep{bentata2012markovian, duong2025coarse}.
Latent conditional dynamics, induced by a target dynamical process which we take as our prior belief,  may depend on hidden variables and thus be non-Markovian; projecting them via conditional expectation yields a Markov process that exactly matches prescribed one-time marginals.
This is the deterministic analogue of Gy\"ongy’s Markovian projection for SDEs \citep{gyongy1986mimicking, brunick2013mimicking}.

\begin{theorem}[Informal: Gy\"ongy’s mimicking theorem or Markovian projection theorem {\citep{gyongy1986mimicking}}]\label{thm:gyongy}
Let $(X_t)_{t\in[0,1]}$ be an It\^o process in $\R^d$ of the form
\[
dX_t = b_t\,dt + \Sigma_t\,dW_t,
\]
where $(W_t)$ is a standard Brownian motion, and $b_t\in\R^d$, $\Sigma_t\in\R^{d\times m}$ are progressively measurable processes such that the conditional moments below are well-defined.
Define the (state-dependent) drift and diffusion coefficients
\[
\bar b(t,x) := \E[b_t \mid X_t=x],
\qquad
\bar a(t,x) := \E[\Sigma_t\Sigma_t^\top \mid X_t=x].
\]
Assume $\bar b$ and $\bar a$ are measurable and that $\bar a(t,x)$ is uniformly nondegenerate so that there exists a weak solution to the SDE
\[
dY_t = \bar b(t,Y_t)\,dt + \bar \sigma(t,Y_t)\,dW_t,
\qquad \bar\sigma(t,x)\bar\sigma(t,x)^\top=\bar a(t,x),
\quad Y_0\stackrel{d}{=}X_0.
\]
Then the Markov process $(Y_t)$ \emph{mimics} $(X_t)$ in the sense that for every $t\in[0,1]$,
$Y_t \stackrel{d}{=} X_t$.
\end{theorem}

Markovian projection for SDEs has been extensively studied, most notably in mathematical finance (e.g., \citep{piterbarg2006markovian}) following the work of  \citep{dupire1994pricing}.

A stochastic process $(X_t)$ is Markov if the distribution of its future depends only on the present state; in the ODE setting, this means that the flow is fully determined, in distribution, by the current state. 
For example, the linearly interpolating path
$X_t=\alpha_t X_1+\sigma_t X_0,$
with $(X_0,X_1)$ drawn from a coupling of the endpoint marginals, is not Markov in $X_t$ since its velocity depends on the hidden endpoints $X_0$ and $X_1$. 
FM resolves this by projecting the latent dynamics onto a  velocity field defined by the conditional expectation
$v(t, x)=\mathbb{E}[\dot X_t \mid X_t=x].$
The resulting ODE $\dot X_t=v(X_t,t)$ defines a Markov process that exactly matches the prescribed one-time marginal distributions and constitutes the optimal Markov approximation of the original dynamics in the $L^2$ regression sense. 
However, matching marginals alone does not uniquely determine the full path law unless additional structure is imposed; consequently, the Markovian projection generally does not preserve the original non-Markov temporal correlations or multi-time joint distributions. This deterministic projection can be viewed as the zero-noise analogue of Gy\"ongy’s theorem.

\begin{corollary}[Informal: Deterministic Markovian projection (zero-noise analogue)]
Let $(X_t)_{t \in [0,1]}$ be an absolutely continuous process with $\dot X_t$ integrable and define
\[
v(t,x):=\E[\dot X_t\mid X_t=x].
\]
Assume further that the family of marginals $\mu_t:=\mathcal L(X_t)$ solves the continuity equation
$\partial_t\mu_t+\nabla\cdot(\mu_t v)=0$ in the distributional sense, and that this equation admits a unique solution.
If the ODE $\dot Y_t=v(t,Y_t)$ admits a (weak) solution on $[0,1]$ with $Y_0\stackrel{d}{=}X_0$, then
$Y_t\stackrel{d}{=}X_t$ for all $t \in [0,1]$.
\end{corollary}

Markovian projection alone specifies a Markov process that matches prescribed one-time marginal distributions, but it leaves the full path law underdetermined, as different Markov processes may share the same marginals while exhibiting distinct temporal correlations. Schr\"odinger bridges (SBs) \citep{leonard2013survey}  resolve this non-uniqueness by selecting, among all path measures with the given marginals, the unique process minimizing relative entropy with respect to a chosen reference process. 
From this perspective, FM implements a regression-based Markovian projection, while the SB adds a variational principle that fixes a distinguished Markov process consistent with both the marginals and optimal transport in path space.

\medskip
\noindent{\bf Connecting back to our proposed ODE-based forecaster.}
Even if the underlying dynamics generating the data are non-Markovian, the latent-variable CFM framework allows us to construct a Markov flow whose one-time marginals match the data by projecting pairwise or multi-time temporal couplings onto a time-inhomogeneous velocity field. In this paper, we focus on understanding FM for sequential data, and we restrict our attention to pairwise couplings.  The key insight gained is summarized as follows, which can motivate various extensions (e.g., to higher-order temporal couplings) which we leave for future work.

\begin{tcolorbox}[
  colback=sand,
  colframe=sandborder,
  boxrule=0.3pt,
  left=8pt,
  right=8pt,
  top=3pt,
  bottom=3pt
]
\centering
\emph{
Viewed through the latent-variable FM formulation, empirical FM, once equipped with a choice of latent variable and conditional probability path which reflects an inductive bias, induces a rich family of nonparametric, data-adaptive ODE  samplers for probabilistic forecasting.}  
\end{tcolorbox} 

Once a conditional probability path $p_t(\cdot\mid C)$ is specified, the associated empirical  velocity $\hat v^*$ defines a deterministic ODE, while stochastic perturbations or Schr\"odinger bridge relaxations naturally lead to diffusion-based (SDE) samplers.
When the latent variable $C$ encodes temporal information (e.g., pairs or windows of observations), these constructions yield data-driven dynamical models for forecasting that do not require an explicit parametric specification of the underlying dynamics.
On the theoretical side, characterizing the model expressiveness and approximation properties induced by different choices of latent variables and probability paths under various data generating assumptions is an interesting direction for future work.

\subsection{More on FM for Sequential Data}
\label{app:fm_time_series}

The finite-bank identities in the main text are purely empirical and do not require
independence among transitions. 
When the underlying dynamics are stationary and ergodic, $\hat p_1$ provides a Monte Carlo approximation of the population one-step transition law; otherwise, it is interpreted as a purely empirical, data-driven measure. Although transitions within a single trajectory are temporally dependent, ergodicity or mixing ensures convergence of their empirical distribution along long trajectories (see below for details), while transitions across trajectories are independent by construction. We therefore treat $\mathcal D_M$ as a valid Monte Carlo--type approximation for minimizing the objective.

In this subsection, we justify why FM remains statistically valid when trained on a
\emph{single observed trajectory} from a stationary time series.
The key point is that the empirical FM loss is a \emph{time average}
of a stationary integrable observable, and hence converges to its population expectation
by Birkhoff's ergodic theorem, even though the samples are not independent.

Let $\Omega:=(\mathbb R^d)^{\mathbb Z}$ be the canonical path space equipped with the product Borel
$\sigma$-field $\mathcal F$ and a probability measure $\mathbb P$ under which the coordinate process
$X_\tau(\omega):=\omega_\tau$ is defined.
Let $\varphi:\Omega\to\Omega$ be the left shift $(\varphi\omega)_\tau=\omega_{\tau+1}$ and denote by
$\mathcal I:=\{A\in\mathcal F:\varphi^{-1}(A)=A\}$ the invariant $\sigma$-field.

\begin{definition}[Stationarity and ergodicity]
The process $(X_\tau)_{\tau\in\mathbb Z}$ is \emph{(strictly) stationary} if $\mathbb P$ is
$\varphi$-invariant, i.e.\ $\mathbb P\circ\varphi^{-1}=\mathbb P$.
It is \emph{ergodic} if $\mathcal I$ is $\mathbb P$-trivial, i.e.\ $\mathbb P(A)\in\{0,1\}$ for all $A\in\mathcal I$.
\end{definition}

\paragraph{Stationary and ergodic transitions}
Define the pair process $Y_\tau:=(X_\tau,X_{\tau+1})\in\mathbb R^{2d}$.
If $(X_\tau)$ is stationary and ergodic, then $(Y_\tau)$ is also stationary and ergodic.

\paragraph{FM objective along a trajectory}
Let $v_\theta(t,z)$ be a parametric velocity field.
For each observed pair $Y_\tau=(X_\tau,X_{\tau+1})$, define the (noise-injected) interpolation
\[
Z^{(\tau)}_t := (1-t)X_\tau + t X_{\tau+1} + \sigma_t \xi,
\qquad \xi\sim\mathcal N(0,I_d),
\]
where $\sigma_t\ge 0$ is differentiable and $\xi$ is independent of $(X_\tau)$.
Then $\dot Z^{(\tau)}_t = (X_{\tau+1}-X_\tau)+\dot\sigma_t\,\xi$.

Define the per-transition \emph{integrated} FM loss as the measurable function $\ell_\theta:\mathbb R^{2d}\to\mathbb R$,
\[
\ell_\theta(x_0,x_1)
:=
\mathbb E_{t,\xi}\!\left[
\big\|
v_\theta(t,(1-t)x_0+t x_1+\sigma_t\xi)
-
(x_1-x_0)-\dot\sigma_t\,\xi
\big\|^2
\right],
\]
where the expectation is over $t\sim \mathcal U[0,1]$ and $\xi\sim\mathcal N(0,I_d)$.
The empirical objective from a single trajectory is
\[
\mathcal L_T(\theta):=\frac1T\sum_{\tau=0}^{T-1}\ell_\theta(Y_\tau).
\]

\begin{theorem}[Ergodic consistency of FM (fixed $\theta$)]
\label{thm:ergodic_fm}
Fix $\theta\in\Theta$ and assume $\ell_\theta$ is measurable with $\mathbb E[|\ell_\theta(Y_0)|]<\infty$.
If $(X_\tau)$ is stationary and ergodic, then
\[
\mathcal L_T(\theta)=\frac1T\sum_{\tau=0}^{T-1}\ell_\theta(Y_\tau)
\;\xrightarrow[T\to\infty]{a.s.}\;
\mathcal L(\theta):=\mathbb E[\ell_\theta(Y_0)].
\]
\end{theorem}

\begin{proof}
Define the integrable observable $H(\omega):=\ell_\theta(X_0(\omega),X_1(\omega))$.
Then $\ell_\theta(Y_\tau(\omega)) = H(\varphi^\tau\omega)$ and hence
\[
\mathcal L_T(\theta)=\frac1T\sum_{\tau=0}^{T-1} H(\varphi^\tau\omega).
\]
By Birkhoff's ergodic theorem for measure-preserving transformations
(e.g., applying\ Thm.~6.2.1 in \citep{durrett2019probability}),
\[
\frac1T\sum_{\tau=0}^{T-1} H(\varphi^\tau\omega)\to \mathbb E[H\mid\mathcal I]
\quad\text{a.s.}
\]
Ergodicity implies $\mathcal I$ is trivial, hence $\mathbb E[H\mid\mathcal I]=\mathbb E[H]$ a.s.
Therefore the limit equals $\mathbb E[\ell_\theta(Y_0)]$.
\end{proof}

Theorem~\ref{thm:ergodic_fm} shows that FM training on a single stationary trajectory is statistically well-defined at the population level, despite temporal dependence.

\section{Detailed Algorithm and Further Interpretations for the Proposed Training-Free Sampler}
\label{app:algorithm}

\subsection{Detailed Algorithm of FreeFM}

To generate new samples, we numerically integrate the closed-form ODE model \eqref{eq:CFM-memory-ODE} on a time-grid $0=t_0<\dots <t_L=1$.
Let $Z_{\ell}$ denote the hidden state of the ODE at time $t_\ell$.
Algorithm \ref{alg:sampler}  provides a detailed algorithm to describe the proposed training-free sampler for probabilistic forecasting. We mention the use of explicit forward Euler scheme there, but in principle any ODE integration method could be used.

\begin{algorithm}[!h]
\vspace{+0.1cm}
\caption{Training-Free ODE Sampler for One-Step Forecasting}
\label{alg:sampler}
\begin{algorithmic}[1]
\REQUIRE Current state $x_\tau$, memory bank $\{X^{(j)}\}_{j=1}^M$, parameters $(\sigma_{\min},\sigma)$, time grid $\{t_\ell\}_{\ell=0}^L$
\STATE Sample $Z_0 \sim \mathcal{N}(x_\tau, \sigma_{\min}^2 I_d)$
\FOR{$\ell = 0$ to $L-1$}
    \STATE $t \gets t_\ell$, $\Delta_\ell \gets t_{\ell+1}-t_\ell$
    \STATE $c_t^2 \gets \sigma_{\min}^2 + \sigma^2 t(1-t)$
    \STATE $G_t \gets \frac{\sigma^2(1-2t)}{2c_t^2} I_d$
    \FOR{$j = 1$ to $M$}
        \STATE $m \gets (1-t)X^{(j)}_1 + t X^{(j)}_2$
        \STATE $\dot{m} \gets X^{(j)}_2 - X^{(j)}_1$
        \STATE $B_j \gets \dot{m} - G_t m$ 
        \STATE $s_j \gets -\frac{\|Z_\ell - m\|^2}{2c_t^2}$
    \ENDFOR
    \STATE $\alpha \gets \text{softmax}(s)$   
    \STATE $h \gets \sum_{j=1}^M \alpha_j B_j$
    \STATE $v \gets G_t Z_\ell + h$ 
    \STATE {\it Choose integration method and step sizes $\Delta_\ell$:}
    \IF{using explicit forward Euler}
        \STATE $Z_{\ell+1} \gets Z_\ell + \Delta_\ell \cdot v$
    \ENDIF
\ENDFOR
\STATE  $x_{\tau+1} \gets Z_L$
\end{algorithmic}
\end{algorithm}

\subsection{Other Interpretations of the Training-Free Model}
\label{subsec:ts_unified_interpretation}

\noindent {\bf Closed-form, data-adaptive model.}
For any admissible affine conditional path, the empirical CFM
minimizer $\hat{v}^*(t,z)$ is available in closed form and depends on the data only
through the weights $w_j(t,z)$ and per-transition velocity labels
$v(t,z\mid X^{(j)})$ (Theorem~\ref{thm:ts-emp-affine}). In particular, no
parametric training or optimization is required: the sampler drift is computed
directly from the memory bank.

\noindent {\bf Kernel conditional expectation and attention.}
In the Gaussian-bridge specialization, the optimal drift decomposes as
$\hat{v}^*(t,z) =
G_t z + \sum_{j=1}^M \alpha_j(t,z)\, y_j(t),$ where $y_j(t) = \dot m_t^{(j)} - G_t m_t^{(j)}$.  
For each fixed $t$, the nonlinear term
$z \mapsto \sum_j \alpha_j(t,z)\, y_j(t)$
coincides with a Nadaraya--Watson estimator \citep{nadaraya1964estimating} of the conditional expectation of the  velocity label given $Z_t=z$, using a Gaussian kernel over the bridge cloud
$\{m_t(X^{(j)})\}_{j=1}^M$.
Equivalently, the weights $\alpha_j(t,z) = \mathrm{softmax}_j\!\left(
    -\tfrac{1}{2c_t^2}\|z - m_t(X^{(j)})\|^2
\right)$ define a soft attention mechanism over a
memory bank of transitions $\{X^{(j)}\}_{j=1}^M$, so that the empirical velocity
$\hat v_t(z)
=
\sum_{j=1}^M \alpha_j(t,z)\, u_t(X^{(j)})$
is a similarity-weighted average of historical instantaneous velocities.
This yields a data-driven sampler with Gaussian-kernel fading memory.

\noindent {\bf Connection with EDM.} 
Although reminiscent of continuous-time RNNs \citep{hopfield1984neurons}
in their state evolution, the mechanism of the ODE sampler more closely resembles kernel-based
nonparametric regression or memory-based prediction. It can be viewed as a continuous-time analogue of Empirical Dynamic Modeling (EDM) \citep{sugihara1990nonlinear, majeedi2025lets}, which relies on Takens' embedding \citep{huke2006embedding} and nearest-neighbors geometry. While EDM obtains a locally weighted estimator of the discrete-time map via nonparametric regression, our  formulation  yields, in closed form, the corresponding locally weighted estimator of a  velocity field whose flow reproduces the transitions.  
Our probabilistic approach goes beyond EDM (which is fundamentally deterministic and geometric) and admits principled  extensions that allow for uncertainty quantification. 

\noindent{\bf Stochastic extensions.}
The present work focuses on the deterministic ODE induced by empirical CFM.
Adding an arbitrary diffusion term while keeping the same drift generally
changes the Fokker--Planck equation and therefore does not preserve the same
probability path. Constructing marginal-preserving stochastic extensions,
or relating the empirical transition construction to stochastic
interpolants and Schr\"odinger bridges, is an interesting direction for
future work.

\noindent {\bf Operator-theoretic and diffusion map connection.}
The ODE sampler has a unique solution for all $t \in [0,1]$ and admits a Duhamel representation that separates global linear transport
from a data-driven forcing term. After a co-moving change of variables, the forcing at each
sampler time is exactly given by a Nystr\"om extension of a row-stochastic diffusion-map
operator \citep{coifman2006diffusion} applied to intrinsic velocity labels; see
App.~\ref{app:duhamel_dm} for more details.

\section{Proof of Theoretical Results}
\label{app:proof}

\subsection{Proof of Theorem \ref{thm:ts-emp-affine}}
\label{proof_thm1}

Theorem~\ref{thm:ts-emp-affine} is a direct specialization of the general
latent-variable CFM minimizer (see App.~\ref{app:background}) to the transition dataset
$\mathcal D_M=\{X^{(j)}\}_{j=1}^M$, with latent index $C\in[M] := \{1, \dots, M\}$ and uniform prior
$\pi(j)=1/M$. For completeness, we provide the proof specialized to this setting.

\begin{proof}
Let $C\in[M]$ be a discrete latent variable with prior $\pi(j)=1/M$ and let $X=X^{(C)}$.
By definition,
\[
\hat{\mathcal{L}}_{\mathrm{CFM}}[v']
=
\E_{t\sim\mathcal U[0,1],\,C\sim\pi,\,Z_t\sim p_t(\cdot\mid X^{(C)})}
\|v'(t,Z_t)-v(t,Z_t\mid X^{(C)})\|^2 .
\]

Fix $t\in[0,1]$. Since the objective is quadratic in $v'(t,\cdot)$, the unique minimizer
over measurable functions $z\mapsto v'(t,z)$ is given by the $L^2$ projection:
\[
\hat v^*(t,z)
=
\E\!\left[v(t,z\mid X^{(C)})\mid Z_t=z\right].
\]

Because $C$ is discrete and $X^{(C)}=X^{(j)}$ when $C=j$, this conditional expectation
can be written~as
\[
\hat v^*(t,z)
=
\sum_{j=1}^M \mathbb{P}(C=j\mid Z_t=z)\, v(t,z\mid X^{(j)}).
\]

By Bayes’ rule,
\[
\mathbb{P}(C=j\mid Z_t=z)
=
\frac{\pi(j)\,p_t(z\mid X^{(j)})}{\sum_{k=1}^M \pi(k)\,p_t(z\mid X^{(k)})}
=
\frac{p_t(z\mid X^{(j)})}{\sum_{k=1}^M p_t(z\mid X^{(k)})}
=: w_j(t,z),
\]
where we used $\pi(j)=1/M$.

Substituting the affine conditional velocity
$v(t,z\mid X^{(j)})=a_t(X^{(j)})z+b_t(X^{(j)})$ yields
\[
\hat v^*(t,z)
=
\sum_{j=1}^M
w_j(t,z)\,\big(a_t(X^{(j)})\,z+b_t(X^{(j)})\big),
\]
which proves the claim. Uniqueness follows from strict convexity of the quadratic loss.
\end{proof}

\subsection{Proof of Proposition \ref{prop_grads}} \label{app_lipschitzproof}

For an open set $\Omega\subset\R^d$, we denote by $C^k(\Omega)$ the space of functions
$f:\Omega\to\R^m$ whose partial derivatives up to order $k$ exist and are continuous on
$\Omega$.
We write $C^\infty(\Omega)=\bigcap_{k\ge1} C^k(\Omega)$ for the space of smooth functions.

\paragraph{Spatial Lipschitz constant}
For each fixed $t\in[0,1]$, the (spatial) Lipschitz constant of a map
$f(t,\cdot):\R^d\to\R^d$ is defined in the standard metric sense as
\[
\operatorname{Lip}_z(f)(t) :=\sup_{z\neq z'}\frac{\|f(t,z)-f(t,z')\|}{\|z-z'\|}\in[0,\infty].
\]
If, in addition, $f(t,\cdot)\in C^1(\R^d)$ and
$\sup_{z\in\R^d}\|\nabla_z f(t,z)\|<\infty$, where $\nabla_z f$ denotes the Jacobian
with respect to $z$ and $\|\cdot\|$ is the operator norm induced by the Euclidean norm,
then by the mean value theorem,
\[
\operatorname{Lip}_z(f)(t) 
\le
\sup_{z\in\R^d}\|\nabla_z f(t,z)\|.
\]
On compact, convex domains $\Omega\subset\R^d$ (instead of $\R^d$ as above), this inequality becomes an equality
(see Proposition~14 in \citep{tsimpos2025optimal}).

The (spatial) Lipschitz constant of the velocity field plays an important role in both the stability and
approximation theory of ODE-based models.
Prior work
\citep{benton2023error,marzouk2024distribution,tsimpos2025optimal} shows that distributional error
in such models is controlled by the spatial Lipschitz constant of the underlying velocity field,
reflecting the sensitivity of transported measures to perturbations in the dynamics.
These results identify the spatial Lipschitz constant as a key quantity linking dynamical stability
and robustness of transported distributions. Moreover, large spatial Lipschitz constants are directly associated with numerical stiffness of the
induced ODE, necessitating smaller time steps and more restrictive stability conditions for
time-integration schemes.

The above considerations, together with issues of numerical stiffness, motivate our study of the spatial Lipschitz constant of $\hat{v}^*$. To prove Proposition~\ref{prop_grads}, we first establish the following auxiliary result.

\begin{lemma} \label{lem:grad_alpha}

Fix $t\in[0,1]$ and let $c_t>0$. Define
\[
\pi_j(t,z)\;:=\;\exp\!\Big(-\frac{\|z-m_t^{(j)}\|^2}{2c_t^2}\Big),
\qquad
\alpha_j(t,z)\;:=\;\frac{\pi_j(t,z)}{\sum_{k=1}^M \pi_k(t,z)},
\]
and 
$\bar m_t(z)\;:=\;\sum_{k=1}^M \alpha_k(t,z)\,m_t^{(k)}$.
Then for each $j\in\{1,\dots,M\}$ and all $z\in\mathbb R^d$,
\begin{equation}\label{eq:grad_alpha_exact}
\nabla_z \alpha_j(t,z)
= \frac{1}{c_t^2}\,\alpha_j(t,z)\,\Big(m_t^{(j)}-\bar m_t(z)\Big).
\end{equation}
In particular,
\begin{equation}\label{eq:grad_alpha_bound_simple}
\big\|\nabla_z \alpha_j(t,z)\big\| \leq
\frac{1}{c_t^2}\,\alpha_j(t,z)\,\big\|m_t^{(j)}-\bar m_t(z)\big\|.
\end{equation}
\end{lemma}

\begin{proof}
Fix $t\in[0,1]$ and suppress the explicit dependence on $t$ in the notation for notational simplification:
write $c:=c_t$, $m_j:=m_t^{(j)}$, $\pi_j(z):=\pi_j(t,z)$, and $\alpha_j(z):=\alpha_j(t,z)$.

Let $\Pi(z):=\sum_{k=1}^M \pi_k(z),
\ \text{so that} \ 
\alpha_j(z)=\frac{\pi_j(z)}{\Pi(z)}.$
By the quotient rule,
\[
\nabla \alpha_j(z)
=
\frac{\nabla \pi_j(z)}{\Pi(z)}-\frac{\pi_j(z)}{\Pi(z)^2}\,\nabla \Pi(z)
=
\alpha_j(z)\bigg(\nabla \log \pi_j(z)-\sum_{k=1}^M \alpha_k(z)\,\nabla \log \pi_k(z)\bigg).
\]
Since $\log \pi_j(z)=-\frac{1}{2c^2}\|z-m_j\|^2$ and thus $\nabla \log \pi_j(z)=-\frac{1}{c^2}(z-m_j),$
we obtain
\begin{align*}
\nabla \alpha_j(z)
&=
\alpha_j(z)\left(-\frac{1}{c^2}(z-m_j)+\sum_{k=1}^M \alpha_k(z)\frac{1}{c^2}(z-m_k)\right) \\
&= \frac{\alpha_j(z)}{c^2}\left(m_j-\sum_{k=1}^M \alpha_k(z)m_k\right).
\end{align*}
Reinstating $t$ and recalling $\bar m_t(z)=\sum_k \alpha_k(t,z)m_t^{(k)}$ yields
\eqref{eq:grad_alpha_exact}. The bound \eqref{eq:grad_alpha_bound_simple} follows by taking norms.
\end{proof}

The following lemma guarantees that $h(t,\cdot;\mathcal D_M)$ is $C^1(\R^d)$, hence its $z$-Lipschitz constant
is well-defined.

\begin{lemma}[Smoothness in $z$]\label{lem:smooth_h_v}
Fix $t\in[0,1]$ and a dataset $\mathcal D_M$.
Assume $\sigma_{\min}>0$ and 
$c_t^2:=\sigma_{\min}^2+\sigma^2 t(1-t)>0$. Let
\[
\alpha_j(t,z)
=
\frac{
\exp\!\left(-\frac{\|z-m_t^{(j)}\|^2}{2 c_t^2}\right)
}{
\sum_{k=1}^M 
\exp\!\left(-\frac{\|z-m_t^{(k)}\|^2}{2 c_t^2}\right)
},
\qquad
h(t,z;\mathcal D_M)=\sum_{j=1}^M \alpha_j(t,z)\,y_j(t),
\]
where $y_j(t)=\dot m_t^{(j)}-G_t m_t^{(j)}$ and $G_t=g(t)I_d$ is independent of $z$.
Then for each fixed $t$, the maps $z\mapsto \alpha_j(t,z)$, $z\mapsto h(t,z;\mathcal D_M)$, and
$\hat v^*(t,z)=G_t z + h(t,z;\mathcal D_M)$ is $C^\infty(\R^d)$ in $z$.
\end{lemma}

\begin{proof}
For fixed $t$, each function
\[
z\mapsto \exp\!\left(-\frac{\|z-m_t^{(j)}\|^2}{2c_t^2}\right)
\]
is $C^\infty(\R^d)$ and strictly positive since $c_t^2>0$.
Hence the denominator
\[
S(z,t):=\sum_{k=1}^M \exp\!\left(-\frac{\|z-m_t^{(k)}\|^2}{2c_t^2}\right)
\]
is $C^\infty(\R^d)$ and satisfies $S(z,t)>0$ for all $z$.
Therefore, $\alpha_j(t,z)$ is $C^\infty(\R^d)$ since it is a quotient of $C^\infty$ functions
with a nonvanishing denominator.
Since $y_j(t)$ does not depend on $z$, $h(t,z;\mathcal D_M)$ is a finite linear combination of $C^\infty$
functions in $z$, hence $C^\infty$.
Finally, $z\mapsto G_t z$ is linear, so $\hat v^*(t,\cdot)$ is $C^\infty$ as well.
\end{proof}

Proposition~\ref{prop_grads} is a special case of the following theorem
(with $\sigma_{\min}>0$ and $\sigma>0$).

\begin{theorem}[Spatial Lipschitz bound]
\label{prop_grads_long}
Fix $t\in[0,1]$ and a dataset $\mathcal D_M$.
Assume that for all $j$ and $t$, $\|\dot m_t^{(j)}\|\le R_1$, $\|m_t^{(j)}\|\le R_m.$
Then the $z$-Lipschitz constant of $h$ can be bounded as:
\[
\operatorname{Lip}_z(h)(t) \leq \sup_{z \in \R^d}\|\nabla_z h(t,z;\mathcal D_M)\|
\ \le\ 
\frac{2R_1R_m}{c_t^2}
+\|G_t\|\frac{2R_m^2}{c_t^2}.
\]
Since $\|G_t\|\le \sigma^2/(2c_t^2)$, for bounded $\sigma$ along any
parameter regime in which $c_t\to0$ we have
$\|G_t\|=O(c_t^{-2})$, and hence
\[
\operatorname{Lip}_z(h)(t)=O(c_t^{-4})\qquad \text{ as } c_t\to 0.
\]
Moreover,
\[
\operatorname{Lip}_z(\hat{v}^*)(t) \leq \sup_{z \in \R^d} \|\nabla_z \hat{v}^*(t,z)\|
\le
\|G_t\|+\operatorname{Lip}_z(h)(t)
=
O(c_t^{-4})
\]
as $c_t\to 0$.
\end{theorem}

\begin{proof}[Proof of Theorem \ref{prop_grads_long}]
    Let $t \in [0,1]$ and $\mathcal{D}_M$ be given.
    We split $h(t, z; \mathcal{D}_M)$ into two parts and analyze them separately:
    $$h(t, z; \mathcal{D}_M) = h_1(t, z; \mathcal{D}_M) - h_2(t, z; \mathcal{D}_M),$$
    where 
    $$h_1(t, z; \mathcal{D}_M) = \sum_{j=1}^M \alpha_j(t,z) \dot{m}_t^{(j)}, $$ 
    $$h_2(t, z; \mathcal{D}_M) = \sum_{j=1}^M \alpha_j(t,z) G_t m_t^{(j)} =: G_t \bar{m}_t(z).$$
    
    The full Jacobian is $\nabla_z h = \nabla_z h_1 - \nabla_z h_2$. We can bound its norm using the triangle inequality: $\|\nabla_z h\| \le \|\nabla_z h_1\| + \|\nabla_z h_2\|$.
    
    For the first term, $$\nabla_z h_1 = \nabla_z \left( \sum_{j=1}^M \alpha_j \dot{m}_t^{(j)} \right) = \sum_{j=1}^M \dot{m}_t^{(j)} (\nabla_z \alpha_j)^\top.$$
    
    From Lemma \ref{lem:grad_alpha}, we have  $\nabla_z \alpha_j = \frac{\alpha_j}{c_t^2} (m_t^{(j)} - \bar{m}_t(z))$, and so: 
    $$\nabla_z h_1 =  \frac{1}{c_t^2} \sum_{j=1}^M \alpha_j \dot{m}_t^{(j)} (m_t^{(j)} - \bar{m}_t(z))^\top.$$
    
    We bound its norm:
    $$\|\nabla_z h_1\| \le \frac{1}{c_t^2} \sum_{j=1}^M \alpha_j \|\dot{m}_t^{(j)}\| \cdot \|m_t^{(j)} - \bar{m}_t(z)\|$$
    
    Using our bounds $R_1$ and $R_m$ (and $\|m_t^{(j)} - \bar{m}_t(z)\| \le \|m_t^{(j)}\| + \|\bar{m}_t(z)\| \le 2R_m$): $$\|\nabla_z h_1\| \le \frac{1}{c_t^2} \sum_{j=1}^M \alpha_j (R_1) (2R_m) = \frac{2 R_1 R_m}{c_t^2} \left( \sum_{j=1}^M \alpha_j \right).$$
    
    Since $\sum \alpha_j = 1$, we have $\|\nabla_z h_1\| \le \frac{2 R_1 R_m}{c_t^2}$. Thus, $\operatorname{Lip}_z(h_1)(t) = O(c_t^{-2})$.
    
    For the second term, $$\nabla_z h_2 = \nabla_z \left( G_t \bar{m}_t(z) \right) = G_t \cdot \nabla_z \bar{m}_t(z).$$
    First, we compute the Jacobian of the posterior mean $\bar{m}_t(z)$:
    $$\nabla_z \bar{m}_t(z) = \nabla_z \left( \sum_{j=1}^M \alpha_j m_t^{(j)} \right) = \sum_{j=1}^M m_t^{(j)} (\nabla_z \alpha_j)^\top = 
    \frac{1}{c_t^2} \sum_{j=1}^M \alpha_j m_t^{(j)} (m_t^{(j)} - \bar{m}_t(z))^\top.$$
    
    We bound its norm similarly:
    $$\|\nabla_z \bar{m}_t(z)\| \le \frac{1}{c_t^2} \sum_{j=1}^M \alpha_j \|m_t^{(j)}\| \cdot \|m_t^{(j)} - \bar{m}_t(z)\| \le \frac{1}{c_t^2} \sum_{j=1}^M \alpha_j (R_m) (2R_m) = \frac{2 R_m^2}{c_t^2}.$$
    
    Now, we compute the norm of $\nabla_z h_2$ by multiplying by the norm of $G_t$:
    $$\|\nabla_z h_2\| \le \|G_t\| \cdot \|\nabla_z \bar{m}_t(z)\|.$$
    
    Note that $\|G_t\| = \frac{|\sigma^2(1-2t)|}{2c_t^2} \le \frac{\sigma^2}{2c_t^2} = O(c_t^{-2})$. Hence, 
    $$\|\nabla_z h_2\| \le \left( \frac{\sigma^2}{2c_t^2} \right) \left( \frac{2 R_m^2}{c_t^2} \right) = \frac{\sigma^2 R_m^2}{c_t^4}.$$
    
    Thus, $\operatorname{Lip}_z(h_2)(t) = O(c_t^{-4})$. The Lipschitz constant of the full nonlinear term is bounded by the sum of the two parts:
    $$\operatorname{Lip}_z(h)(t) \le \operatorname{Lip}_z(h_1)(t) + \operatorname{Lip}_z(h_2)(t) \le O(c_t^{-2}) + O(c_t^{-4}).$$

    Now, for velocity field:
    $$\nabla_z \hat{v}^\ast(t,z) = \nabla_z (G_t z) + \nabla_z h(t, z; \mathcal{D}_M).$$
    The Jacobian of the linear part $\nabla_z (G_t z)$ is simply the matrix $G_t$, and so:
    $$\nabla_z \hat{v}^\ast(t,z) = G_t + \nabla_z h(t, z; \mathcal{D}_M).$$
    Since $\|G_t\| = O(c_t^{-2})$, we have:
    $$\operatorname{Lip}_z(\hat{v}^\ast) \leq  \sup_z \|\nabla_z \hat{v}^\ast\| \leq \|G_t\| + \sup_z  \|\nabla_z h(t, z; \mathcal{D}_M)\|   \le O(c_t^{-2}) + O(c_t^{-4}),$$
    which is $O(c_t^{-4})$ as $c_t \to 0.$
\end{proof}

\subsection{Proof of Proposition \ref{prop_truncate}}
\label{proof_truncate}

\begin{proof}[Proof of Proposition \ref{prop_truncate}]
    Let $t \in [0,1]$ and  $\mathcal{D}_M$ be given.
    Let $h(t, z; \mathcal{D}_M) = \sum_{j=1}^M \alpha_j B_t^{(j)}$, where $B_t^{(j)} := y_j(t) =  \dot{m}_t^{(j)} - G_t m_t^{(j)}$. Then, we can rewrite: $\hat{v}_t^\ast(z) = \hat{v}^*(t,z) =  G_t z + h(t, z; \mathcal{D}_M)$.
    
    Let $S_R = \sum_{j \in \mathcal{I}_R} \alpha_j$ be the weight mass of the top-R elements. We have:
    $$\hat{v}_{t,R}(z) = G_t z + \frac{\sum_{j \in \mathcal{I}_R} \alpha_j B_t^{(j)}}{S_R}.$$ 
    
    Let $\bar{B}_R = \frac{\sum_{j \in \mathcal{I}_R} \alpha_j B_t^{(j)}}{S_R}$ denote the top-R weighted average. The approximation error is $$ \|\hat{v}_t^\ast - \hat{v}_{t,R}\| = \|(G_t z +h(t, z; \mathcal{D}_M)) - (G_t z + \bar{B}_R)\| = \|h(t, z; \mathcal{D}_M) - \bar{B}_R\|. $$
    
    We can split $h(t, z; \mathcal{D}_M)$ into top-R and remaining terms:
    $$h(t, z; \mathcal{D}_M) = \sum_{j \in \mathcal{I}_R} \alpha_j B_t^{(j)} + \sum_{j \notin \mathcal{I}_R} \alpha_j B_t^{(j)} = S_R \bar{B}_R + (1-S_R) \bar{B}_{\neg R},$$
    where $\bar{B}_{\neg R}$ is the weighted average of the remaining terms.
    
    Substituting this into the error:
    \begin{align}
        \|h(t, z; \mathcal{D}_M) - \bar{B}_R\| &= \|S_R \bar{B}_R + (1-S_R) \bar{B}_{\neg R} - \bar{B}_R \| \\
        &= \|(S_R - 1) \bar{B}_R + (1-S_R) \bar{B}_{\neg R}\| \\
        &= \|(1-S_R) (\bar{B}_{\neg R} - \bar{B}_R)\| = (1-S_R) \|\bar{B}_{\neg R} - \bar{B}_R\|.
    \end{align}  
    
    By assumption, $\|B_t^{(j)}\| \le C$ for all $j$ and $t$.  Using the triangle inequality, we obtain $\|\bar{B}_{\neg R} - \bar{B}_R\| \le \|\bar{B}_{\neg R}\| + \|\bar{B}_R\| \le C + C = 2C$. Thus, the error is bounded uniformly by $(1-S_R) (2C) = 2 C \left( 1 - \sum_{j \in \mathcal{I}_R} \alpha_j \right)$, which is the bound that we wanted to show.
\end{proof}

Lastly, we provide some remarks on the proposed ODE sampler. 
Although the mixture term $h(z,t)\in\text{Conv}\{B_t^{(1)},\dots,B_t^{(M)}\},$ where $\text{Conv}$ denotes convex hull, the full velocity $G_t z + h(z,t; \mathcal{D}_M)$ needs not lie in the convex hull of observed velocities, nor must the trajectories remain inside convex hulls of the data. In the zero-correction limit ($G_t=0$), an  Euler step satisfies $z_{n+1}-z_n = \Delta t \sum_j \alpha_j(t_n, z_n)\dot m_t^{(j)},$ a convex combination of observed instantaneous velocities. Once $G_t\neq 0$, if we use an exponential Euler scheme, then the linear map $e^{G_t\Delta t}$ induces geometric distortion, causing amplification/contraction of the coordinates and moving the trajectory away from the convex hull.

\section{Additional  Theoretical Results and Insights on the Proposed Training-Free Model}
\label{app:duhamel_dm}

This section further studies the properties of the proposed training-free ODE sampler. We shall: (i) provide existence and uniqueness result, then study the Duhamel (variation-of-constants) representation and explore the implications;  and (ii) provide interpretation for the nonlinear forcing term in the velocity field as a Nyström-extended, row-stochastic diffusion-map operator \citep{coifman2006diffusion, lafon2004diffusion} applied to certain velocity labels. 

\subsection{Global Well-Posedness and Duhamel's Representation}
\label{subsec:bb_sampler_duhamel}

First, we provide a result on existence and uniqueness of the solution to the proposed ODE sampler.

\begin{theorem}[Global well-posedness]\label{thm:ode_wellposed}
Assume $\sigma_{\min}>0$ so that $c_t\ge \sigma_{\min}$ for all $t\in[0,1]$.
Assume there exist constants $R_m,R_1>0$ such that for all $j\in\{1,\dots,M\}$ and all $t\in[0,1]$,
\[
\|m_t^{(j)}\|\le R_m,\qquad \|\dot m_t^{(j)}\|\le R_1.
\]
Let $\hat{v}^*(t,z)=G_t z + h(t,z;\mathcal D_M)$ with
$G_t=\frac{\sigma^2(1-2t)}{2c_t^2}I_d$ and $h$ as in \eqref{eq:bb-vstar}. Then:
\begin{enumerate}
\item[(i)] $\hat{v}^*(t,\cdot)$ is globally Lipschitz in $z$, uniformly in $t\in[0,1]$.
\item[(ii)] $\hat{v}^*$ has linear growth: there exist constants $A,B<\infty$ such that
$\|\hat{v}^*(t,z)\|\le A+B\|z\|$ for all $(t,z)\in[0,1]\times\R^d$.
\end{enumerate}
Consequently, for any initial condition $Z_0\in\R^d$, the ODE
\[
\dot Z_t=\hat{v}^*(t,Z_t),\qquad t\in[0,1],
\]
admits a unique solution on $[0,1]$. Moreover, solutions depend continuously on $Z_0$.
\end{theorem}

\begin{proof}
Since $c_t\ge\sigma_{\min}$,
\[
\|G_t\|=\frac{\sigma^2|1-2t|}{2c_t^2}\le \frac{\sigma^2}{2\sigma_{\min}^2}=:C_G<\infty.
\]
By Theorem~\ref{prop_grads_long},
\[
\operatorname{Lip}_z(h)(t)\le \frac{2R_m}{c_t^2}\big(R_1+\|G_t\|R_m\big)
\le \frac{2R_m}{\sigma_{\min}^2}\big(R_1+C_G R_m\big)=:L_h<\infty.
\]
Since $z\mapsto G_t z$ has Lipschitz constant $\|G_t\|\le C_G$,
\[
\operatorname{Lip}_z(\hat{v}^*)(t)\le \|G_t\|+\operatorname{Lip}_z(h)(t)\le C_G+L_h=:L<\infty,
\]
uniformly in $t\in[0,1]$. This proves (i).

For (ii), we bound
\[
\|\hat{v}^*(t,z)\|
\le \|G_t\|\,\|z\|+\|h(t,z)\|.
\]
Moreover,
\[
\|h(t,z)\|\le \sum_{j=1}^M \alpha_j(t,z)\,\|y_j(t)\|
\le \sup_{1\le j\le M}\|y_j(t)\|
\le R_1+\|G_t\|R_m
\le R_1+C_G R_m.
\]
Hence
\[
\|\hat{v}^*(t,z)\|\le C_G\|z\|+(R_1+C_G R_m),
\]
which is linear growth.

Since $\hat{v}^*$ is continuous in $t$ and globally Lipschitz in $z$ uniformly on $[0,1]$,
Picard--Lindel\"of implies existence and uniqueness of a solution on $[0,1]$ for any $Z_0\in\R^d$
(see, e.g. \citep{hartman2002ordinary}). Continuous dependence on $Z_0$ follows from Gr\"onwall's inequality.
\end{proof}

With the uniqueness of the solution to the ODE established, the forecast sequence obtained by repeatedly integrating the ODE is well defined  and forms a data-driven Markov chain.
To analyze its structure for 1-step forecasting, we now study Duhamel's representation of the ODE~solution.

\paragraph{Duhamel's representation}
The proposed sampler is the nonlinear ODE:
\begin{equation}\label{eq:cfm_ode}
\dot Z_t = G_t Z_t + h(t,Z_t;\mathcal D_M),
\qquad
Z_0\sim\mathcal N\!\big(x_\tau, \sigma_{\min}^2 I_d\big).
\end{equation}
Let $\Phi(t,s)$ denote the fundamental matrix for $\dot z=G_t z$.
Then the standard variation-of-constants identity gives, for any $T\le 1$ and $t\in[0,T]$,
\begin{equation}\label{eq:duhamel_Z}
Z_t=\Phi(t,0)Z_0+\int_0^t \Phi(t,s)\,h(s,Z_s;\mathcal D_M)\,ds.
\end{equation}
Equivalently, with the co-moving variable $Y_t:=\Phi(t,0)^{-1}Z_t$,
\begin{equation}\label{eq:duhamel_Y}
Y_t=Y_0+\int_0^t \Phi(s,0)^{-1}\,h\!\big(s,\Phi(s,0)Y_s;\mathcal D_M\big)\,ds.
\end{equation}

\begin{proposition}[Duhamel representation for the solution of the proposed ODE sampler]
\label{prop:Phi_closed}
The fundamental matrix is
\begin{equation}\label{eq:Phi_closed_form}
\Phi(t,s)=\frac{c_t}{c_s}\,I_d,
\end{equation}
and thus the ODE solution admits the following representation: for $t \in (0,1]$, 
\begin{equation} \label{eq_analyticalsol}
    Z_t = \sqrt{1 + \left(\frac{\sigma}{\sigma_{\min}} \right)^2 t (1-t) } Z_0 + \sum_{j=1}^M \int_0^t \sqrt{\frac{\sigma_{\min}^2 + \sigma^2 t(1-t)}{\sigma_{\min}^2 + \sigma^2 s(1-s)}} \alpha_j(s, Z_s) y_j(s) \ ds.
\end{equation}
\end{proposition}

\begin{proof}
Let $d(t):=c_t^2$. Then $d'(t)=\sigma^2(1-2t)$ and
$g(t)=\frac{\sigma^2(1-2t)}{2c_t^2}=\frac{d'(t)}{2d(t)}$.
Since $G_t=g(t)I_d$ commutes with itself at all times,
\[
\Phi(t,s)=\exp\!\Big(\int_s^t G_r\,dr\Big)
=\exp\!\Big(\int_s^t g(r)\,dr\Big)I_d
=\exp\!\Big(\tfrac12\int_s^t \tfrac{d'(r)}{d(r)}\,dr\Big)I_d
=\sqrt{\frac{d(t)}{d(s)}}\,I_d
=\frac{c_t}{c_s}\,I_d.
\]
Plugging the above formula into \eqref{eq:duhamel_Z}, we obtain the desired  expression for the ODE solution. 
\end{proof}

Therefore, we see that the one-step ODE  forecaster behaves like a nonparametric smoother that averages over past transitions, rather than learning a single Markov transition rule, which partially explains why the proposed training-free model can be competitive with other sequence models on nonlinear dynamics benchmark tasks. To better see this, consider the case when $\sigma = 0$ for the Duhamel's representation, in which case the formula \eqref{eq_analyticalsol}  simplifies to: 
\begin{equation}
    Z_t = Z_0 + \sum_{j=1}^M \beta_j(t; \mathcal{D}_M) \cdot (X_2^{(j)} - X_1^{(j)}),
\end{equation}
where $$\beta_j(t; \mathcal{D}_M) = \int_0^t \frac{\exp(-\|Z_s - m_s^{(j)}\|^2/(2\sigma_{\min}^2))}{\sum_{k=1}^M \exp(-\|Z_s - m_s^{(k)}\|^2/(2\sigma_{\min}^2))} ds \geq 0, \quad m_s^{(j)} = (1-s) X_1^{(j)} + s X_2^{(j)},$$
and $\sum_{j=1}^M \beta_j(t; \mathcal{D}_M) = t$.
This structural representation tells us that the ODE  
sampler evolves by mixing past transition directions, and the mixing coefficients are integrals over time of similarity weights along the trajectory. In particular, the displacement $Z_1 - Z_0$ is a time-scaled convex combination of stored increments with time-accumulated attention weights. For a small $\delta > 0$, we can approximate, for every $s \in [0, 1-\delta]$, as:
$$\frac{Z_{s + \delta} - Z_s}{\delta} \approx  \sum_{j=1}^M \frac{\exp(-\|Z_s - m_s^{(j)}\|^2/(2\sigma_{\min}^2))}{\sum_{k=1}^M \exp(-\|Z_s - m_s^{(k)}\|^2/(2\sigma_{\min}^2))} \cdot (X_2^{(j)} - X_1^{(j)}),$$
which is a data-adaptive Nadaraya-Watson kernel estimator on all the past transition directions.  

Importantly, for $\sigma>0$ the variance schedule $c_t^2=\sigma_{\min}^2+\sigma^2 t(1-t)$ induces a time-dependent smoothing and a linear drift term (via $g_t=\dot c_t/c_t$), interpolating between sharper, memory-based dynamics and more regularized transport. 
From a dynamical-systems viewpoint, when the kernel bandwidth is small (so that the responsibilities concentrate), the velocity field becomes close to a piecewise-defined mixture and the ODE can be viewed as an approximate switching system that follows the locally most similar stored transition.

\subsection{Connections to Diffusion Geometry, Markov Operators and Diffusion Maps}
\label{subsec:diffusion_geometry}

\paragraph{Background on diffusion geometry and Markov operators}
Much of the Riemannian geometry of a smooth manifold can be expressed in terms of its Laplacian
operator $\Delta=\mathrm{div}\circ\nabla$ via the \emph{carr\'e du champ} identity \citep{bakry2013analysis}:
\begin{equation}\label{eq:carre_du_champ}
\Gamma(f,h)
:=
\frac12\big(\Delta(fh) - f\Delta h - h\Delta f \big)
=
\langle \nabla f, \nabla h\rangle,
\end{equation}
which recovers the Riemannian metric on $1$-forms from the action of $\Delta$.
A central idea (diffusion geometry \citep{jones2024diffusion}) is to replace the Laplacian by a more general operator $L$
defined on an abstract state space, and to interpret \eqref{eq:carre_du_champ} as defining the
intrinsic geometry associated with $L$.

In many settings, $L$ arises as the \emph{infinitesimal generator} of a \emph{Markov semigroup}.
A family $(P_t)_{t\ge 0}$ of linear operators acting on a function space is called a Markov
semigroup if:
(i) $P_0=\mathrm{Id}$,
(ii) $P_tP_s=P_{t+s}$ for all $t,s\ge 0$, and
(iii) $P_t$ preserves positivity and constant functions.
The generator is defined by
\begin{equation}\label{eq:generator_def}
Lf := \lim_{t\downarrow 0}\frac{P_t f - f}{t},
\end{equation}
whenever the limit exists.
Rather than working directly with generators, one data-driven approach typically proceeds by constructing
\emph{finite-time} Markov diffusion operators from data; the generator and the associated geometry
emerge in suitable small-time or small-bandwidth limits.

Basic examples of Markov semigroup include: 
\begin{itemize}
\item \textbf{Heat diffusion on $\mathbb{R}^d$.}
Let $p_t(x,y)$ denote the Gaussian heat kernel
\[
p_t(x,y)
=
\frac{1}{(4\pi t)^{d/2}}
\exp\!\left(-\frac{\|x-y\|^2}{4t}\right),
\qquad t>0.
\]
It defines the Markov semigroup
\[
(P_t f)(x)
=
\int_{\mathbb{R}^d} p_t(x,y)\,f(y)\,dy,
\]
with infinitesimal generator $L=\Delta$.

\item \textbf{Weighted manifolds.}
Let $\mathcal M$ be a smooth compact Riemannian manifold and let
$\pi(dy)=\mu(y)\,d\mathrm{Vol}_{\mathcal M}(y)$ with $\mu>0$ smooth.
Define the weighted Laplacian
\[
Lf
=
\mu^{-1}\mathrm{div}(\mu\nabla f)
=
\Delta f + \nabla\log\mu\cdot\nabla f,
\]
which is symmetric in $L^2(\mathcal M,\pi)$ and admits $\pi$ as an invariant measure.
Let $p_t^\mu(x,y)$ denote the heat kernel associated with $L$ with respect to the measure $\pi$:
\[
(P_t f)(x)
=
\int_{\mathcal M} p_t^\mu(x,y)\,f(y)\,\pi(dy),
\qquad
\int_{\mathcal M} p_t^\mu(x,y)\,\pi(dy)=1.
\]
Then $(P_t)_{t\ge 0}$ is a Markov semigroup with generator $L$.

\item \textbf{General Markov processes.}
If $(X_t)_{t\ge 0}$ is a Markov process on a measurable space, then
\[
(P_tf)(x)=\mathbb{E}[f(X_t)\mid X_0=x]
\]
defines a Markov semigroup. The earlier two examples are specific instantiations within this framework.
\end{itemize}

\paragraph{Diffusion maps} 
Diffusion maps \citep{coifman2006diffusion, lafon2004diffusion, nadler2006diffusion} are a family of kernel-based methods that construct intrinsic geometric
representations of data by interpreting a normalized kernel matrix as a Markov diffusion operator on the data manifold. In the classical setting, one studies the spectrum of this
operator and uses its leading eigenfunctions as low-dimensional intrinsic coordinates. Here, we are mainly interested in relating the proposed ODE sampler to diffusion maps, and do not perform spectral decomposition; instead, we directly use the
associated (population and empirical) diffusion operators and their Nystr\"om extensions
to transport vector-valued observables.

Fix a sampler time $s\in[0,1]$ and consider an intrinsic point cloud
$\tilde{\mathcal M}_s=\{\tilde m_s^{(j)}\}_{j=1}^M\subset\mathbb{R}^d$,
viewed as samples from a probability measure $\pi_s$ supported on a smooth manifold.
Given a bandwidth $\varepsilon_s>0$, define the Gaussian kernel
\[
k_{\varepsilon_s}(x,y)
=
\exp\!\left(-\frac{\|x-y\|^2}{2\varepsilon_s}\right).
\]
The associated population diffusion operator is
\begin{equation}\label{eq:T_eps_def}
(T_{\varepsilon_s} f)(x)
=
\frac{\int k_{\varepsilon_s}(x,y)f(y)\,\pi_s(dy)}
{\int k_{\varepsilon_s}(x,y)\,\pi_s(dy)}.
\end{equation}

Throughout, we use a hat to denote finite-sample or empirical objects constructed from
the point cloud $\tilde{\mathcal M}_s$ via Nystr\"om approximation.
In particular, $\hat K(s)$, $\hat D(s)$ and $\hat W(s)$ denote the empirical kernel,
degree and random-walk matrices, while $\hat P_s$ denotes the corresponding empirical
diffusion operator evaluated at arbitrary query points.
The unhatted objects correspond to population-level integral operators.

Define
\[
\hat K_{ij}(s)
:=
\exp\!\left(-\frac{\|\tilde m_s^{(i)}-\tilde m_s^{(j)}\|^2}{2\varepsilon_s}\right),
\qquad
\hat D(s):=\mathrm{diag}(\hat K(s)\mathbf 1),
\qquad
\hat W(s):=\hat D(s)^{-1}\hat K(s).
\]
For arbitrary $y\in\mathbb{R}^d$, the Nystr\"om weights are
\[
\hat w_j(y,s)
:=
\frac{\exp\!\left(-\frac{\|y-\tilde m_s^{(j)}\|^2}{2\varepsilon_s}\right)}
{\sum_{k=1}^M \exp\!\left(-\frac{\|y-\tilde m_s^{(k)}\|^2}{2\varepsilon_s}\right)},
\]
and the empirical diffusion operator is
\begin{equation}\label{eq:hatP_def}
(\hat P_s f)(y)
:=
\sum_{j=1}^M \hat w_j(y,s)\,f(\tilde m_s^{(j)}).
\end{equation}

Here ``Nystr\"om" refers both to the finite-sample approximation of the population integral
operator and to its evaluation at out-of-sample query points via kernel weights, as is standard
in the diffusion-maps literature. 

\paragraph{Markov interpretation}
The row-normalized kernel $\hat W(s)=\hat D(s)^{-1}\hat K(s)$ is a stochastic matrix and thus
defines a discrete-time Markov chain on the cloud $\tilde{\mathcal M}_s$, with transition
probabilities
\[
\mathbb P(X_{n+1}=\tilde m_s^{(j)}\mid X_n=\tilde m_s^{(i)})=\hat W_{ij}(s).
\]
The empirical diffusion operator $\hat P_s$ is the associated conditional expectation operator,
\[
(\hat P_s f)(\tilde m_s^{(i)})
=
\mathbb E[f(X_{n+1})\mid X_n=\tilde m_s^{(i)}],
\]
while its Nystr\"om extension \eqref{eq:hatP_def} provides transition probabilities
$\hat w(\cdot\mid y,s)$ from an arbitrary query point $y\in\mathbb R^d$.

At the population level, $T_{\varepsilon_s}$ admits the representation
\[
T_{\varepsilon_s}f(x)=\int f(y)\,q_{\varepsilon_s}(x,dy),
\]
with $q_{\varepsilon_s}$ a Markov transition kernel.
In the joint limit $M\to\infty$ and $\varepsilon_s\to0$ (with suitable scaling),
$\hat P_s$ converges to $T_{\varepsilon_s}$ and
$(T_{\varepsilon_s}-\mathrm{Id})/\varepsilon_s$
recovers, up to a constant factor, a (possibly density-weighted) Laplace--Beltrami operator.

\paragraph{Intrinsic coordinates and diffusion-scale invariance}
Define intrinsic coordinates
\begin{equation}\label{eq:rescaled_coords}
\tilde m_t^{(j)}:=\Phi(t,0)^{-1}m_t^{(j)},\qquad j=1,\dots,M,
\end{equation}
and intrinsic diffusion scale
\begin{equation}\label{eq:tilde_bandwidth_def}
\tilde c_t := \frac{c_t}{\varphi(t,0)},\qquad \varepsilon_t:=\tilde c_t^2,
\end{equation}
where $\Phi(t,0)=\varphi(t,0)I_d$.
For the proposed ODE sampler,
\begin{equation}\label{eq:comoving_bandwidth_constant}
\Phi(t,0)=\frac{c_t}{c_0}I_d
\quad\Rightarrow\quad
\tilde c_t\equiv c_0,
\qquad
\varepsilon_t\equiv c_0^2.
\end{equation}

\begin{proposition}[Kernel equivariance and intrinsic labels]
\label{prop:equiv_intrinsic}
Fix $t\in[0,1]$ and let $y\in\mathbb{R}^d$.
Then:
\begin{enumerate}
\item[(a)] \textbf{(Kernel equivariance)}
\[
\alpha_j\!\big(t, \Phi(t,0)y\big)
=
\hat w_j(y,t).
\]

\item[(b)] \textbf{(Intrinsic label identity)}
\[
\Phi(t,0)^{-1}y_j(t)=\dot{\tilde m}_t^{(j)}.
\]
\end{enumerate}
\end{proposition}
\begin{proof}
\emph{(a)}
Since $\Phi(t,0)=\varphi(t,0)I_d$,
\[
\|\Phi(t,0)y-m_t^{(j)}\|^2
=\varphi(t,0)^2\|y-\tilde m_t^{(j)}\|^2.
\]
Substituting into the definition of $\alpha_j$ and using
$c_t^2/\varphi(t,0)^2=\tilde c_t^2$ yields $\alpha_j(t, \Phi(t,0)y)=\hat w_j(y,t)$.

\emph{(b)}
Differentiating $m_t^{(j)}=\Phi(t,0)\tilde m_t^{(j)}$ gives
\[
\dot m_t^{(j)}
=G_t m_t^{(j)}+\Phi(t,0)\dot{\tilde m}_t^{(j)}.
\]
Thus $y_j(t)=\Phi(t,0)\dot{\tilde m}_t^{(j)}$ and left-multiplying by
$\Phi(t,0)^{-1}$ yields the result.
\end{proof}

It then follows directly from Proposition~\ref{prop:equiv_intrinsic} that, for every $s\in[0,1]$ and every $y\in\mathbb{R}^d$,
\begin{equation}\label{eq:dm_integrand_exact}
\Phi(s,0)^{-1} h\!\big(s,\Phi(s,0)y;\mathcal D_M\big)
=
(\hat P_s\,\dot{\tilde m}(s))(y),
\end{equation}
where $\dot{\tilde m}(s)(\tilde m_s^{(j)})=\dot{\tilde m}_s^{(j)}$.
On-cloud,
\begin{equation}\label{eq:dm_integrand_oncloud}
\Phi(s,0)^{-1} h\!\big(s,m_s^{(i)};\mathcal D_M\big)
=
\sum_{j=1}^M \hat W_{ij}(s)\,\dot{\tilde m}_s^{(j)}.
\end{equation}

\paragraph{Discretizations as Nystr\"om regression steps}  
Eq.~\eqref{eq:dm_integrand_exact} shows that, in co-moving
coordinates, the nonlinear forcing is exactly a Nystr\"om extension of the
empirical diffusion operator applied to the intrinsic velocity labels.
Consequently, standard time discretizations of the sampler alternate the
known linear transport with evaluations of this data-dependent regression
operator along the transformed trajectory.

\paragraph{Relation to diffusion forecasting}
Diffusion forecasting \citep{harlim2018data, berry2015nonparametric} builds a data-driven Markov
operator from a kernel on a point cloud and uses its eigenfunctions to propagate
densities/expectations under an (unknown) generator. Here, the same kernel--Markov machinery
appears inside the Lagrangian sampler: at frozen time $s$, the intrinsic forcing is exactly a Nystr\"om extension of a row-stochastic diffusion-map random walk applied to the vector-valued observable
$\dot{\tilde m}_s^{(j)}$. Thus the map
$y\mapsto \sum_{j=1}^M \hat{w}_j(y,s)\,\dot{\tilde m}_s^{(j)}$
is precisely the Nadaraya--Watson (kernel conditional expectation) estimator of the intrinsic
velocity observable on $\tilde{\mathcal M}_s$.
Unlike classical diffusion maps, we do \emph{not} perform eigendecomposition or spectral truncation:
the full operator is evaluated locally along the sampler trajectory.

\section{Experimental Details and Additional Results}
\label{app:experiments}

We deliberately focus on the Dysts benchmark as it provides a controlled and diverse collection of chaotic systems with varying levels of predictability and complexity. Unlike real-world datasets, Dysts enables systematic evaluation across many distinct dynamical regimes, making it particularly suitable for assessing generalization in forecasting tasks.

Moreover, prior work \citep{zhang2024zero} uses Dysts to evaluate the zero-shot forecasting ability of large foundation models. In contrast, we use the same benchmark to evaluate whether FreeFM can capture and extrapolate dynamics without training. This allows for a comparison between zero-shot inference from pretrained models and our training-free model, highlighting that strong forecasting performance can be achieved without reliance on large-scale pretraining.

Unless otherwise stated, the main Dysts experiments evaluate the full
transition bank; top-\(R\) truncation is studied separately as an ablation.

\subsection{Baseline Model Details}
\label{subsec:BaselineModelDetails}
Our baseline models follow the model design and range of hyperparameter used in previous studies ~\citep{gilpin2021chaos,zhang2024zero,lipman2022flow}. The qualitative details can be obtained from those works. We use the reference implementation and hyperparameter settings from Darts library for all baseline model other than vanilla flow matching model. For vanilla flow matching model we follow settings in prior work~\citep{lipman2022flow}. For simplicity, we only tune one important hyperparameter. The hyperparameter details are shown below.

\textbf{Vanilla Flow Matching} ~\citep{lipman2022flow}
\begin{enumerate}
    \item[-] \textit{Input Length: \{0.05, 0.25 , 0.5, 0.75, 1\} Lyapunov times}
    \item[-] \textit{Hidden Dimension: 256}
    \item[-] \textit{Time Embedding Dimension: 64}
    \item[-] \textit{Number of Residual Blocks: 4}
    \item[-] \textit{Dropout Fraction: 0.1}
    \item[-] \textit{Activation Function: ReLU}    
\end{enumerate}

\textbf{N-BEATS} ~\citep{oreshkin2020nbeatsneuralbasisexpansion}
\begin{enumerate}
    \item[-] \textit{Input Length: \{0.05, 0.25 , 0.5, 0.75, 1\} Lyapunov times}
    \item[-] \textit{Number of Stacks: 30} 
    \item[-] \textit{Number of Blocks: 1}
    \item[-] \textit{Number of Layers: 4}
    \item[-] \textit{Expansion Coefficient Dimension: 5}
    \item[-] \textit{Layer Widths: 256}
    \item[-] \textit{Degree of Trend Polynomial: 2}
    \item[-] \textit{Dropout Fraction: 0.0} 
    \item[-] \textit{Activation Function: ReLU}
\end{enumerate}

\textbf{Transformer} ~\citep{vaswani2017attention}
\begin{enumerate}
    \item[-] \textit{Input Length: \{0.05, 0.25 , 0.5, 0.75, 1\} Lyapunov times}
    \item[-] \textit{Number Attention Heads: 4}
    \item[-] \textit{Number Encoder Layers: 3}
    \item[-] \textit{Number Decoder Layers: 3}
    \item[-] \textit{Feedforward Dimension: 512}
    \item[-] \textit{Dropout Fraction: 0.1}
    \item[-] \textit{Activation Function: ReLU}
\end{enumerate}

\textbf{TiDE} ~\citep{das2024longtermforecastingtidetimeseries}
\begin{enumerate}
    \item[-] \textit{Input Length: \{0.05, 0.25 , 0.5, 0.75, 1\} Lyapunov times}
    \item[-] \textit{Number of Encoder Layers: 1}
    \item[-] \textit{Number of Decoder Layers: 1}
    \item[-] \textit{Decoder Output Dimension: 16}
    \item[-] \textit{Hidden Dimension Size: 128}
    \item[-] \textit{Past Temporal Width: 4}
    \item[-] \textit{Future Temporal Width: 4}
    \item[-] \textit{Past Temporal Hidden: None}
    \item[-] \textit{Future Temporal Hidden: None}
    \item[-] \textit{Temporal Decoder Hidden: 32}
    \item[-] \textit{Dropout Fraction: 0.1}
\end{enumerate}

\textbf{LSTM} ~\citep{hochreiter1997long}
\begin{enumerate}
    \item[-] \textit{Input Length: \{0.05, 0.25 , 0.5, 0.75, 1\} Lyapunov times}
    \item[-] \textit{Hidden Dimensionality: 25}
    \item[-] \textit{Number of Recurrent Layers: 2}
    \item[-] \textit{Dropout Fraction: 0.0}
    \item[-] \textit{Training Length: 24}
\end{enumerate}

\textbf{ESN} ~\citep{jaeger2004harnessing}
\begin{enumerate}
    \item[-] \textit{Input Length: \{0.05, 0.25 , 0.5, 0.75, 1\} Lyapunov times}
    \item[-] \textit{Number of Reservoir Units: 500}
    \item[-] \textit{Spectral Radius: 0.8}
    \item[-] \textit{Leak Rate: 0.1}
    \item[-] \textit{Reservoir Connectivity: 0.1}
    \item[-] \textit{Input Scaling: 1.0}
    \item[-] \textit{Input Connectivity: 0.2}
    \item[-] \textit{Ridge Regularization: 1e-4}
\end{enumerate}

\subsection{Evaluation Metrics}
\label{subsec:EvaluationMetrics}

\textbf{Symmetric mean absolute percentage error (sMAPE).} 
sMAPE is an accuracy measure based on percentage (or relative) errors, and it is commonly used in time series prediction and forecasting tasks \citep{koehler2001asymmetry}. It is defined as:
\begin{equation}
    \text{sMAPE}(\mathbf{y},\mathbf{\hat{y}})\equiv\frac{100\%}{n}\sum^{n}_{t=1}\frac{|\mathbf{y}_t-\mathbf{\hat{y}}_t|}{(|\mathbf{y}_t|+|\mathbf{\hat{y}}_t|)/2},
\end{equation}

where $n$ is forecast horizon, $\mathbf{y}_t$ is true value of test time series and $\mathbf{\hat{y}}_t$ is predicted value of the forecast model. sMAPE is bounded between 0\% and 200\%, and penalizes larger over and underestimations in a ``symmetric" manner.

\textbf{Valid prediction time (VPT).} 
VPT measures the first forecast horizon when sMAPE exceeds a fixed threshold $\epsilon$ \citep{gilpin2023modelscaleversusdomain}: 
\begin{equation}
    \text{VPT}\equiv \arg\max_{\tilde{t}}\{\tilde{t}|\text{sMAPE}(\textbf{y}_t,\mathbf{\hat{y}}_t)<\epsilon,\forall t<\tilde{t}\}.
\end{equation}
We set $\epsilon=20$, tighter than prior studies
\citep{gilpin2023modelscaleversusdomain}.

\textbf{Continuous ranked probability score (CRPS).} 
CRPS measures the respective accuracy of two probabilistic forecasting models. For a predictive distribution $F$ and observation $y$, it is the scoring rule defined as \citep{matheson1976scoring}:

\begin{equation}
    \text{CRPS}(F,y)=\int_{-\infty}^{+\infty}[F(x)-\mathbf{1}\{x\geq y\}]^2dx .
\end{equation}

For practical computation, we consider its equivalent form:

\begin{equation}
    \text{CRPS}(F,y)= \mathbb{E}[|X-y|]-\frac{1}{2}\mathbb{E}[|X-X^\prime|] ,
\end{equation}

where $X, X'$ are independent samples drawn from the forecast distribution $F$.
For multivariate forecasts, we compute CRPS componentwise and average over
state dimensions and forecast times.

\textbf{Correlation dimension.} 
Correlation dimension is a measure of the dimensionality of the space occupied by a set of random points, often referred to as a type of fractal dimension. It characterizes how the attractor fills the phase space by measuring the scaling behavior of nearby point pairs \citep{grassberger1983measuring}. 
For a scalar time series $\{x(t)\}$, reconstruct the phase space using time-delay embedding:
\begin{equation}
    \textbf{X}(i)=[x(i), x(i+\tau),\dots,x(i+(m-1)\tau)] ,
\end{equation}
where $m$ is embedding dimension, $\tau$ is time delay, and let total number of embedded vectors be $N$. Then the correlation integral $C(r)$ counting the fraction of point pairs is defined according to the radius $r$:
\begin{equation}
    C(r)=\lim_{N\to\infty}\frac{1}{N^2}\sum_{i,j=1}^N\Theta(r-\Vert\textbf{X}(i)-\textbf{X}(j)\Vert) ,
\end{equation}
where $\Theta(\cdot)$ is Heaviside step function. As in scaling region, a power-law relationship holds $C(r)\propto r^{d_{corr}}$, therefore we have correlation dimension computed as:
\begin{equation}
    d_{corr} = \lim_{r\to0}\frac{\log C(r)}{\log r} .
\end{equation}

\textbf{Kullback–Leibler divergence (KL divergence).} 
KL divergence is a type of statistical distance measuring how much an approximating probability distribution $Q$ different from a true probability distribution $P$ \citep{kullback1951information}. It is defined as:
\begin{equation}
    D_{\text{KL}} (P\Vert Q)=\sum_{x\in\mathcal{X}}P(x)\log\left(\frac{P(x)}{Q(x)}\right). 
\end{equation}

\subsection{Computational Cost}
We provide a computational cost analysis comparing our training-free model with the vanilla flow matching model and the six baseline models. All experiments are conducted on CPUs belonging to an internal cluster. Let $M$ denote the number of transition pairs, $d$ denote the state dimension of the chaotic system, and $H$ denote the forecast horizon. For our training-free model, let $T_g$ denote the time grid size, $N_{\text{ode}}$ denote the number of ODE integration steps, $B$ denote the batch size, and $S$ denote the number of Monte Carlo samples.

Our model is training-free, resulting in minimal computational cost before inference. During the pre-computation phase, the model computes the transition means $m_j$ and velocity correction terms $B_j$, with a complexity of only $\mathcal{O}(T_g \cdot M \cdot d)$. During the single-step inference phase, at each ODE integration step, the model first computes pairwise differences between the current state and the transition means, with complexity $\mathcal{O}(B \cdot M \cdot d)$; then computes the softmax of Gaussian responsibilities, with complexity $\mathcal{O}(B \cdot M)$; and finally computes the nonlinear correction term, with complexity $\mathcal{O}(B \cdot M \cdot d)$. Thus, the total computational cost per integration step is $\mathcal{O}(B \cdot M \cdot d)$. For multi-step probabilistic forecasting with horizon $H$ and $S$ Monte Carlo samples, the total computational cost is:
\begin{equation}
    \mathcal{C}_{\text{dense}} = \mathcal{O}(S \cdot H \cdot N_{\text{ode}} \cdot M \cdot d)  .
\end{equation}

To improve practical efficiency, we introduce a top-$R$ approximation that restricts attention to the $R$ nearest transitions at each query point, where $R\ll M$. This approximation first computes distances to all $M$ transitions, with complexity $\mathcal{O}(B \cdot M \cdot d)$; then performs a partial sort to identify the $R$ nearest neighbors, with complexity $\mathcal{O}(B \cdot M)$; and finally computes the weighted velocity over only the $R$ neighbors, with complexity $\mathcal{O}(B \cdot R \cdot d)$. Thus, the inference complexity for multi-step probabilistic forecasting becomes:

\begin{equation}
    \mathcal{C}_{\text{top-}R} = \mathcal{O}(S \cdot H \cdot N_{\text{ode}} \cdot (M \cdot d + R \cdot d))  .
\end{equation}

According to our dataset settings, for each system we have $20$ trajectories with $812 - 500 = 312$ observed time points, yielding a total of $M = 6220$ transition pairs. Based on this and the baseline model settings described in App.~\ref{subsec:BaselineModelDetails}, we compute the approximate Floating Point Operations (FLOPs) for our training-free model and the baseline models. The results are presented in Fig.~\ref{fig:computationalcost}.
\begin{figure}[!h]
    \centering
    \begin{subfigure}{0.48\textwidth}
        \centering
        \includegraphics[width=\textwidth]{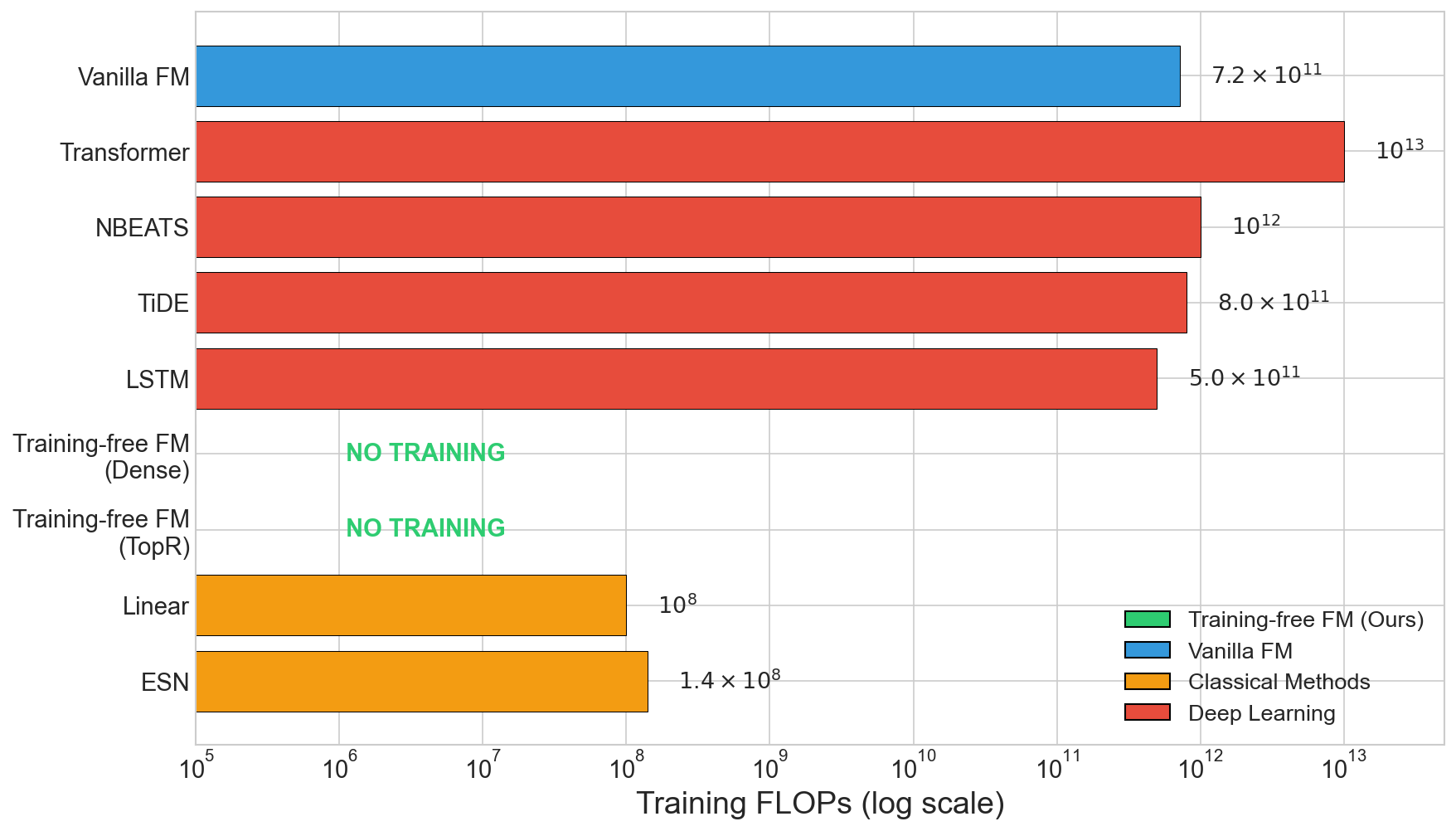}
        \caption{Training FLOPs Comparison}
        \label{fig:trainingflops}
    \end{subfigure}
    \begin{subfigure}{0.48\textwidth}
        \centering
        \includegraphics[width=\textwidth]{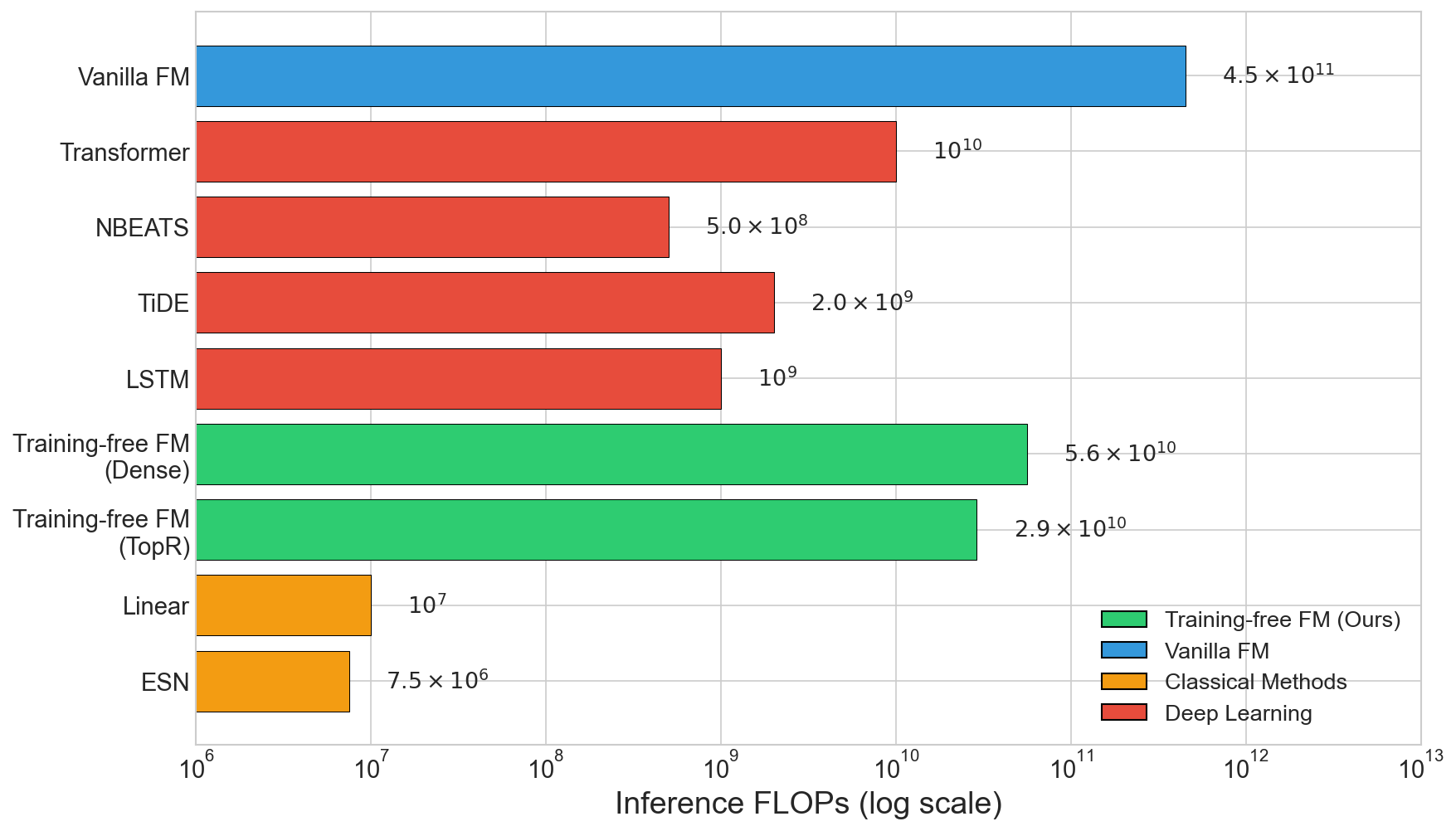}
        \caption{Inference FLOPs Comparison}
        \label{fig:inferenceflops}
    \end{subfigure}
    \caption{\textbf{Computational Cost.} (a) FLOPs Comparison during training phase. (b) FLOPs Comparison during inference phase. All the FLOPs are computed among $6220$ transition pairs.}
    \label{fig:computationalcost}
\end{figure}

\subsection{Wall-Clock Runtime}

The FLOPS counts above do not capture the cost of repeated velocity evaluation along a long autoregressive rollout. In this subsection we additionally report CPU-only wall-clock runtime on Dysts in the setting of Sec.~\ref{sec:exp}, on a single CPU with eight threads.

Table~\ref{tab:runtime-dysts} shows the wall-clock cost comparison. Forecasting of FreeFM takes about 225 seconds per system with full memory bank and 128 seconds with top-$R$ truncation. VanillaFM takes 91 seconds, and the other neural baseline between 14 and 294 seconds. For each velocity evaluation, the empirical field is cheaper than the Vanilla network, and the difference comes from using 100 instead of 20 integration steps per transition. So the forecasting cost of FreeFM is of the same order as the cost of the neural baselines. Top-$R$ truncation reduces it by about 43\% at almost the same accuracy.

\begin{table*}[!h]
\centering
\footnotesize
\caption{CPU-only mean runtime, in seconds, grouped by state dimension.}
\label{tab:runtime-dysts}
\setlength{\tabcolsep}{3pt}
\begin{tabular}{lrrrrrrrrr}
\toprule
& \multicolumn{2}{c}{FreeFM} & & & & & & & \\
\cmidrule(lr){2-3}
Systems & dense & top-$R$ & VanillaFM & TiDE & LSTM & Trans. & N-BEATS & Linear & ESN\\
\midrule
All  & 225.4 & 128.4 & 90.7 & 14.1 & 61.4 & 177.0 & 294.4 & 0.3 & 0.4\\
$d=3$  & 227.3 & 129.4 & 90.6 & 14.1 & 61.4 & 177.0 & 291.7 & 0.3 & 0.4\\
$d=4$  & 212.0 & 118.8 & 90.0 & 14.1 & 61.2 & 176.8 & 295.9 & 0.3 & 0.4\\
$d=5$  & 214.8 & 125.0 & 91.2 & 14.3 & 61.2 & 176.5 & 310.3 & 0.4 & 0.4\\
$d=6$  & 221.3 & 121.9 & 88.5 & 14.0 & 61.4 & 176.3 & 303.8 & 0.4 & 0.4\\
$d=10$ & 244.5 & 147.0 & 95.2 & 14.6 & 62.1 & 178.3 & 323.2 & 0.4 & 0.4\\
\bottomrule
\end{tabular}
\end{table*}

\subsection{Conditional Forecast Examples}
For conditional forecasting, we present additional forecast results from various chaotic systems beyond the Aizawa attractor shown in Fig.~\ref{fig:ConditionalForecastSub1}. These results are provided in Figs.~\ref{fig:henonheiles}-\ref{fig:ForcedBrusselator}.

\begin{figure}[!h]
    \centering
    \includegraphics[width=0.65\textwidth]{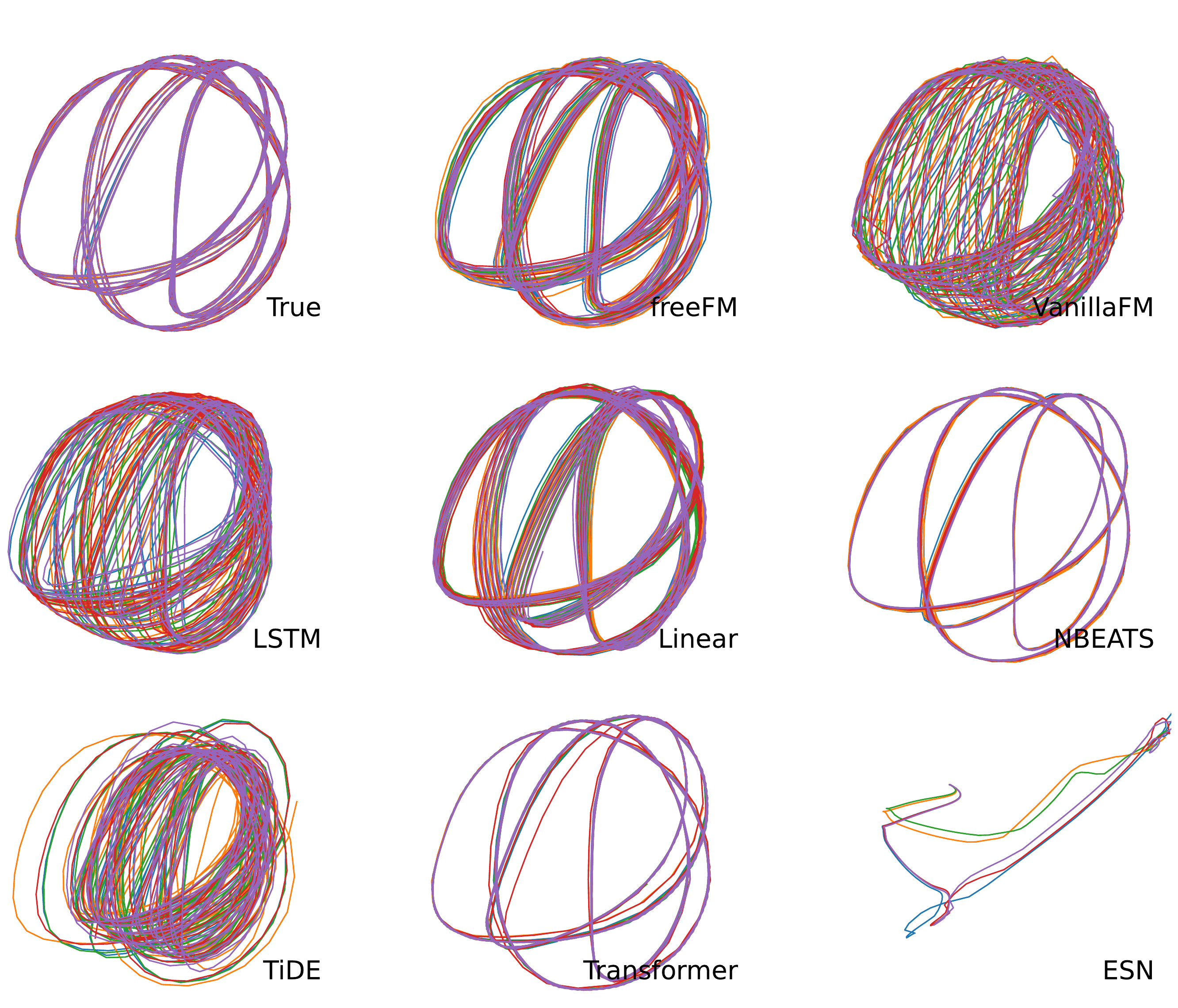}
    \caption{Hénon–Heiles System.}
    \label{fig:henonheiles}
\end{figure}

\begin{figure}[!h]
    \centering
    \includegraphics[width=0.65\textwidth]{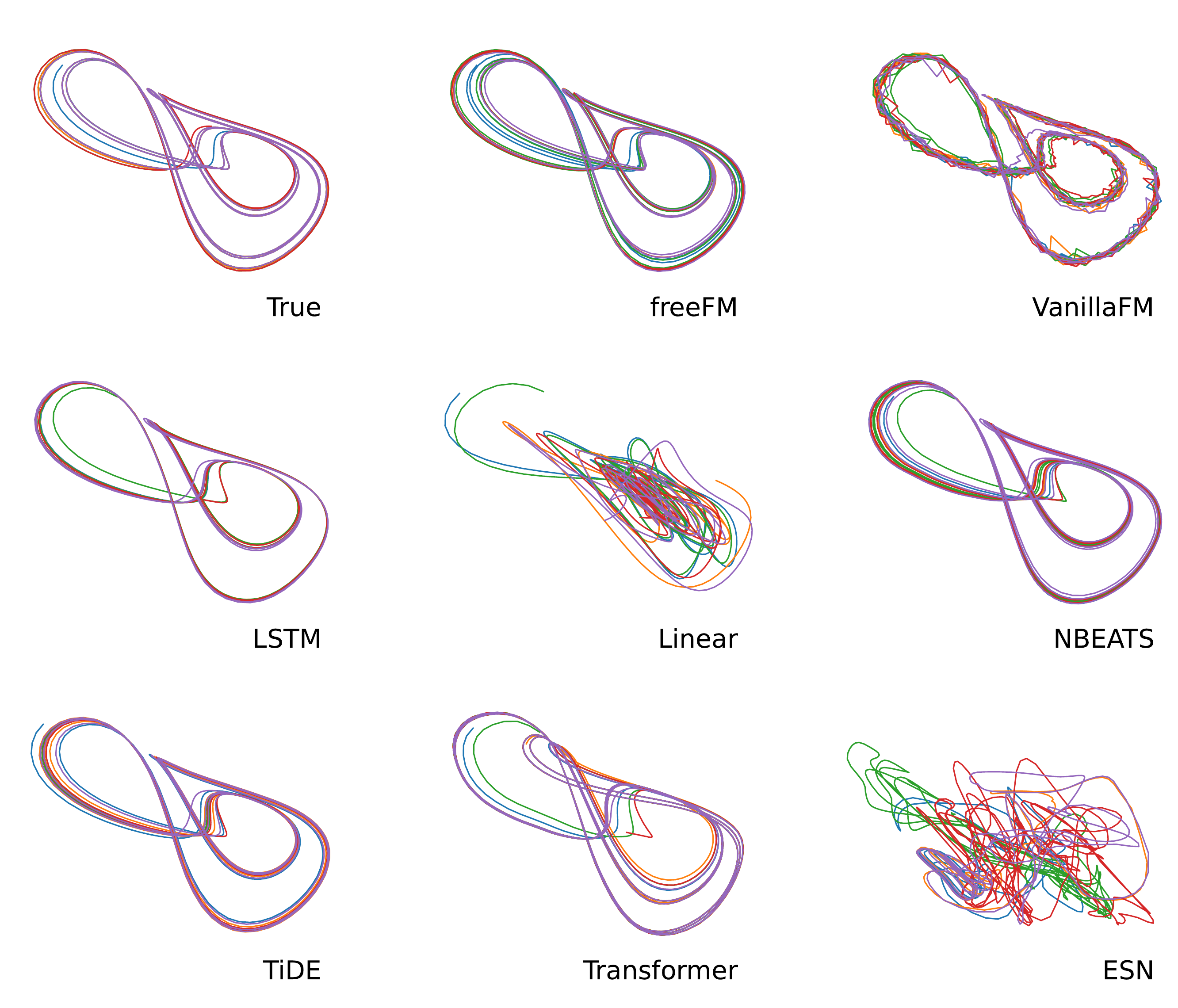}
    \caption{Sprott G System.}
\end{figure}

\begin{figure}[!h]
    \centering
    \includegraphics[width=0.65\textwidth]{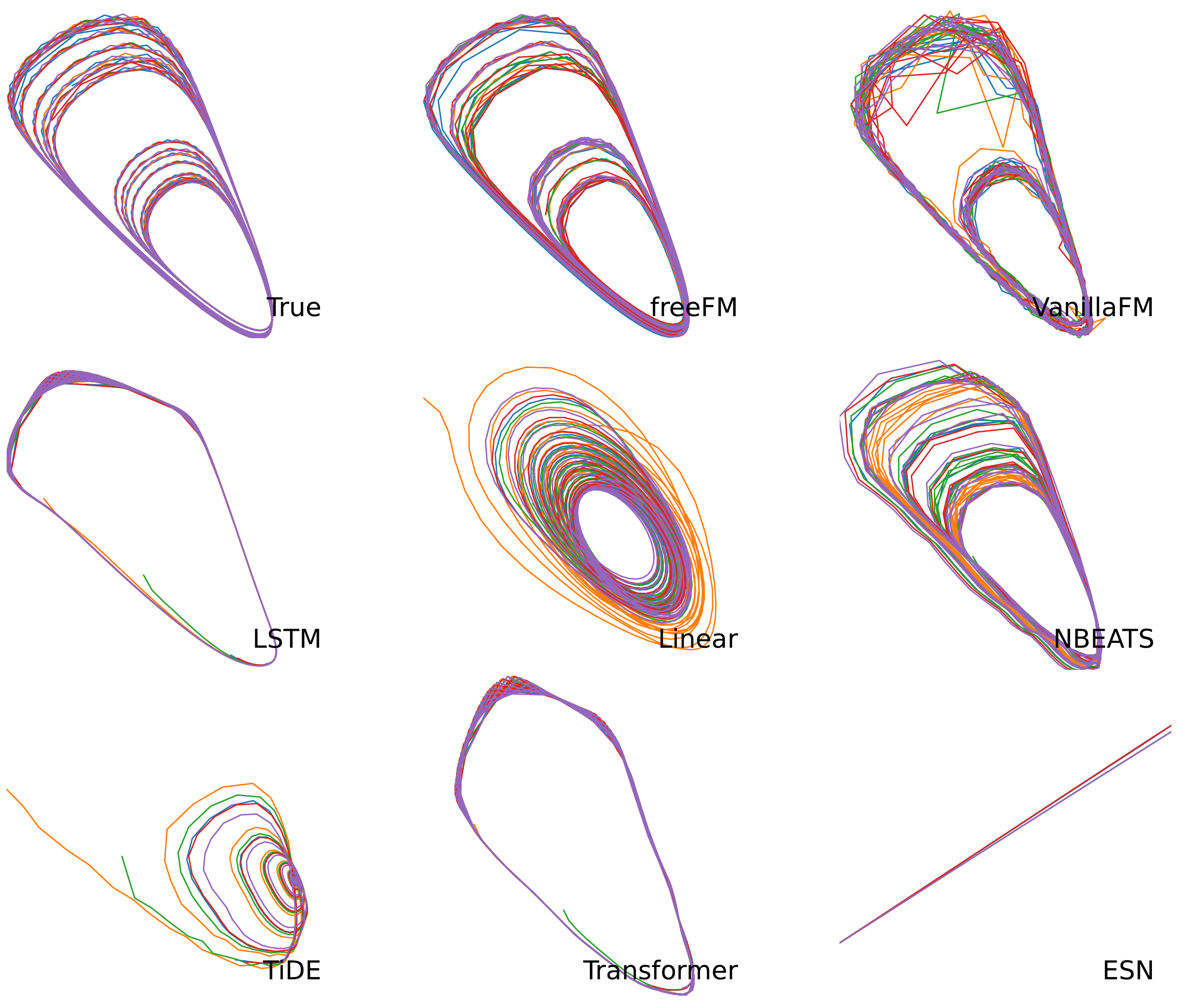}
    \caption{Isothermal Chemical Process.}
\end{figure}

\begin{figure}[!h]
    \centering
    \includegraphics[width=0.65\textwidth]{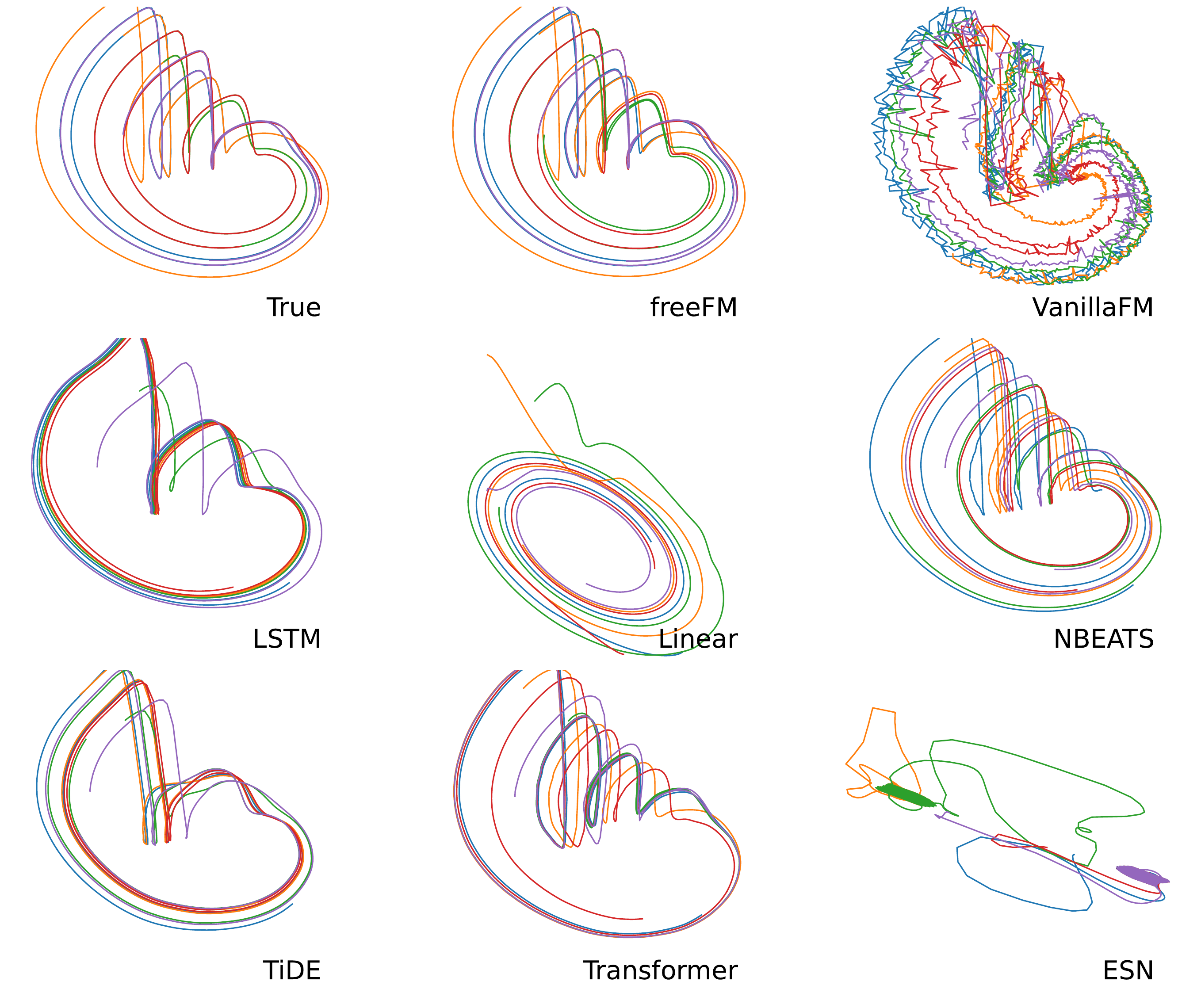}
    \caption{Jerk Circuit Oscillator.}
\end{figure}

\begin{figure}[!h]
    \centering
    \includegraphics[width=0.65\textwidth]{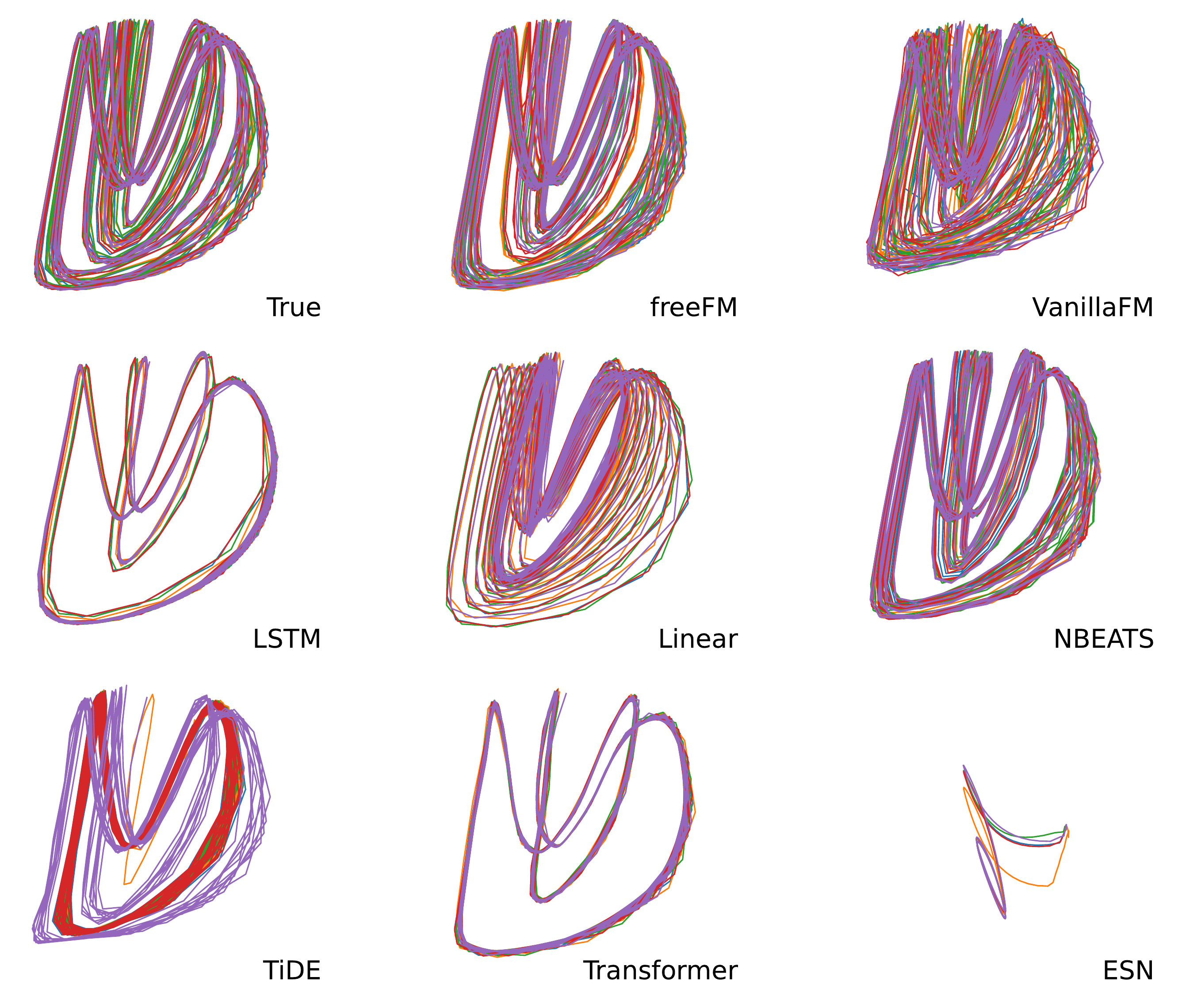}
    \caption{Forced Brusselator System.}
    \label{fig:ForcedBrusselator}
\end{figure}

\clearpage 
\subsection{Ablation Study on ODE Integrators}
\label{app_ablation}

Our training-free model provides a closed-form optimal velocity field and integrates the induced ODE to generate samples. The choice of ODE integrator settings can significantly impact the generated trajectories. In this ablation study, we examine different integrator configurations, including the number of integration steps and the choice of ODE solvers, to investigate whether our model behaves consistently with standard flow matching models across various settings.

We conduct the ablation study under the same setting as Sec.~\ref{subsec:ConditionalForecast}, except that we vary the number of ODE integration steps in $[30,50,100]$ and the ODE solver in [Euler, Runge–Kutta, Exponential Euler], respectively. The results are presented in Fig. \ref{fig:AblationStudy}. From Fig. \ref{fig:AblationStudySub1}, we see that more integration steps lead to better prediction quality. And Fig. \ref{fig:AblationStudySub2} shows that the Euler method is already good enough for our training-free models. These conclusions are identical with flow matching models from previous studies \citep{brinke2025flowmatchinggeometrictrajectory,lipman2022flow}.

We also run an ablation experiment to study if  the top $R$ truncation scheme leads to differences in the proposed training-free model. We set truncation number $R=256$, the result is presented in Fig.~\ref{fig:AblationStudySub3}. We can see that the results are almost the same after we apply truncation regime.

\begin{figure}[!h]
    \centering
    \begin{subfigure}{0.45\textwidth}
        \centering
        \includegraphics[width=\textwidth]{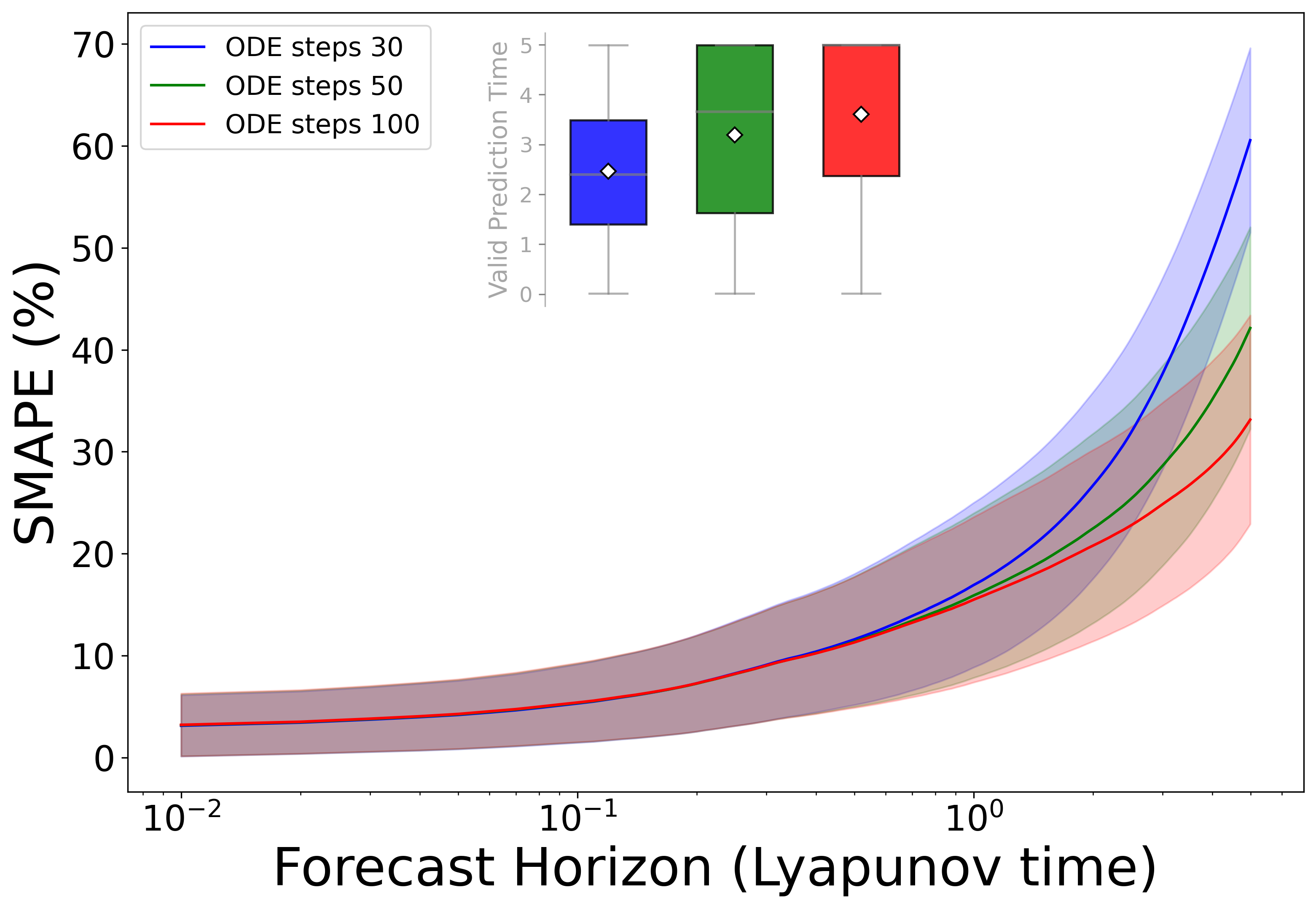}
        \caption{Ablation Study for ODE Integration Steps}
        \label{fig:AblationStudySub1}
    \end{subfigure}
    \begin{subfigure}{0.45\textwidth}
        \centering
        \includegraphics[width=\textwidth]{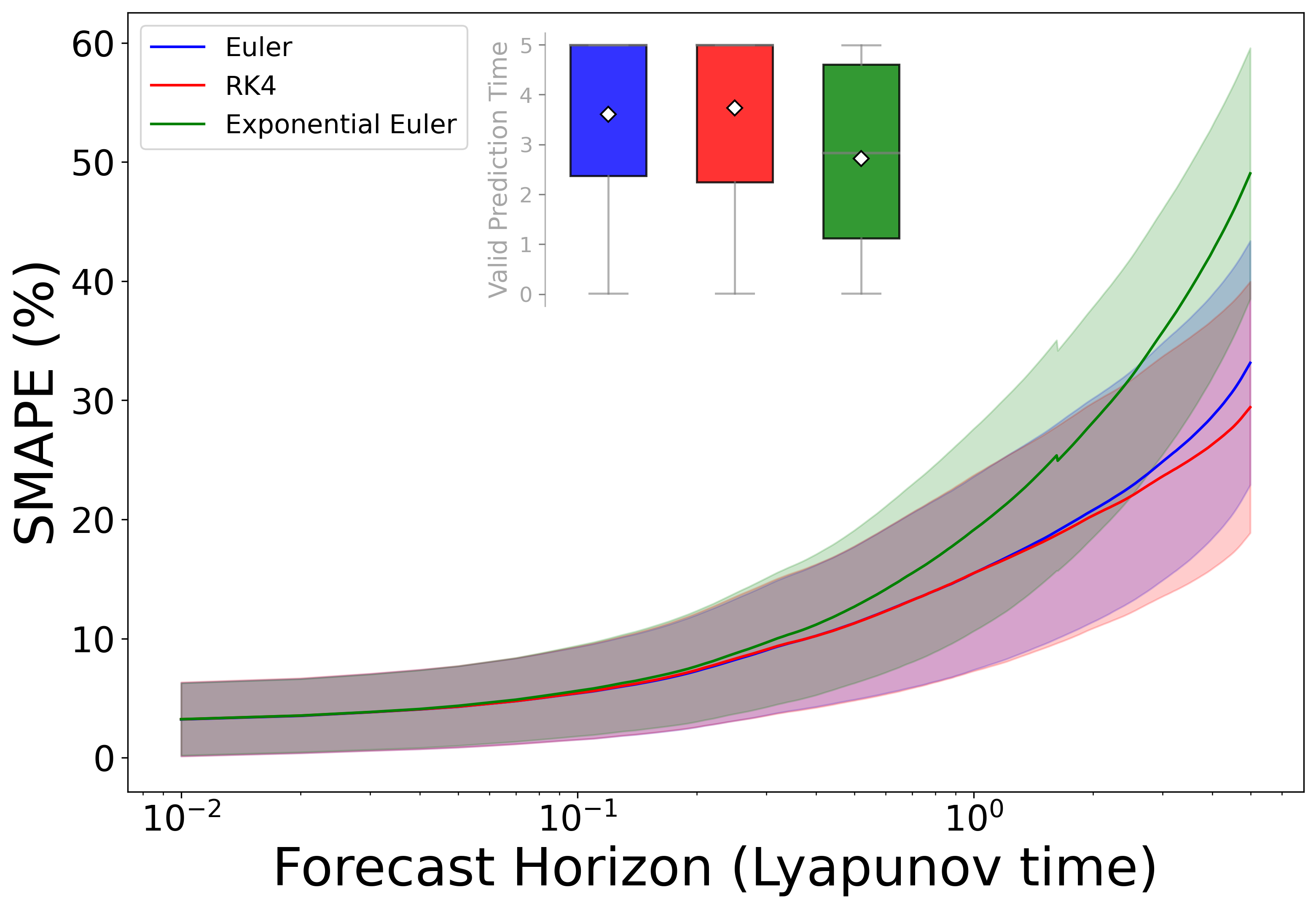}
        \caption{Ablation Study for ODE Solvers}
        \label{fig:AblationStudySub2}
    \end{subfigure}
    \begin{subfigure}{0.45\textwidth}
        \centering
        \includegraphics[width=\textwidth]{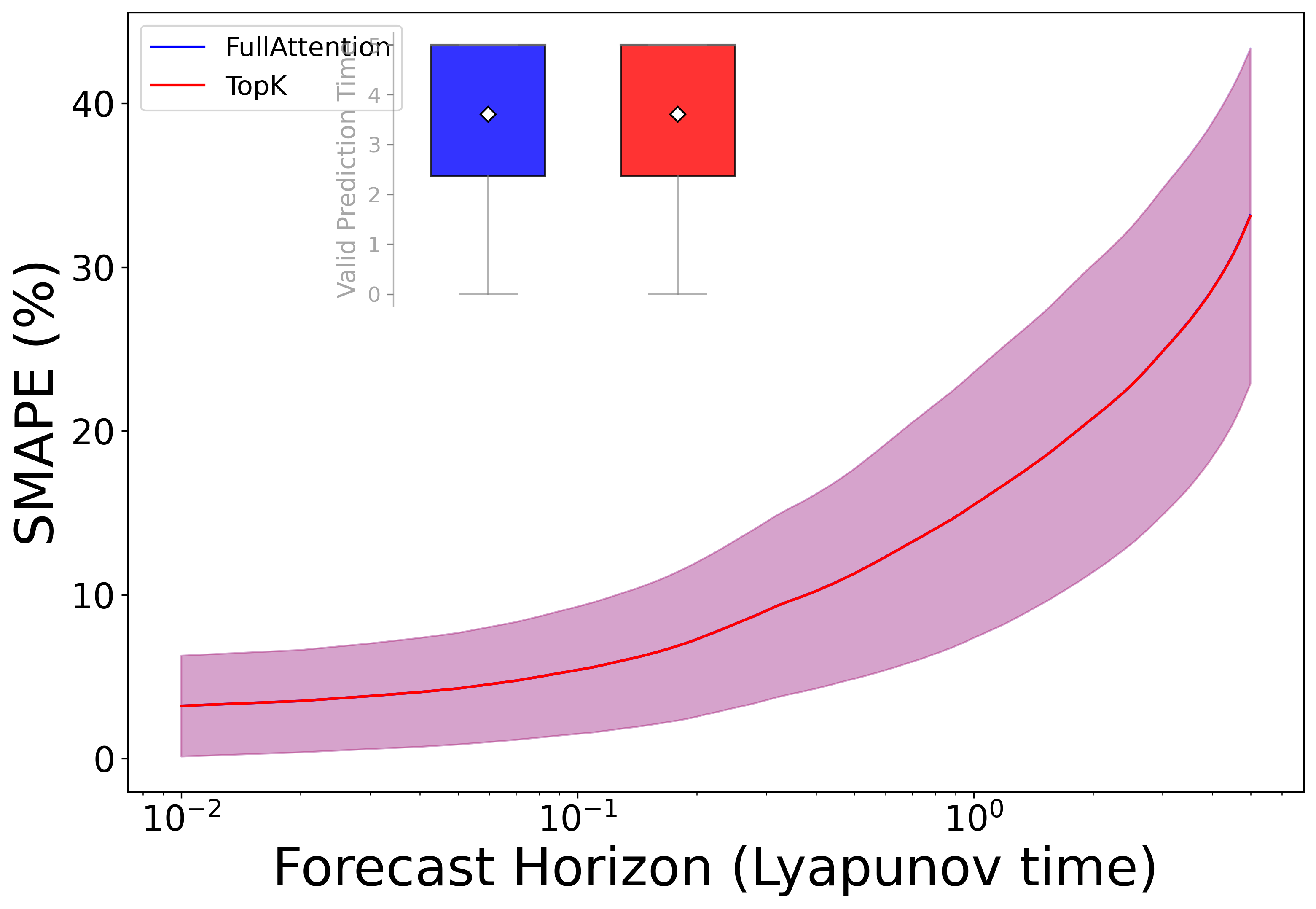}
        \caption{Ablation Study for Truncation Model}
        \label{fig:AblationStudySub3}
    \end{subfigure}
    \caption{\textbf{Ablation Study.} (a) Ablation study for ODE integration steps: $[30,50,100]$. (b) Ablation study for three ODE solvers: $[\text{Euler},\text{Runge–Kutta},\text{Exponential Euler}]$. (c) Ablation study for truncation model, truncation number $R=256$. Shaded regions in (a)-(c) indicate the mean $\pm0.3$ standard deviation across the 135 dynamical systems, each with 20 trajectories originating from randomly sampled initial conditions.}
    \label{fig:AblationStudy}
\end{figure}

\clearpage 
\section{Additional Empirical Results on Real-World Datasets} \label{app_realworld}

Although this work focuses on scientific datasets arising in nonlinear dynamics, we provide a broader evaluation and  compare FreeFM with several representative models (using default-config baselines) in terms of forecasting performance, measured by Mean Squared Error (MSE) and Continuous Ranked Probability Score (CRPS), on several real-world datasets of varying dimensionality.

The real-world benchmarks comprise Electricity, Exchange Rate, Solar Energy, and Traffic from the multivariate time series repository\footnote{\url{https://github.com/laiguokun/multivariate-time series-data}} of \citep{lai2018modeling}, the Bitcoin dataset and the Australian electricity demand dataset from the Monash Time Series Forecasting Repository\footnote{\url{https://forecastingdata.org/}} \citep{godahewa2021monash}.
Detailed descriptions of the real-world dataset we used are as follows:

\begin{itemize}
    \item Australian Electricity: half-hourly electricity demand across Australian states
    \item Exchange Rate: daily exchange rates for 8 countries
    \item  Bitcoin: 18-dimensional daily financial time series
    \item Solar: 10-minute solar power production measurements
    \item Electricity: hourly electricity consumption from the UCI Electricity Load Diagrams dataset
    \item Traffic: hourly road occupancy from the California PeMS Bay Area system
\end{itemize}
These datasets span diverse real-world domains and exhibit noise, heterogeneity, and potential non-stationarity (e.g., the Bitcoin dataset).

For FreeFM, we use distance scanning followed by
top-\(R\) aggregation, with \(R=16\) for all datasets except Solar Energy,
where \(R=256\).
FreeFM and VanillaFM use validation-based model selection, whereas the
generic neural baselines use the default configurations of their reference
implementations. We therefore interpret these broader architecture
comparisons descriptively rather than as fully matched hyperparameter
comparisons.

\subsection{Short-Term Forecasting} \label{app_shortreal}
All experiments follow a unified protocol. For each dataset, we retain the last 1005 time points of the full trajectory. The first 1000 points are treated as the observed segment, and we evaluate 5-step-ahead forecasts on the final 5 held-out points. The observed segment is further split chronologically, with 70\% used for training and 30\% for validation. All variables are standardized using z-score normalization computed from the training split only. The use of the last 1000 points is a task-design choice (for consistency across all datasets) rather than an attempt to ignore longer-term heterogeneity.

For model configuration, FreeFM performs model selection over a small grid of hyperparameters and then generates forecasts autoregressively using top-$R$ mode (with $R=16$ for all except Solar, where we use $R=256$), Euler integration with 100 ODE steps, and 50 particles. VanillaFM uses a validation-based search over context length, hidden dimension, number of layers, learning rate, and noise scale, followed by retraining with the selected configuration. For the benchmark baselines, we use implementations from the Darts library \citep{herzen2022darts}, including Transformer, LSTM, N-BEATS, and TiDE.

\begin{table}[!t]
\centering
\caption{Forecasting performance in terms of MSE on real-world datasets for horizon 5.}
\label{tab:mse_results}
\scriptsize
\setlength{\tabcolsep}{1pt}
\begin{tabular}{lcccccc}
\toprule
Model & Aus. E. ($d=5$) & Exchange ($d=8$) & Bitcoin ($d=18$) & Solar ($d=137$) & Elec. ($d=321$) & Traffic ($d=862$) \\
\midrule
\rowcolor{orange!12}
FreeFM      & $\underline{0.0248 \pm 0.0025}$ & $\underline{0.0085 \pm 0.0002}$ & $\underline{5.1263 \pm 0.0202}$ & $\underline{0.0004 \pm 0.0003}$ & $\underline{1.0900 \pm 0.0009}$ & $0.1090 \pm 0.0003$ \\
VanillaFM   & $0.0343 \pm 0.0079$ & $1.4640 \pm 0.3683$ & $12.5831 \pm 1.2431$ & $0.0055 \pm 0.0006$ & $1.8024 \pm 0.1189$ & $\underline{0.0764 \pm 0.0017}$ \\
LSTM        & $0.0611 \pm 0.0362$ & $1.1282 \pm 0.1050$ & $11.4140 \pm 0.7479$ & $0.0016 \pm 0.0015$ & $1.7666 \pm 0.3259$ & $0.9175 \pm 0.3364$ \\
NBEATS      & $0.0620 \pm 0.0128$ & $0.9466 \pm 0.2604$ & $13.0375 \pm 4.1430$ & $0.0090 \pm 0.0121$ & $1.9242 \pm 0.0523$ & $0.0766 \pm 0.0051$ \\
TiDE        & $0.0402 \pm 0.0120$ & $0.3209 \pm 0.2017$ & $7.5923 \pm 2.5347$ & $0.0091 \pm 0.0105$ & $1.4096 \pm 0.0681$ & $0.0819 \pm 0.0267$ \\
Transformer & $0.0680 \pm 0.0295$ & $1.3253 \pm 0.0697$ & $14.3725 \pm 0.8058$ & $0.0010 \pm 0.0007$ & $1.3642 \pm 0.1202$ & $0.1124 \pm 0.0408$ \\
\bottomrule
\end{tabular}
\vspace{0.5cm}
\centering
\caption{Forecasting performance in terms of CRPS on real-world datasets for horizon 5.}
\label{tab:crps_results}
\scriptsize
\setlength{\tabcolsep}{1pt}
\begin{tabular}{lcccccc}
\toprule
Model & Aus. E. ($d=5$) & Exchange ($d=8$) & Bitcoin ($d=18$) & Solar ($d=137$) & Elec. ($d=321$) & Traffic ($d=862$) \\
\midrule
\rowcolor{orange!12} FreeFM      & $\underline{0.0846 \pm 0.0057}$ & $\underline{0.0532 \pm 0.0012}$ & $\underline{1.6801 \pm 0.0038}$ & $\underline{0.0128 \pm 0.0006}$ & $\underline{0.7552 \pm 0.0004}$ & $0.1860 \pm 0.0009$ \\
VanillaFM   & $0.1058 \pm 0.0160$ & $0.8286 \pm 0.1267$ & $2.2955 \pm 0.1405$ & $0.1225 \pm 0.0013$ & $0.8680 \pm 0.0405$ & $\underline{0.1675 \pm 0.0011}$ \\
LSTM        & $0.1898 \pm 0.0688$ & $0.7259 \pm 0.0295$ & $2.4040 \pm 0.1325$ & $0.0260 \pm 0.0120$ & $1.0704 \pm 0.1206$ & $0.7540 \pm 0.1486$ \\
NBEATS      & $0.2001 \pm 0.0204$ & $0.8398 \pm 0.1259$ & $2.9510 \pm 0.4986$ & $0.0545 \pm 0.0262$ & $1.1138 \pm 0.0166$ & $0.1968 \pm 0.0100$ \\
TiDE        & $0.1566 \pm 0.0230$ & $0.4151 \pm 0.1594$ & $2.1409 \pm 0.4213$ & $0.0531 \pm 0.0307$ & $0.9907 \pm 0.0238$ & $0.2062 \pm 0.0420$ \\
Transformer & $0.2064 \pm 0.0584$ & $0.8239 \pm 0.0350$ & $2.6231 \pm 0.0692$ & $0.0259 \pm 0.0107$ & $0.8854 \pm 0.0405$ & $0.2341 \pm 0.0614$ \\
\bottomrule
\end{tabular}
\end{table}

\begin{figure}
    \centering
    \includegraphics[width=0.95\linewidth]{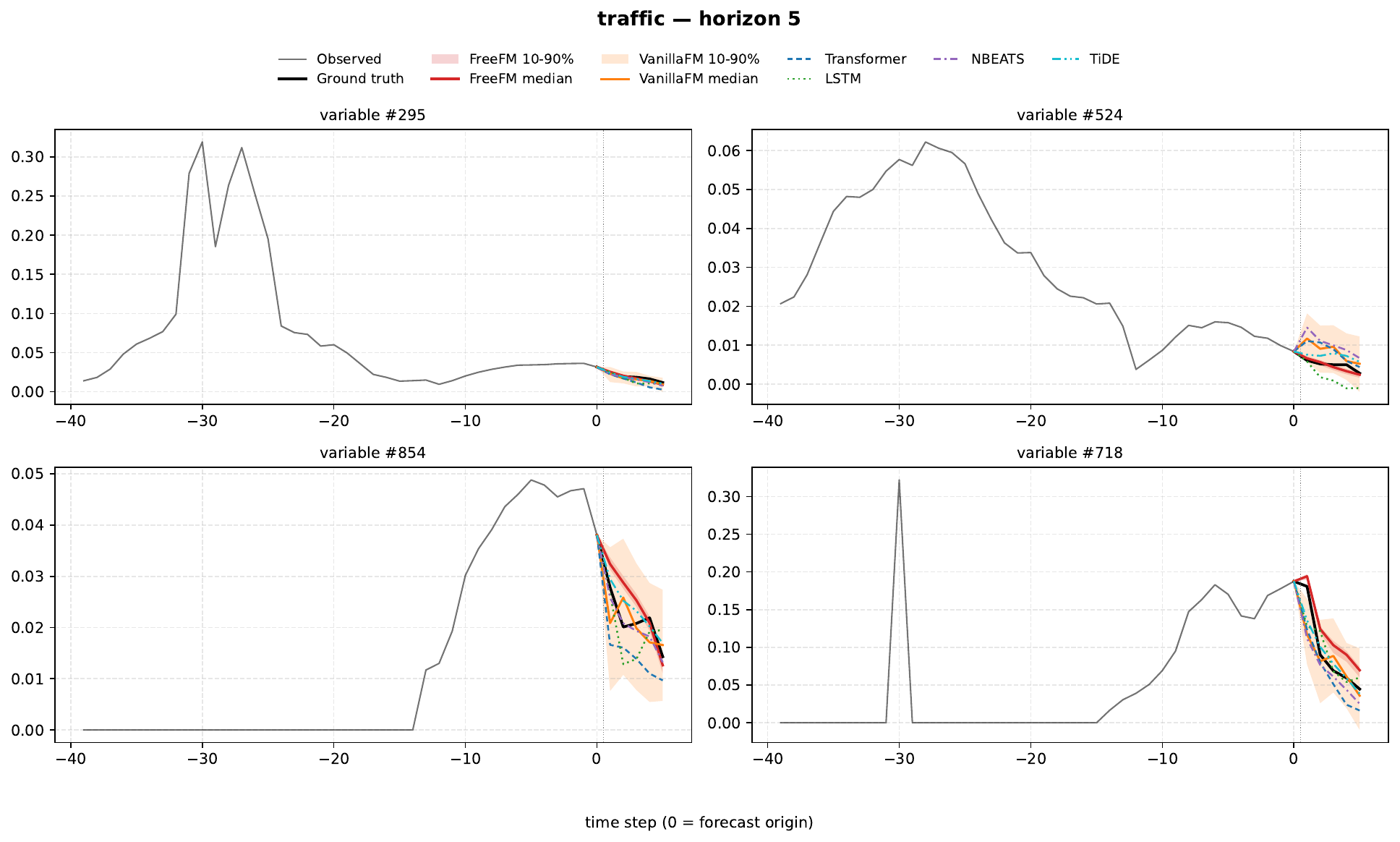}
    \caption{\textbf{Real-World Trajectory Forecasting Examples (Horizon 5).} Predicted trajectories on the Traffic dataset. We show four coordinates forecast by FreeFM and the baseline models over a prediction horizon of five time steps. Shaded regions indicate the 10–90th percentile interval, computed at each future time step from 50 Monte Carlo samples drawn from each model.}
    \label{fig:traffic_trajh5}
\end{figure}

Each experiment is repeated over 5 random seeds, and we report the mean $\pm$ standard deviation. Figure \ref{fig:traffic_trajh5} gives visualizations of the forecasted trajectories for the Traffic dataset. Table~\ref{tab:mse_results}-\ref{tab:crps_results} shows that FreeFM consistently outperforms the considered sequence models in low to moderately high-dimensional settings, highlighting the effectiveness of its underlying structure. In very high-dimensional regimes (e.g., the Traffic dataset with $d=862$), the performance becomes more mixed. While still competitive, it is less consistently dominant, which aligns with the known limitations of nonparametric, kernel-based, or nearest-neighbor-type methods. In high dimensions, distance-based similarity measures become less informative, and significantly more data may be required to maintain accurate local approximations. This behavior of FreeFM is also consistent with the closed-form diffusion models in \citep{scarvelis2023closed}, where smoothing/regularization is likewise
needed to mitigate  memorization but does not by itself remove the
high-dimensional limitations of nonparametric estimators.

Overall, these results suggest that FreeFM remains competitive for short-term forecasting even in high-dimensional settings, indicating practical robustness despite the challenges posed by increasing dimensionality.

\subsection{Beyond Short-Term Forecasting} \label{app_longreal}

We further test the limit of FreeFM and the baselines in a longer-horizon  forecasting setting. For each dataset, we retain instead the last 2500 time points of the full trajectory. The first 2450 points are treated as the observed segment, and we evaluate 50-step-ahead forecasts on the final 50 held-out points. The observed segment is further split chronologically as before. All variables are standardized using z-score normalization computed from the training split only. The model configuration for FreeFM and VanillaFM is the same as before. 

\begin{table}[!t]
\centering
\vspace{0.5cm}
\caption{Forecasting performance in terms of MSE on real-world datasets for horizon $50$.}
\label{tab:mse_results_h50}
\scriptsize
\setlength{\tabcolsep}{1pt}
\begin{tabular}{lcccccc}
\toprule
Model & Aus. E. ($d=5$) & Exchange ($d=8$) & Bitcoin ($d=18$) & Solar ($d=137$) & Elec. ($d=321$) & Traffic ($d=862$) \\
\midrule
\rowcolor{orange!12} FreeFM      & $0.5600 \pm 0.0682$ & $\underline{0.4706 \pm 0.0024}$ & $\underline{10.5514 \pm 1.0760}$ & $0.8675 \pm 0.7373$ & $2.1937 \pm 0.0592$ & $0.3087 \pm 0.0123$ \\
VanillaFM   & $\underline{0.3017 \pm 0.1273}$ & $13.6606 \pm 1.5202$ & $15.5929 \pm 2.1633$ & $\underline{0.1497 \pm 0.1485}$ & $\underline{0.8512 \pm 0.0575}$ & $\underline{0.2238 \pm 0.0189}$ \\
LSTM        & $0.5842 \pm 0.3684$ & $10.7276 \pm 3.1165$ & $19.5922 \pm 2.0108$ & $0.5650 \pm 0.7364$ & $2.0271 \pm 0.7627$ & $1.5010 \pm 0.7156$ \\
NBEATS      & $0.5358 \pm 0.1991$ & $9.6630 \pm 1.5681$ & $11.1885 \pm 3.4449$ & $0.6536 \pm 0.7858$ & -- & $0.2298 \pm 0.0239$ \\
TiDE        & $0.8860 \pm 0.5535$ & $10.3629 \pm 7.4060$ & $27.6160 \pm 35.7030$ & $0.1664 \pm 0.2195$ & $1.0334 \pm 0.1614$ & $0.4046 \pm 0.2179$ \\
Transformer & $0.5309 \pm 0.3648$ & $7.5731 \pm 0.9196$ & $18.4517 \pm 3.3422$ & $1.1303 \pm 2.1847$ & $1.1184 \pm 0.2967$ & $0.9943 \pm 0.7531$ \\
\bottomrule
\end{tabular}
\vspace{0.5cm}
\centering
\caption{Forecasting performance in terms of CRPS on real-world datasets for horizon $50$.}
\label{tab:crps_results_h50}
\scriptsize
\setlength{\tabcolsep}{1pt}
\begin{tabular}{lcccccc}
\toprule
Model & Aus. E. ($d=5$) & Exchange ($d=8$) & Bitcoin ($d=18$) & Solar ($d=137$) & Elec. ($d=321$) & Traffic ($d=862$) \\
\midrule
\rowcolor{orange!12} FreeFM      & $0.4681 \pm 0.0298$ & $\underline{0.4811 \pm 0.0014}$ & $\underline{1.9399 \pm 0.0809}$ & $0.3814 \pm 0.1132$ & $0.9836 \pm 0.0210$ & $0.3027 \pm 0.0119$ \\
VanillaFM   & $\underline{0.3440 \pm 0.0880}$ & $3.1322 \pm 0.2022$ & $2.5521 \pm 0.2497$ & $\underline{0.2018 \pm 0.0781}$ & $\underline{0.5031 \pm 0.0199}$ & $\underline{0.2345 \pm 0.0121}$ \\
LSTM        & $0.5252 \pm 0.1883$ & $2.7034 \pm 0.5403$ & $3.1461 \pm 0.2806$ & $0.4158 \pm 0.4416$ & $1.0460 \pm 0.2496$ & $0.9456 \pm 0.2516$ \\
NBEATS      & $0.4930 \pm 0.1055$ & $2.5811 \pm 0.2892$ & $2.5826 \pm 0.5024$ & $0.4861 \pm 0.4625$ & -- & $0.2842 \pm 0.0271$ \\
TiDE        & $0.7122 \pm 0.2999$ & $2.3743 \pm 0.8921$ & $3.1367 \pm 1.7973$ & $0.2700 \pm 0.2080$ & $0.7184 \pm 0.0760$ & $0.4277 \pm 0.1514$ \\
Transformer & $0.5056 \pm 0.1865$ & $2.0540 \pm 0.1397$ & $3.1877 \pm 0.3875$ & $0.4908 \pm 0.8088$ & $0.7303 \pm 0.1072$ & $0.6764 \pm 0.3228$ \\
\bottomrule
\end{tabular}
\end{table}

\begin{figure}
    \centering
    \includegraphics[width=0.95\linewidth]{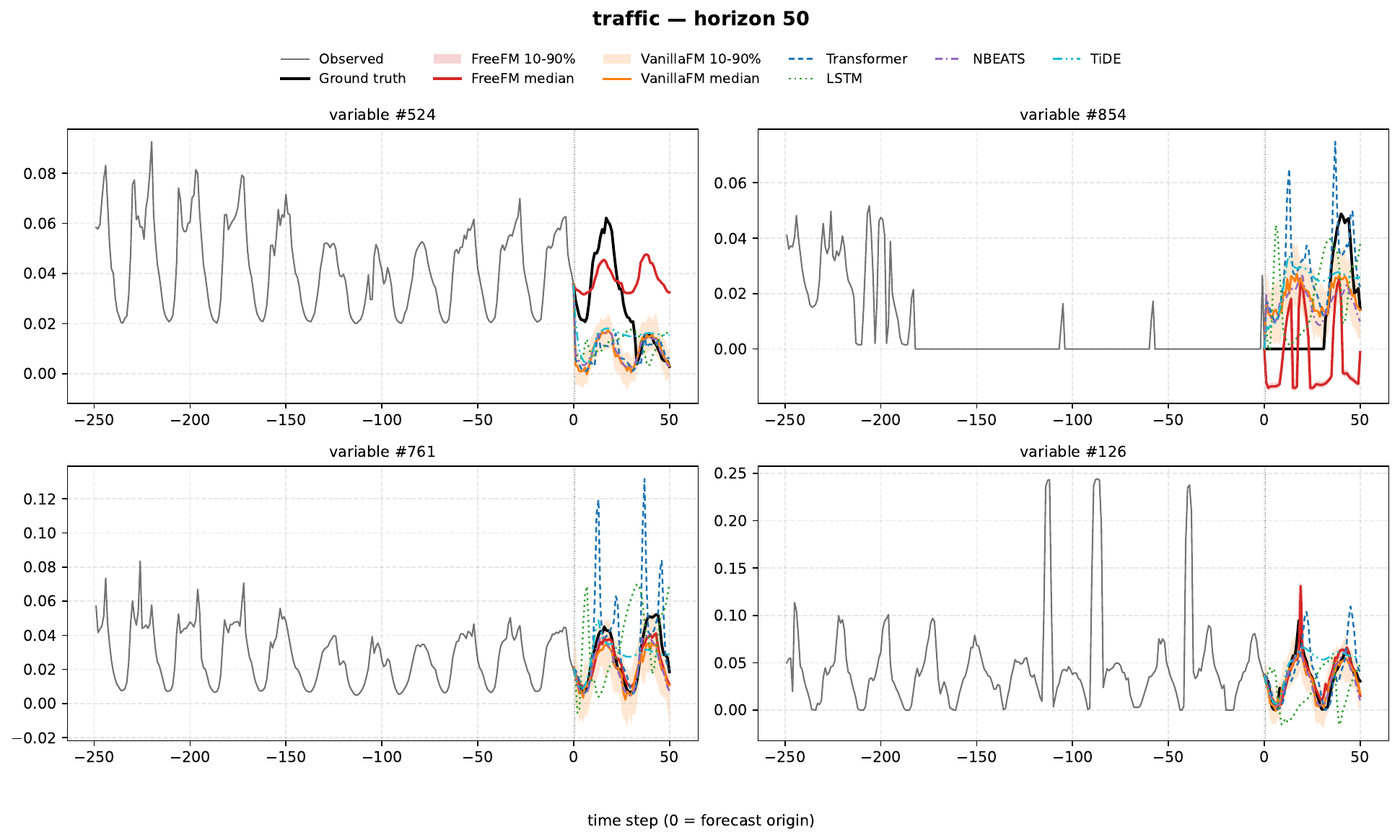}
    \caption{\textbf{Real-World Trajectory Forecasting Examples (Horizon 50).} Predicted trajectories on the Traffic dataset. We show four coordinates forecast by FreeFM and the baseline models over a prediction horizon of fifty time steps. Shaded regions indicate the 10–90th percentile interval, computed at each future time step from 50 Monte Carlo samples drawn from each model.}\label{fig:traffic_trajh50}
\end{figure}

Figure \ref{fig:traffic_trajh50} gives visualizations of the forecasted trajectories for the Traffic dataset. Tables \ref{tab:mse_results_h50}-\ref{tab:crps_results_h50} report the forecasting results for horizon $50$. In this setting, all the considered methods struggle, and the performance pattern is mixed rather than uniformly favorable to a single method. VanillaFM achieves the best results on four of the six datasets, while FreeFM performs best on Exchange and Bitcoin in both MSE and CRPS. These results indicate that FreeFM is competitive, but not uniformly dominant, in medium-horizon forecasting on real-world multivariate series. Note that we omit the N-BEATS result on Electricity at horizon $50$, since the saved seeded runs were numerically unstable: only two of five runs were finite, and those finite runs exhibited overflow-scale errors.

A clearer trend emerges in relation to dimensionality. On very high-dimensional datasets, such as Traffic ($d=862$), FreeFM remains competitive but no longer attains the best performance, which is consistent with the limitations of nonparametric nearest-neighbor type mechanisms in high dimensions, where local distance-based similarity becomes less informative. At the same time, the strong results on Exchange and Bitcoin show that FreeFM can still be highly effective on selected datasets, suggesting that its performance depends not only on dimension but also on the structure and regularity of the underlying dynamics.

Finally, we note that results on these high-dimensional noisy datasets should not necessarily be expected to mirror those on the low-dimensional chaotic benchmarks, as the two settings differ substantially in dimensionality, complexity, and forecasting difficulty. 

\subsection{Wall-Clock Runtime}
\label{app:runtime}

Asymptotic complexity does not fully capture the cost of repeated ODE
evaluation and hyperparameter selection. We therefore additionally report
CPU-only wall-clock runtime. For FreeFM and VanillaFM, the reported totals
include preprocessing, validation/model selection, and the final forecasting
stage. For the remaining neural baselines, the available measurements include
fit and prediction only. 

Table~\ref{tab:runtime-real} shows the wall-clock cost comparison.
Across the six datasets, the final FreeFM rollout itself takes only
\(1.3\)--\(3.8\) seconds; its larger end-to-end runtime is dominated by
preprocessing and hyperparameter validation. FreeFM is therefore not
uniformly the cheapest inference method, but it is substantially cheaper
end-to-end than the VanillaFM baseline in all six settings.

\begin{table}[!h]
\centering
\small
\caption{
CPU-only mean runtime over five seeds, in seconds.
FreeFM and VanillaFM totals include model selection.
}
\label{tab:runtime-real}
\begin{tabular}{lrrrrrr}
\toprule
Dataset & FreeFM & VanillaFM & TiDE & LSTM & Trans. & N-BEATS\\
\midrule
Aus.\ Elec. & 82.2 & 580.0 & 12.6 & 25.3 & 171.5 & 333.5\\
Exchange & 57.2 & 528.0 & 11.3 & 22.4 & 139.1 & 317.2\\
Bitcoin & 74.1 & 552.2 & 11.6 & 23.2 & 140.3 & 320.4\\
Solar & 125.3 & 843.2 & 14.8 & 30.2 & 155.7 & 943.9\\
Electricity & 229.7 & 1263.5 & 17.3 & 41.0 & 155.7 & 1730.8\\
Traffic & 574.6 & 4814.0 & 32.3 & 87.6 & 156.4 & 5223.5\\
\bottomrule
\end{tabular}
\end{table}

\section{Ablation Study on Sigma} 
\label{app_ablation_sigma}

\paragraph{Sensitivity to the Gaussian path scale}
Proposition~\ref{prop:freefm} shows that $\sigma$ and $\sigma_{\min}$ play
related but distinct roles through
\[
c_t^2=\sigma_{\min}^2+\sigma^2t(1-t),
\qquad
g_t=\frac{\sigma^2(1-2t)}{2c_t^2}.
\]
Holding $\sigma_{\min}$ fixed, increasing $\sigma$ broadens the Gaussian
responsibilities in the interior of the probability path and simultaneously
changes the magnitude and time profile of the path-induced score correction.
The endpoint responsibility bandwidth remains fixed,
$c_0=c_1=\sigma_{\min}$, whereas
\[
|g_0|=|g_1|
=
\frac{\sigma^2}{2\sigma_{\min}^2}
\]
increases with $\sigma$. Thus, a $\sigma$-sweep probes the coupled effects of
less localized transition replay in the interior and a changing path-induced
correction; it is not a pure hard-replay versus smoothing limit. In
particular, with $\sigma_{\min}>0$ fixed, $\sigma=0$ yields finite-bandwidth ``replay'' with no score correction, whereas hard transition lookup
requires the effective bandwidth $c_t$ itself to become small.

On Dysts, we evaluate this sensitivity while fixing $\sigma_{\min} = 0.05$ and all other FreeFM settings (top-R, R=8, 50 Euler steps, 30 particles). This is a
controlled sensitivity experiment rather than the fully tuned FreeFM
configuration used in the main benchmark.

\paragraph{Dysts}
Table~\ref{tab:sigma-dysts} reports the results on Dysts, using the same data split and evaluation protocol as the TFM comparison in App.~\ref{app:tfm} (forecast
horizon $H=300$, 15 training and 5 test trajectories per system, memory bank
built from the 12 training trajectories, sMAPE computed on the ensemble mean,
VPT threshold $\epsilon=30$). Metrics are first averaged over
systems for each seed and then over five seeds. Performance is
nearly flat for the smallest values of $\sigma$ and deteriorates as $\sigma$
becomes large, consistent with these clean, densely sampled systems favoring
localized transition retrieval.

\begin{table}[!h]
\centering
\small
\caption{
Sensitivity to $\sigma$ on Dysts with $\sigma_{\min}=0.05$ fixed and $H=300$.
Mean $\pm$ standard error over five seeds. Lower is better for sMAPE, MSE, and CRPS;
higher is better for VPT.
}
\label{tab:sigma-dysts}
\setlength{\tabcolsep}{5pt}
\begin{tabular}{ccccc}
\toprule
$\sigma$ & sMAPE (\%) $\downarrow$ & VPT $\uparrow$
& MSE $\downarrow$ & CRPS $\downarrow$ \\
\midrule
0.001 & \textbf{42.02 $\pm$ 0.27} & 0.730 $\pm$ 0.008
      & \textbf{2.8521 $\pm$ 0.0436} & \textbf{0.3417 $\pm$ 0.0023} \\
0.005 & 42.12 $\pm$ 0.28 & \textbf{0.733 $\pm$ 0.007}
      & 2.8988 $\pm$ 0.0535 & 0.3423 $\pm$ 0.0025 \\
0.010 & 42.37 $\pm$ 0.27 & 0.728 $\pm$ 0.007
      & 2.9214 $\pm$ 0.0611 & 0.3456 $\pm$ 0.0025 \\
0.020 & 43.51 $\pm$ 0.26 & 0.707 $\pm$ 0.007
      & 3.0479 $\pm$ 0.0373 & 0.3561 $\pm$ 0.0014 \\
0.050 & 53.88 $\pm$ 0.29 & 0.614 $\pm$ 0.005
      & 8.0008 $\pm$ 0.1510 & 0.5295 $\pm$ 0.0038 \\
0.100 & 105.86 $\pm$ 0.35 & 0.404 $\pm$ 0.005
      & $(1.1834\pm0.0130)\times 10^5$
      & 22.4631 $\pm$ 0.1591 \\
\bottomrule
\end{tabular}
\end{table}

\paragraph{Real-world tasks}
We further sweep $\sigma$ in the short-horizon ($H=5$) real-world
experiments while keeping $\sigma_{\min} = 0.01$ fixed. Tables~\ref{tab:sigma-real-mse}
and~\ref{tab:sigma-real-crps} show that there is no universally best fixed
value. Australian Electricity favors a larger $\sigma$, Bitcoin an
intermediate value, whereas Exchange Rate, Electricity, and Traffic favor
the smallest values; Solar is comparatively insensitive over much of the
range. In particular, the preferred value is not monotone in ambient
dimension. Although this comparison does not isolate dimensionality
causally (note that the datasets also differ in their geometry, dynamics, noise, and
sampling density), it rules out a simple dimension-based global choice and
supports FreeFM's use of validation-based tuning.

\begin{table}[!h]
\centering
\scriptsize
\caption{
Sensitivity of MSE  to $\sigma$ with
$\sigma_{\min}$ fixed for the short-horizon real-world tasks. Lower is better.
}
\label{tab:sigma-real-mse}
\setlength{\tabcolsep}{3pt}
\begin{tabular}{lcccccc}
\toprule
$\sigma$
& Aus.\ Elec.
& Exchange
& Bitcoin
& Solar
& Electricity
& Traffic \\
& $(d=5)$
& $(d=8)$
& $(d=18)$
& $(d=137)$
& $(d=321)$
& $(d=862)$ \\
\midrule
0.001
& 0.0295 $\pm$ 0.0021
& \textbf{0.0084 $\pm$ 0.0004}
& 5.1512 $\pm$ 0.0170
& \textbf{0.0006 $\pm$ 0.0005}
& \textbf{1.0829 $\pm$ 0.0008}
& \textbf{0.1084 $\pm$ 0.0004} \\
0.005
& 0.0295 $\pm$ 0.0021
& \textbf{0.0084 $\pm$ 0.0004}
& 5.1423 $\pm$ 0.0175
& \textbf{0.0006 $\pm$ 0.0005}
& \textbf{1.0829 $\pm$ 0.0008}
& 0.1085 $\pm$ 0.0005 \\
0.010
& 0.0295 $\pm$ 0.0021
& \textbf{0.0084 $\pm$ 0.0004}
& 5.1132 $\pm$ 0.0167
& \textbf{0.0006 $\pm$ 0.0005}
& 1.0832 $\pm$ 0.0008
& 0.1090 $\pm$ 0.0006 \\
0.050
& 0.0273 $\pm$ 0.0019
& 0.0099 $\pm$ 0.0004
& \textbf{4.6111 $\pm$ 0.0139}
& \textbf{0.0006 $\pm$ 0.0005}
& 1.0899 $\pm$ 0.0007
& 0.1226 $\pm$ 0.0003 \\
0.100
& \textbf{0.0235 $\pm$ 0.0016}
& 0.0167 $\pm$ 0.0005
& 4.6349 $\pm$ 0.0146
& 0.0006 $\pm$ 0.0006
& 1.1020 $\pm$ 0.0008
& 0.1321 $\pm$ 0.0002 \\
\bottomrule
\end{tabular}
\end{table}

\begin{table}[!h]
\centering
\scriptsize
\caption{
Sensitivity of  CRPS to $\sigma$ with
$\sigma_{\min}$ fixed for the short-horizon real-world tasks. Lower is better.
}
\label{tab:sigma-real-crps}
\setlength{\tabcolsep}{3pt}
\begin{tabular}{lcccccc}
\toprule
$\sigma$
& Aus.\ Elec.
& Exchange
& Bitcoin
& Solar
& Electricity
& Traffic \\
& $(d=5)$
& $(d=8)$
& $(d=18)$
& $(d=137)$
& $(d=321)$
& $(d=862)$ \\
\midrule
0.001
& 0.0931 $\pm$ 0.0051
& \textbf{0.0528 $\pm$ 0.0012}
& 1.6828 $\pm$ 0.0031
& \textbf{0.0131 $\pm$ 0.0007}
& \textbf{0.7526 $\pm$ 0.0003}
& \textbf{0.1846 $\pm$ 0.0009} \\
0.005
& 0.0930 $\pm$ 0.0051
& 0.0529 $\pm$ 0.0012
& 1.6816 $\pm$ 0.0034
& \textbf{0.0131 $\pm$ 0.0007}
& \textbf{0.7526 $\pm$ 0.0003}
& 0.1850 $\pm$ 0.0010 \\
0.010
& 0.0930 $\pm$ 0.0051
& 0.0530 $\pm$ 0.0012
& 1.6780 $\pm$ 0.0028
& 0.0132 $\pm$ 0.0007
& 0.7527 $\pm$ 0.0003
& 0.1860 $\pm$ 0.0010 \\
0.050
& 0.0884 $\pm$ 0.0047
& 0.0589 $\pm$ 0.0013
& \textbf{1.6318 $\pm$ 0.0027}
& 0.0134 $\pm$ 0.0007
& 0.7550 $\pm$ 0.0003
& 0.2002 $\pm$ 0.0002 \\
0.100
& \textbf{0.0804 $\pm$ 0.0035}
& 0.0787 $\pm$ 0.0014
& 1.7000 $\pm$ 0.0046
& 0.0139 $\pm$ 0.0007
& 0.7594 $\pm$ 0.0003
& 0.2102 $\pm$ 0.0001 \\
\bottomrule
\end{tabular}
\end{table}

\paragraph{Isolating the path-induced score-based correction}
Because changing \(\sigma\) modifies both retrieval and the score-like
correction, the preceding sweep does not isolate the latter mechanism. We
therefore define the diagnostic family
\begin{equation}
    v^{(\lambda_{\mathrm{corr}})}(t, z)
    :=
    R(t, z)
    +
    \lambda_{\mathrm{corr}} C(t, z),
    \qquad
    \lambda_{\mathrm{corr}}\in[0,1],
    \label{eq:correction-ablation}
\end{equation}
where \(R(t,z)\) and \(C(t,z)\) are the transition replay term and the score-based correction  respectively (see Sec. \ref{sec:method}). Thus,
\(\lambda_{\mathrm{corr}}=0\) gives the replay-only field, while
\(\lambda_{\mathrm{corr}}=1\) recovers the original FreeFM field. For
\(\lambda_{\mathrm{corr}}\neq1\), this modified field is used only as a
diagnostic and need not generate the prescribed Gaussian FM probability
path.

We report this diagnostic on the \(H=5\) real-world forecasting setting.
For each dataset and seed, we first select the FreeFM hyperparameters using
the original validation procedure and then hold fixed the preprocessing,
transition bank, top-\(R\) rule, initialization, and ODE solver while varying
only \(\lambda_{\mathrm{corr}}\).

\begin{table*}[!h]
\centering
\tiny
\caption{
Correction-strength ablation on the real-world tasks for the
short-horizon ($H=5$) setting. Only the coefficient
$\lambda_{\mathrm{corr}}$ is varied, with
$\lambda_{\mathrm{corr}}=0$ corresponding to replay only and
$\lambda_{\mathrm{corr}}=1$ recovering the original FreeFM field.
}
\label{tab:lambda-correction-mse}
\setlength{\tabcolsep}{4pt}
\begin{tabular}{lccccc}
\toprule
& \multicolumn{5}{c}{MSE $\downarrow$} \\
\cmidrule(lr){2-6}
Dataset
& $0$
& $0.25$
& $0.50$
& $0.75$
& $1$ \\
\midrule
Electricity
& \textbf{1.082710 $\pm$ 0.001030}
& 1.084720 $\pm$ 0.001040
& 1.086580 $\pm$ 0.001070
& 1.088400 $\pm$ 0.001070
& 1.089880 $\pm$ 0.001020 \\

Exchange Rate
& \textbf{0.008701 $\pm$ 0.000462}
& 0.008715 $\pm$ 0.000462
& 0.008728 $\pm$ 0.000463
& 0.008742 $\pm$ 0.000463
& 0.008755 $\pm$ 0.000463 \\

Solar Energy
& \textbf{0.000268 $\pm$ 0.000098}
& 0.000292 $\pm$ 0.000128
& 0.000294 $\pm$ 0.000129
& 0.000294 $\pm$ 0.000129
& 0.000303 $\pm$ 0.000127 \\

Traffic
& \textbf{0.108113 $\pm$ 0.000520}
& 0.108292 $\pm$ 0.000428
& 0.108414 $\pm$ 0.000376
& 0.108606 $\pm$ 0.000347
& 0.108796 $\pm$ 0.000339 \\

Bitcoin
& 5.155820 $\pm$ 0.013000
& 5.147300 $\pm$ 0.012800
& 5.136480 $\pm$ 0.012100
& 5.127580 $\pm$ 0.012200
& \textbf{5.118460 $\pm$ 0.012200} \\

Aus.\ Electricity
& 0.029931 $\pm$ 0.002670
& 0.029473 $\pm$ 0.002960
& 0.029065 $\pm$ 0.003280
& 0.028912 $\pm$ 0.003530
& \textbf{0.028567 $\pm$ 0.003760} \\
\bottomrule
\end{tabular}
\end{table*}
\vspace{0.2cm}

\begin{table*}[!h]
\centering
\tiny
\caption{
Correction-strength ablation on the real-world tasks for the
short-horizon ($H=5$) setting. Only the coefficient
$\lambda_{\mathrm{corr}}$ is varied, with
$\lambda_{\mathrm{corr}}=0$ corresponding to replay only and
$\lambda_{\mathrm{corr}}=1$ recovering the original FreeFM field.
}
\label{tab:lambda-correction-crps}
\setlength{\tabcolsep}{4pt}
\begin{tabular}{lccccc}
\toprule
& \multicolumn{5}{c}{CRPS $\downarrow$} \\
\cmidrule(lr){2-6}
Dataset
& $0$
& $0.25$
& $0.50$
& $0.75$
& $1$ \\
\midrule
Electricity
& \textbf{0.752643 $\pm$ 0.000450}
& 0.753326 $\pm$ 0.000438
& 0.753950 $\pm$ 0.000455
& 0.754562 $\pm$ 0.000473
& 0.755057 $\pm$ 0.000454 \\

Exchange Rate
& \textbf{0.054103 $\pm$ 0.001580}
& 0.054160 $\pm$ 0.001580
& 0.054218 $\pm$ 0.001580
& 0.054276 $\pm$ 0.001580
& 0.054334 $\pm$ 0.001580 \\

Solar Energy
& \textbf{0.012725 $\pm$ 0.000144}
& 0.012762 $\pm$ 0.000175
& 0.012764 $\pm$ 0.000175
& 0.012766 $\pm$ 0.000175
& 0.012780 $\pm$ 0.000153 \\

Traffic
& \textbf{0.183996 $\pm$ 0.001220}
& 0.184448 $\pm$ 0.001120
& 0.184820 $\pm$ 0.001080
& 0.185137 $\pm$ 0.000981
& 0.185689 $\pm$ 0.001030 \\

Bitcoin
& 1.683860 $\pm$ 0.001710
& 1.682560 $\pm$ 0.001810
& 1.681220 $\pm$ 0.001600
& 1.680110 $\pm$ 0.001720
& \textbf{1.678980 $\pm$ 0.001480} \\

Aus.\ Electricity
& 0.094519 $\pm$ 0.006360
& 0.093531 $\pm$ 0.006770
& 0.092822 $\pm$ 0.007160
& 0.092475 $\pm$ 0.007620
& \textbf{0.091903 $\pm$ 0.007930} \\
\bottomrule
\end{tabular}
\end{table*}

Table~\ref{tab:lambda-correction-mse}-\ref{tab:lambda-correction-crps} show the results.
The reported mean values vary monotonically with
$\lambda_{\mathrm{corr}}$ on all six datasets, although the direction is
dataset-dependent. Increasing the correction strength improves both MSE and
CRPS on Bitcoin and Australian Electricity, whereas replay-only performs
better on Electricity, Exchange Rate, Solar Energy, and Traffic.

This dependence is consistent with the structure of
\(C(t, z)=g_t(z-\bar m_t(z))\), where
\[
    \bar m_t(z)
    :=
    \sum_j\alpha_j(t,z)m_t^{(j)}.
\]
The correction is therefore not a dataset-independent shrinkage term: it
depends on the local density and geometry of the retrieved transitions and
on the displacement of the forecast particle from their weighted bridge
location. Moreover, \(g_t\) changes sign at \(t=1/2\). The correction thus
changes how particles enter and leave locally supported transition regions
rather than simply pulling them toward a fixed mean.

When the retrieved transitions form a coherent and well-populated
neighborhood, the weighted bridge geometry can provide a useful path
constraint; when the neighborhood is sparse or heterogeneous, the same
correction can reinforce a locally averaged or biased geometry. The present
ablation establishes that the correction has a systematic but
dataset-dependent effect; identifying the precise local geometric cause on
each dataset would require additional diagnostics.

\section{Comparison with Trajectory Flow Matching (TFM)}
\label{app:tfm}

We additionally compare FreeFM with TFM~\citep{zhang2024trajectory}, a neural FM approach designed
for sequential forecasting. This comparison serves a different role from
VanillaFM. VanillaFM is trained on the same one-step transition pairs and
conditional probability path as FreeFM, but replaces the exact empirical
velocity field by a neural network learned through the standard FM regression
objective. It generates the next state by integrating the learned ODE and is
then applied autoregressively. Thus, VanillaFM is the most tightly matched
neural comparison and isolates the effect of replacing the exact empirical
field by a learned neural approximation.

TFM, in contrast, uses a history-conditioned neural stochastic dynamics model
with a different conditioning mechanism, stochastic formulation, and
probability path construction. We therefore regard TFM as a complementary
sequential-FM baseline rather than a controlled neural counterpart to FreeFM.
We adapt TFM to the same forecasting splits and evaluation protocol used in
this comparison. Throughout this appendix, both methods are run with five
seeds while the data split is held fixed, and TFM uses the
TFM-ODE variant with a memory length of 3, trained for 1000 epochs with
learning rate $10^{-3}$ and rolled out autoregressively through its predicted
next state.

\paragraph{Dysts}
Table~\ref{tab:dysts-tfm} reports results on Dysts. This comparison uses a
forecast horizon of $H=300$ (rather than $H=500$ as in Sec.~\ref{sec:exp}),
15 training and 5 test trajectories per system (the training trajectories are
further split 12/3 into train/validation), sMAPE computed on the ensemble
mean, and a VPT threshold of $\epsilon=30$. FreeFM selects
$(\sigma_{\min},\sigma)\in\{0.001,0.005,0.01,0.05\}^2$ by one-step validation
CRPS and forecasts with 100 Euler steps, 50 particles, and top-$R$ retrieval
with $R=8$. Metrics are first averaged over systems for each seed and then over
seeds. FreeFM substantially outperforms TFM on all four aggregate metrics.

These results support the regime emphasized in the main paper: when the
state dimension is low and the transition bank densely covers the local
dynamics, direct empirical transition memory provides a strong forecasting
mechanism.

\begin{table}[!h]
\centering
\small
\caption{
Comparison with TFM on Dysts ($H=300$).
Values are mean \(\pm\) standard error over five seeds.
}
\label{tab:dysts-tfm}
\begin{tabular}{lcc}
\toprule
Metric & FreeFM & TFM \\
\midrule
sMAPE (\%) \(\downarrow\)
& \textbf{23.52 $\pm$ 0.02}
& 76.42 $\pm$ 0.79 \\
VPT \(\uparrow\)
& \textbf{1.192 $\pm$ 0.003}
& 0.370 $\pm$ 0.012 \\
MSE \(\downarrow\)
& \textbf{1.4397 $\pm$ 0.0082}
& 6.6982 $\pm$ 0.4929 \\
CRPS \(\downarrow\)
& \textbf{0.2238 $\pm$ 0.0004}
& 1.1108 $\pm$ 0.0283 \\
\bottomrule
\end{tabular}
\end{table}

\paragraph{Real-world tasks}
Table~\ref{tab:real-tfm} reports the corresponding short-horizon comparison
on the six real-world tasks. For this experiment, FreeFM and TFM are
evaluated using the same splits and protocol within this comparison: each
series is tail-truncated to its last 1000 points, of which the first 995 are
observed and the final 5 are forecast; z-scoring and validation use a 70/30
chronological split of the observed segment. FreeFM uses a fixed
$\sigma_{\min}=0.05$ and a per-dataset $\sigma\in\{0.01,0.05\}$ (fixed across
seeds), a memory bank built from the full observed segment, top-$R$
retrieval with $R=64$, 100 exponential-Euler steps, and 50 particles. The
numbers should therefore be interpreted as a matched FreeFM--TFM comparison
rather than compared directly across tables obtained under different
experimental configurations.

\begin{table*}[!h]
\centering
\scriptsize
\caption{
Short-horizon comparison with TFM on the real-world tasks.
Values are mean \(\pm\) standard error over five seeds.
Bold indicates the best result for each metric and dataset.
}
\label{tab:real-tfm}
\setlength{\tabcolsep}{2.7pt}
\begin{tabular}{lcccccc}
\toprule
& \multicolumn{2}{c}{sMAPE (\%) \(\downarrow\)}
& \multicolumn{2}{c}{CRPS \(\downarrow\)}
& \multicolumn{2}{c}{MSE \(\downarrow\)} \\
\cmidrule(lr){2-3}
\cmidrule(lr){4-5}
\cmidrule(lr){6-7}
Dataset
& FreeFM & TFM
& FreeFM & TFM
& FreeFM & TFM \\
\midrule
Aus.\ Elec.
& 34.08 $\pm$ 0.09
& \textbf{33.88 $\pm$ 1.23}
& \textbf{0.1211 $\pm$ 0.0006}
& 0.1633 $\pm$ 0.0115
& \textbf{0.0415 $\pm$ 0.0002}
& 0.0582 $\pm$ 0.0054
\\
Exchange
& \textbf{9.16 $\pm$ 0.05}
& 9.37 $\pm$ 0.68
& \textbf{0.0685 $\pm$ 0.0006}
& 0.0712 $\pm$ 0.0054
& \textbf{0.0120 $\pm$ 0.0002}
& 0.0134 $\pm$ 0.0018
\\
Bitcoin
& \textbf{61.66 $\pm$ 0.17}
& 71.53 $\pm$ 3.17
& \textbf{0.7821 $\pm$ 0.0005}
& 0.8940 $\pm$ 0.0569
& \textbf{1.6334 $\pm$ 0.0016}
& 2.0112 $\pm$ 0.2558
\\
Solar E
& 8.82 $\pm$ 2.14
& \textbf{5.54 $\pm$ 0.45}
& \textbf{0.0194 $\pm$ 0.0027}
& 0.0241 $\pm$ 0.0022
& 0.0060 $\pm$ 0.0026
& \textbf{0.0017 $\pm$ 0.0002}
\\
Elec.
& 119.99 $\pm$ 0.03
& \textbf{98.31 $\pm$ 0.84}
& 1.5216 $\pm$ 0.0003
& \textbf{0.7252 $\pm$ 0.0089}
& 4.1381 $\pm$ 0.0008
& \textbf{0.9880 $\pm$ 0.0217}
\\
Traffic
& 70.15 $\pm$ 0.02
& \textbf{65.87 $\pm$ 1.85}
& 0.2248 $\pm$ 0.0001
& \textbf{0.1775 $\pm$ 0.0054}
& 0.1747 $\pm$ 0.0001
& \textbf{0.0785 $\pm$ 0.0041}
\\
\bottomrule
\end{tabular}
\end{table*}

The real-world comparison is more nuanced. FreeFM is generally stronger on
the lower-dimensional short-horizon tasks: it outperforms TFM across all
reported metrics on Exchange Rate and Bitcoin, and obtains lower MSE and CRPS
on Australian Electricity. On Solar Energy the two methods are mixed across
metrics, while TFM is substantially stronger on Electricity and Traffic.

Taken together, the TFM comparison reinforces the distinction between the two
approaches. FreeFM is particularly effective in the low-dimensional,
densely sampled regime in which observed transitions form an informative
empirical memory, whereas a learned history-conditioned model such as TFM can
become more competitive when direct local retrieval is less informative.
Crucially, TFM does not provide the matched empirical optimum analyzed in
this work: for a fixed transition-level CFM objective and probability path,
the full-memory FreeFM field is the exact empirical \(L^2\) minimizer, while
VanillaFM remains the appropriate controlled neural approximation to this
same field.

\section{Additional Controlled and Complementary Baseline Comparisons}
\label{app_more}

The comparisons in the main text emphasize standard forecasting architectures
and the VanillaFM model. Here, we add classical nonparametric baselines  to distinguish generic transition retrieval from the
continuous-time empirical-FM construction.

\subsection{Comparison Design}

The comparisons serve different purposes. Kernel conditional estimation
(KCE) is the closest controlled retrieval ablation: it uses the same
train-only transition bank, Euclidean state-space geometry, Gaussian-kernel
family, top-\(R\) convention, and validation protocol as FreeFM, but removes
the interpolation-time FM dynamics, score correction, and ODE integration.
Kernel analog forecasting (KAF) \citep{zhao2016analog} is a stronger dynamics-adapted
nonparametric benchmark, but may use delay-coordinate information and direct
lead-dependent forecasting and is therefore not information matched to FreeFM.

\subsection{Kernel Conditional Estimator (KCE)}

Given a query state \(x\), KCE defines
\begin{equation}
\beta_j(x)
=
\frac{
\exp\!\left(
-\|x-X_1^{(j)}\|^2/(2h^2)
\right)
}{
\sum_{k=1}^M
\exp\!\left(
-\|x-X_1^{(k)}\|^2/(2h^2)
\right)
}
\end{equation}
and predicts
\begin{equation}
\widehat x_{\tau+1}^{\mathrm{KCE}}
=
\sum_{j=1}^M
\beta_j(x_\tau)X_2^{(j)}.
\end{equation}

We use the same train/validation splits, Gaussian-kernel scale family, and
top-\(R\) convention as FreeFM, selecting \(h\) using validation data only.
KCE is then iterated autoregressively for multi-step forecasting. In the top-$R$ implementation, the sum is restricted to the $R$ nearest
predecessor states and the retained weights are renormalized.

Importantly, KCE and the correction-strength ablation in
App.~\ref{app_ablation_sigma} test different aspects of the construction.
KCE removes the continuous-time FM path altogether, whereas
\(\lambda_{\mathrm{corr}}=0\) retains the time-dependent replay field and ODE
integration while removing only the Gaussian-path score correction.

\subsection{Kernel Analog Forecasting (KAF)}

We implement spectral KAF following \citep{zhao2016analog}, preserving the
kernel normalization, eigendecomposition, Nystr\"om extension, and direct
lead-dependent forecast construction. Delay-coordinate covariates are selected
using the validation split. When nontrivial delays are selected, KAF uses
more historical information than current-state FreeFM and should therefore
be interpreted as a strong but not information-matched nonparametric
benchmark. KAF produces deterministic point forecasts in our implementation,
so we report point-forecast metrics but not CRPS.

Table~\ref{tab:dysts-nonparametric} shows the results for the Dysts benchmark,
obtained under the same protocol as the TFM comparison in App.~\ref{app:tfm}
($H=300$, 15 training and 5 test trajectories per system, VPT threshold
$\epsilon=30$, five seeds with a fixed split; KCE uses 30 particles and
top-$R$ retrieval with $R=8$). Note that the data split and evaluation protocol match App.~\ref{app:tfm}, but the controlled
nonparametric comparison uses its own validation-selected FreeFM
configuration; therefore the FreeFM aggregates in Tables~\ref{tab:dysts-tfm} and~\ref{tab:dysts-nonparametric} need
not coincide.

\begin{table}[!h]
\centering
\small
\caption{
Nonparametric comparison on the Dysts systems for which
all three evaluations are available. FreeFM--KCE is the controlled retrieval
comparison; KAF is a complementary dynamics-adapted benchmark.
}
\label{tab:dysts-nonparametric}
\begin{tabular}{lccc}
\toprule
Metric & FreeFM & KCE & KAF \\
\midrule
MSE \(\downarrow\)
& \textbf{1.7092 $\pm$ 0.0201}
& 2.9089 $\pm$ 0.0225
& 5.3139 \\
sMAPE (\%) \(\downarrow\)
& 26.845 $\pm$ 0.029
& \textbf{26.795 $\pm$ 0.123}
& 39.564 \\
VPT \(\uparrow\)
& \textbf{1.1522 $\pm$ 0.0068}
& 1.0059 $\pm$ 0.0094
& 0.8278 \\
\bottomrule
\end{tabular}
\end{table}

FreeFM obtains the lowest mean MSE and the longest
valid prediction time, whereas KCE obtains slightly lower sMAPE. Because KCE
shares the transition memory and Gaussian retrieval principle but removes
the continuous-time FM construction, these results provide controlled
evidence that the latter can improve mean forecast error and forecast
validity beyond ordinary kernel successor averaging, although the benefit is
not uniform across metrics. KAF is less competitive on this benchmark.

Table~\ref{tab:real-kce-kaf} shows the results for the real-world tasks.
At the short horizon, FreeFM achieves the lowest MSE on four of the six
datasets. Its advantage over KCE on several datasets shows that the
continuous-time velocity field and Gaussian-path correction can improve upon
ordinary kernel successor averaging, but the benefit is not universal:
KCE performs best on Solar and KAF on Traffic.

\begin{table}[!t]
\centering
\small
\caption{
MSE comparison with nonparametric baselines on the real-world tasks.
KCE is the controlled retrieval baseline; KAF is a complementary
dynamics-adapted benchmark.
}
\label{tab:real-kce-kaf}
\setlength{\tabcolsep}{5pt}
\begin{tabular}{lccc}
\toprule
Dataset & FreeFM & KCE & KAF \\
\midrule
\multicolumn{4}{c}{\textit{Horizon 5}}\\
\midrule
Electricity & \textbf{1.0900} & 1.5439 & 1.3538 \\
Exchange & \textbf{0.0085} & 1.1770 & 3.2789 \\
Solar & 0.000361 & \textbf{$2.0\times10^{-7}$} & 0.006683 \\
Traffic & 0.1090 & 0.0744 & \textbf{0.0709} \\
Bitcoin & \textbf{5.1263} & 14.8034 & 14.0718 \\
Aus.\ Electricity & \textbf{0.0248} & 0.0559 & 0.0919 \\
\midrule \\
\midrule 
\multicolumn{4}{c}{\textit{Horizon 50}}\\
\midrule
Electricity & 2.1937 & 1.1697 & \textbf{1.0813} \\
Exchange & \textbf{0.4706} & 6.1607 & 7.3095 \\
Solar & 0.8675 & \textbf{0.0042} & 0.2086 \\
Traffic & 0.3087 & 0.2684 & \textbf{0.2532} \\
Bitcoin & \textbf{10.5514} & 16.3542 & 23.3124 \\
Aus.\ Electricity & 0.5600 & 1.0944 & \textbf{0.4930} \\
\bottomrule
\end{tabular}
\end{table}

At horizon 50, the comparison becomes more dataset-dependent. KAF becomes
more competitive on several datasets, which may partly reflect its direct
lead-dependent construction, spectral regularization, and, when selected,
delay-coordinate information. In contrast, FreeFM and KCE are applied
autoregressively and therefore accumulate one-step forecast error. Thus, KAF
is best interpreted as a stronger but structurally distinct nonparametric
forecasting pipeline, whereas KCE remains the controlled retrieval ablation.

\end{document}